\documentclass{article}
\usepackage[a4paper, total={12in, 8in},
 left=20mm, right=20mm,
 top=20mm,bottom=20mm,]{geometry} 

\usepackage{amsmath}
\usepackage{amsthm}

\usepackage[authoryear]{natbib}
\usepackage{amsfonts}
\usepackage{xcolor}
\setcitestyle{numbers,open={[},close={]},comma}
\usepackage{graphicx} % Required for inserting images
\usepackage{bbm}
\definecolor{myglscolor}{RGB}{3, 36, 110}
\definecolor{mycitecolor}{RGB}{201, 38, 82}
\usepackage[backref=page]{hyperref}
\hypersetup{
    colorlinks=true,    % Enable color for links
    citecolor={mycitecolor},    % Color for citations
    linkcolor={myglscolor},   % Color for internal links (e.g., sections, figures)
    filecolor={magenta},% Color for file links
    urlcolor={blue}     % Color for external links (URLs)
}
\usepackage[toc,page]{appendix}
\usepackage[noabbrev]{cleveref}
\usepackage[export]{adjustbox}
\usepackage{booktabs}
\usepackage{colortbl}
\usepackage{pifont}
\usepackage{subcaption}
\usepackage{multirow}
\usepackage[acronym]{glossaries}
\usepackage{overpic}
\usepackage{multirow, makecell}
\usepackage{pifont}
\usepackage{threeparttable}
\usepackage{wrapfig}
\usepackage{authblk}

\usepackage[utf8]{inputenc}
\usepackage[T1]{fontenc}   

\usepackage{algorithm}
\usepackage{algpseudocode}

\definecolor{blueshade}{RGB}{79,149,189}
\definecolor{redshade}{RGB}{207,76,115}

\newcolumntype{L}{>{\centering\arraybackslash}m{6cm}}
\newcolumntype{G}{>{\centering\arraybackslash}m{3cm}}
\newcolumntype{H}{>{\centering\arraybackslash}m{2.7cm}}

\usepackage{lineno}
\newtheorem{definition}{Definition}
\newtheorem{lemma}{Lemma}
\newtheorem{theorem}{Theorem}
\newtheorem{prop}{Proposition}

\newtheorem{assume}{Assumption}

\Crefformat{appendix}{\S#2#1#3}

\newcommand{\y}[1]{\gls{#1}}
\newcommand{\rom}[1]{\uppercase\expandafter{\romannumeral #1\relax}}

\makeatletter
\renewcommand{\paragraph}{\@startsection{paragraph}{4}{0ex}%
   {-3.25ex plus -1ex minus -0.2ex}%
   {1.5ex plus 0.2ex}%
   {\normalfont\normalsize\bfseries}}
\makeatother
\stepcounter{secnumdepth}
\stepcounter{tocdepth}

\date{}

\definecolor{revCol1}{RGB}{0, 0, 0}
\definecolor{revCol2}{RGB}{0,0,0}

\renewenvironment{abstract}
 {\quotation\small\noindent\rule{\linewidth}{.5pt}\par\smallskip
  {\centering\bfseries\abstractname\par}\medskip}
 {\par\noindent\rule{\linewidth}{.5pt}\endquotation}

\newacronym{sft}{S\textit{f}T}{Shape-from-Template}
\newacronym{nrsfm}{NRS\textit{f}M}{Non-Rigid Structure-from-Motion}
\newacronym{sd}{SD}{Standard-Deviation}
\newacronym{slam}{SLAM}{Simultaneous Localisation and Mapping}
\newacronym{icp}{ICP}{Iterative Closest Point}
\newacronym{svd}{SVD}{Singular Value Decomposition}
\newacronym{lls}{LLS}{Linear Least-Squares}
\newacronym{lm}{LM}{Levenberg-Marquardt}
\newacronym{bnb}{BnB}{Branch-and-Bound}
\newacronym{sfm}{S\textit{f}M}{Structure-from-Motion}
\newacronym{socp}{SOCP}{Second-Order Cone Programming}
\newacronym{sdp}{SDP}{Semi-Definite Programming}
\newacronym{sos}{S\textit{o}S}{Sum-of-Squares}
\newacronym{arap}{ARAP}{As-Rigid-As-Possible}
\newacronym{gt}{G\textit{t}}{Groundtruth}
\newacronym{sl}{SL}{Sight-Line}
\newacronym{po}{PO}{Polynomial Optimisation}
\newacronym{lmi}{LMI}{Linear Matrix Inequality}
\newacronym{pop}{P\textit{o}P}{Polynomial Optimisation Problem}
\newacronym{pmi}{PMI}{Polynomial Matrix Inequality}
\newacronym{llr}{L\textit{h}LR}{Lasserre's hierarchy of LMI Relaxations}
\newacronym{isg}{ISG}{Isometric Surface Generator}
\newacronym{nsc}{G-S\textit{f}T-P}{Generalised-S\textit{f}T-and-Pose}
\newacronym{ns}{G-S\textit{f}T}{Generalised-S\textit{f}T}
\newacronym{scc}{SM-S\textit{f}T}{Stereo corresponding, Multi-view S\textit{f}T}
\newacronym{pnp}{P\textit{n}P}{Perspective-$n$-Point}
\newacronym{p3p}{P3P}{Perspective-3-Point}
\newacronym{ar}{AR}{Augmented Reality}
\newacronym{au}{\textit{au}}{Arbitrary Units}
\newacronym{ssm}{SSM}{Statistical Shape Model}
\newacronym{psd}{PSD}{Positive Semi-Definite}
\newacronym{pca}{PCA}{Principal Component Analysis}
\newacronym{admm}{ADMM}{Alternating Direction Method of Multipliers}
\newacronym{rms}{RMSE}{Root Mean Squared Error}
\newacronym{mcp}{MoCap}{Motion Capture}
\newacronym{imu}{IMU}{Inertial Measurement Unit}
\newacronym{lrf}{LRF}{Laser Range Finder}
\newacronym{mnc}{MNC}{Minimum Necessary Correspondences}
\newacronym{emnc}{$\tau$-$\eta$-MNC}{$\tau$-$\eta$-\y{mnc}}
\newacronym{dof}{D\textit{o}F}{Degree-of-Freedom}
\newacronym{gpa}{GPA}{Generalised Procrustes Analysis}
\newacronym{tps}{TPS}{Thin Plate Spline}
\newacronym{cd}{CD}{Chamfer Distance}
\newacronym{ct}{CT}{Computed Tomography}
\newacronym{ipm}{IPM}{Interior Point Method}

\crefname{paragraph}{paragraph}{paragraphs}
\Crefname{paragraph}{Paragraph}{Paragraphs}

\algrenewcommand\algorithmicrepeat{\textbf{Repeat}}
\algrenewcommand\algorithmicuntil{\textbf{Until}}
\algrenewcommand\algorithmicif{\textbf{If}}
\algrenewcommand\algorithmicthen{\textbf{Then}}
\algrenewcommand\algorithmicelse{\textbf{Else}}
\algrenewcommand\algorithmicfor{\textbf{For}}
\algrenewcommand\algorithmicwhile{\textbf{While}}
\algrenewcommand\algorithmicreturn{\textbf{Return}}

\title{Shape-from-Template with Generalised Camera}
\author[a,$\ast$]{Agniva Sengupta}
\author[a]{Stefan Zachow}
\affil[a]{\small \textit{Zuse Institute Berlin (ZIB), 7 Takustra$\beta$e, 14195, Berlin, Germany}}
\affil[$\ast$]{\small \textit{Corresponding author:} {\tt i.agniva+sengupta@gmail.com}}

\begin{document}

\maketitle

\begin{abstract}
This article presents a new method for non-rigidly registering a 3D shape to 2D keypoints observed by a constellation of multiple cameras. Non-rigid registration of a 3D shape to observed 2D keypoints, i.e., \y{sft}, has been widely studied using single images, but \y{sft} with information from multiple-cameras jointly opens new directions for extending the scope of known use-cases such as 3D shape registration in medical imaging and registration from hand-held cameras, to name a few. We represent such multi-camera setup with the generalised camera model; therefore any collection of perspective or orthographic cameras observing any deforming object can be registered. We propose multiple approaches for such \y{sft}: the \textit{first} approach where the corresponded keypoints lie on a direction vector from a known 3D point in space, the \textit{second} approach where the corresponded keypoints lie on a direction vector from an unknown 3D point in space but with known orientation w.r.t some local reference frame, and a \textit{third} approach where, apart from correspondences, the silhouette of the imaged object is also known. Together, these form the first set of solutions to the \y{sft} problem with generalised cameras. The key idea behind \y{sft} with generalised camera is the improved reconstruction accuracy from estimating deformed shape while utilising the additional information from the mutual constraints between multiple views of a deformed object. The correspondence-based approaches are solved with convex programming while the silhouette-based approach is an iterative refinement of the results from the convex solutions. We demonstrate the accuracy of our proposed methods on many synthetic and real data\footnote{Code in \href{https://git.zib.de/asengupta/sft-generalised}{git.zib.de/asengupta/sft-generalised}}.\\

\vspace{1mm}

\end{abstract}

\glsreset{sft}
\section{Introduction}

The problem of non-rigidly registering the 3D template of an object to some observed 2D keypoints from cameras, also know as the \y{sft} problem, has been widely studied based on standard camera models such as perspective \cite{bartoli2015shape, fuentes2021texture,parashar2019local}, orthographic \cite{wang2014robust, zhu2015single} or weak-perspective \cite{chhatkuli2016stable}. But there exists many applications where a deforming object has been imaged with multiple cameras and it is desired that \y{sft} considers data from all such cameras simultaneously. \y{mcp} \cite{baker2020history} is a classical example of such a case involving a network of cameras observing an usually deformable object. More recently, medical applications such as multiplanar radiography \cite{melhem2016eos, ehlke20213d} has become increasingly popular and need \y{sft} solutions of a very similar nature as \y{mcp}, but with volumetric 3D models of anatomical structures. Given such a problem of \y{sft} with a network of multiple cameras, it is natural to consider the most generic projection model of cameras, i.e., the \textit{generalised camera} setup \cite{grossberg2005raxel}, which models each observed keypoint by some camera centre and a direction vector in Euclidean space, these direction vectors need not necessarily intersect or obey parallelism. Such a generalised camera is very well-suited for the multi-camera setup \cite{pless2003using} but has not yet been used in the context of deformable reconstruction methods such as \y{sft}, despite motivating use-cases such as \y{mcp}, \y{sft} from radiography, etc. In this article, we present solutions to this problem of \y{sft} with multiple cameras, modelled as the generalised camera, and demonstrate how this generalised camera based approach offers valuable advantage over traditional approaches of \y{sft}. 

\begin{figure}[t]
\centering
% \begin{overpic}[width=\textwidth]{graphics/Motivation}
% \put(11.5,-1.5){\small (a)}
% \put(36.5,-1.5){\small (b)}
% \put(61.5,-1.5){\small (c)}
% \put(86.5,-1.5){\small (d)}
% \end{overpic}
% \vspace{7mm}
\begin{overpic}[width=0.44\textwidth]{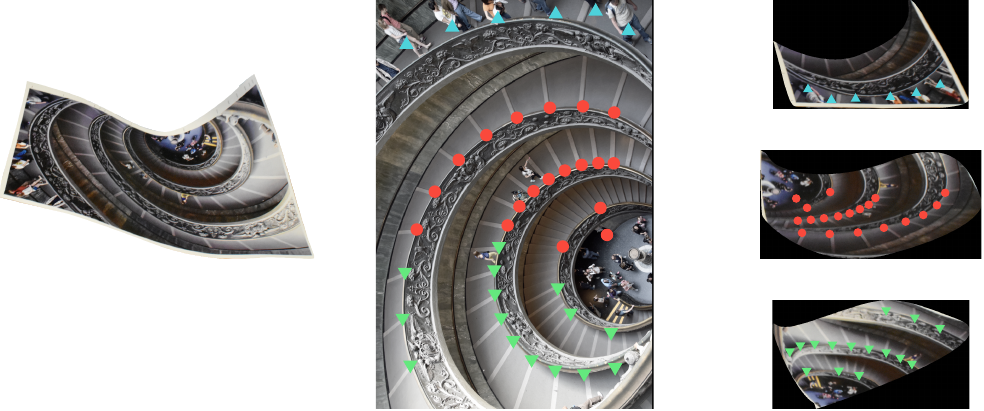}
\put(15,11){\small (a)}
\put(0,38){\footnotesize \parbox{2.0cm}{Object to be reconstructed \tiny (groundtruth)}}
\put(51,-4){\small (b)}
\put(71,32){\small (c)}
\put(71.1,17){\small (d)}
\put(87,-4){\small (e)}
\put(34.5,1){\rotatebox{90}{\footnotesize Undeformed template}}
\put(44,43.5){\small \textbf{Input-1}}
\put(82.5,43.1){\scriptsize Image-1}
\put(82.5,27.7){\scriptsize Image-2}
\put(82.5,12){\scriptsize Image-3}
\end{overpic} 
\begin{overpic}[width=0.54\textwidth]{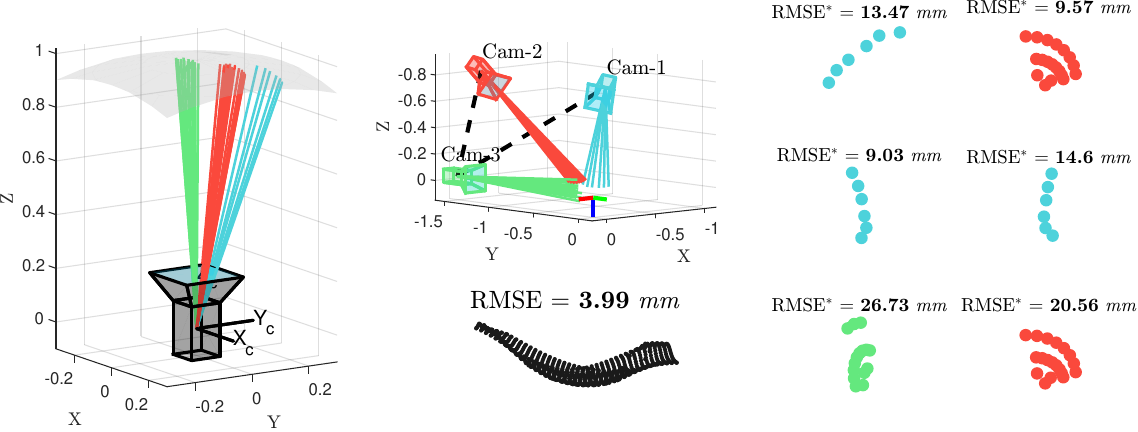}
\put(10,35){\small \textbf{Input-2}}
\put(6.5,18){\tiny \parbox{0.95cm}{3~sets of direction vectors}}
\put(22.6,6.7){\footnotesize $\mathbf{C}$}
\put(54,17){\footnotesize $\mathbb{O}$}
\put(41,35){\small \textbf{Output-1}}
\put(41,1){\small \textbf{Output-2}}
\put(55,24.3){\tiny \parbox{0.95cm}{Relative camera pose}}
\put(14.5,-3){\small (f)}
\put(47,8){\tiny Reconstruction}
\put(32,16){\small (g)}
\put(48.4,-3){\small (h)}
\put(83,-3){\small (i)}
\put(69,27.5){\tiny \cite{bartoli2012template} + \cite{brunet2014monocular}}
\put(90,27.5){\tiny \cite{chhatkuli2014stable}}
\put(69,14){\tiny \cite{chhatkuli2014stable} + \cite{brunet2014monocular}}
\put(90,14){\tiny \cite{brunet2011monocular}}
\put(72.5,1){\tiny \cite{parashar2019local}}
\put(89.5,1){\tiny \cite{shetab2024robusft}}
\put(84,38) {\line(0,-1){38}}
\put(67,26) {\line(1,0){33}}
\put(67,12.5) {\line(1,0){33}}
\put(100.5,30){\rotatebox{-90}{\footnotesize Baseline methods}}
\vspace{2mm}
\end{overpic} 
\caption{An example demonstrating the \y{sft} problem with generalised cameras and an overview of our proposed solution: (a) shows a deformed object that we seek to reconstruct using some template keypoints, (b) shows the undeformed template that has correspondences with three separately captured images of this deformed object, (c), (d), and (e) shows these three images with no overlapping/common keypoints but with correspondences to the undeformed template (colour-coded) and, importantly, all three images have been photoed from three widely separated and unknown camera-poses in space, (f) shows the unit vectors in $\mathbb{S}^2 \subset \mathbb{R}^3$ representing the orientation of each keypoint across the three images, unified in one virtual camera reference frame (colour coded), but the relative pose of each camera w.r.t each other and the shape to be reconstructed are unknown, (g) shows the relative pose of each camera recovered by our proposed method, (h) shows the reconstruction of the complete surface by our proposed method which is significantly more accurate than the ones obtained from state-of-the-art methods, and (i) shows the most accurately reconstructed subset of correspondences (from the three images) for six of the most accurate state-of-the-art methods, the reported error metrics are obtained by combining three separate reconstructions; the details of this experiment are expanded upon in \cref{sec_def_gsftp}}
\label{fig_teaser}
\end{figure}

The generalised camera system has been well-studied in the context of rigid objects including minimal solvers such as a 2-point \cite{hee2013motion} and a 6-point \cite{henrikstewenius2005solutions} algorithm for deriving essential matrix and relative pose of generalised camera systems respectively. But the marriage of generalised camera with deformable objects lead to some challenging but appealing possibilities which are unlike their rigid equivalent. Indeed, unlike rigid scenes, epipolar constraints are not definable on non-rigid scenes. Thus, any hopes of deriving anything similar to a $n$-point minimal solution for relative motion or essential matrix (e.g.: in \cite{hee2013motion}) in non-rigid scene are futile. However, thankfully, we show that there still remains the possibility of \y{sft} with generalised cameras to achieve two broadly defined objectives: 1) when a deformed object has been viewed by multiple cameras with known extrinsics, we want to reconstruct the shape of this deformed object, and 2) when a deformed object has been viewed by multiple cameras with unknown extrinsics, we want to reconstruct the shape of this deformed object and estimate the extrinsics of all the cameras. To achieve these two objectives, we define a problem setup involving an undeformed known template shape that has correspondences to multiple non-overlapping patches of an unknown deformed object, each of these patches being observed by separate cameras. Such a problem setup with non-overlapping observed patches leads to the case of reconstruction from non-stereo corresponding keypoints, also termed as intra-camera correspondences by some authors, which is indeed a challenging setup. Contrast this problem setup of \y{sft} with multiple cameras with the existing literature on \y{sft} \cite{bartoli2015shape, fuentes2021texture,parashar2019local, chhatkuli2016stable, shetab2024robusft, kairanda2022f} where the reconstruction problem is solved in a single camera's local reference frame with no provision for solving \y{sft} from multiple non-overlapping observations jointly; our generalised \y{sft} problem thus offers significant broadening of the scope of \y{sft} and expands its applicability to many new use-cases. Specifically, we show that in the case of known extrinsics of the multiple cameras involved: 1) solving \y{sft} with generalised camera is identical to solving \y{sft} with a single camera (\cref{sol_ns}), and 2) thus there are no minimal requirements of correspondences in each camera, i.e., even one correspondence observed from one of the cameras is reconstructible accurately as long as the constellation of cameras meet some trivial requirements, allowing us to handle combination of perspective and orthographic projections seamlessly. Furthermore, we show that in the case of unknown extrinsics of the multiple cameras involved: 1) \y{sft} with generalised cameras, though ill-posed with unknown extrinsics, is solvable and, using a rigidity prior, the unknown camera extrinsics can also be estimated (\cref{sec_nsc_sol}), 2) the reconstruction accuracy of solving \y{sft} from multiple cameras jointly with our proposed approach is higher than what is achievable by solving the \y{sft} problem separately for each camera (or any trivial combination of these separate results). Importantly, we solve both the problems of \y{sft} with generalised cameras of known and unknown extrinsics using a convex programming approach, leading to the first convex solution of the generalised-\y{sft} problem. An example of the generalised-\y{sft} problem with a summary of our proposed solution is highlighted in \cref{fig_teaser}.

{\color{revCol1} \vspace{2mm}
\noindent \textbf{Motivation.}~In multi-camera setups, it is common to encounter real-world cases where no two cameras share correspondences, yet inter-image rigidity of the observed shape remains exploitable. While classical \y{mcp} systems using industrial rigs are highly accurate and mature, the rise of handheld cameras and widespread amateur practices in video compositing and rotoscoping necessitates extending these capabilities to uncalibrated, unconstrained setups \cite{van2024comparison, mehta2020xnect} - e.g., realistic volumetric rendering of garments from sparse images. Critical applications include medical scenarios such as biplanar radiography \cite{laporte2003biplanar, baka_2d3d_2011}, where clinicians align generic 3D templates to fluoroscopic images captured from multiple views. These images often lack reliable inter-view correspondences and rely on sparse, manually annotated matches. Moreover, extrinsic calibration in such setups is costly and error-prone \cite{zhou2020real,jianu2024guide3d}, making methods that work without it especially valuable. A broader example involves deforming volumetric objects where self-occlusion prevents any single camera from capturing the full surface detail. Although capturing multiple overlapping images is a common workaround, it is often impractical - due to costs (e.g., amateur photography) or other constraints (e.g., reducing radiation exposure from X-rays) - or both. Hence, accurately registering a 3D template to such sparse, non-stereo image sets is a highly desirable capability.}

\vspace{2mm}
\noindent \textbf{Distinction from similar problems.}~Given these above mentioned motivation for deformably registering a 3D template to 2D keypoints in images across multiple views, it is natural to draw comparisons with many parallel methods in computer vision that address problems which are highly similar, at least to verify if a ready solution already exists. Indeed, registering a 3D template to 2D keypoints without deformation falls under the purview of \y{pnp} methods, registering a 3D template to 2D keypoints from one image with deformation gives the \y{sft} methods, while recovering the shape of deformable objects with correspondences but without template \textcolor{revCol2}{and from multiple images with deformations of the object shape inbetween every image} forms the \y{nrsfm} problem and deformably combining multiple shapes is the \y{gpa} problem, all of which have their own respective solution methods but, nonetheless, are very sharply distinguishable from \y{sft} with generalised cameras in the sense that none of these existing problem statements reconstruct a deformable object with a template but across multiple observations with no overlap. We offer a summary of the comparison between all such problem statements in \cref{tab_comp}. \textcolor{revCol1}{Importantly, \y{sft} with generalised cameras cannot be directly compared to \y{pnp}, \y{gpa}, or \y{nrsfm} due to the differences mentioned above, however, \y{sft} in the traditional format remain a special case of \y{sft} with generalised cameras and are thus directly comparable in the case of images acquired via just a single camera}.

\vspace{2mm}
\noindent \textbf{Deformation model.}~
{\color{revCol1} To solve \y{sft} with a generalised camera under potentially strong deformations between the template and the observed object, it is natural to seek methods that leverage inter-camera rigidity while permitting the template to deform significantly from its undeformed state. We adopt learned priors to constrain deformation within a shape space, enforcing consistency of shape across all camera views. This approach, well-established in the classical single-view case~\citep{wang2014robust, zhu2015single, zhou20153d}, lies between two extremes: physically and mathematically grounded models~\citep{bartoli2015shape, casillas2019equiareal, parashar2017isometric, parashar2019local}, which impose structural or differential constraints but break down under non-physical deformations (e.g., free-form or articulated motion), and deep learning methods~\citep{fuentes2021texture, fuentes2022deep, kairanda2022f, stotko2024physics}, which rely on large training datasets and are unsuitable when convexity is essential. In applied settings, deformation models such as isometry, conformality, equiareality, and volume preservation are indeed particularly attractive due to their physical interpretability. However, these constraints are fundamentally incompatible with convex joint estimation of shape and pose. To date, no solution exists that reconciles these properties within a convex framework. We expand upon this challenge in \cref{prob_setup} and \cref{app_challenge_phy}.}

\begin{table}[t]
\begin{adjustbox}{width=\textwidth,center} 
\begin{tabular}{rcLccGGH}
\toprule 
\hline
\multicolumn{1}{l}{}   & \multicolumn{2}{c}{\textbf{What is estimated?}}   & \multirow{2}{*}{\textbf{Shape model}} & \multicolumn{3}{c}{\textbf{Input requirements}}              & \multirow{2}{=}{\centering \textbf{Reconstruction of non-overlapping object-patches}} \\ 
\cline{2-3}  \cline{5-7}  \\    & Shape & Pose   &  & No. of images & Image to image matching & Template to image matching &         \\    \cline{2-8} \\
%\multicolumn{1}{r}{\textbf{Five-point algorithm}} & \checkmark   & Relative motion of cameras  & Rigid   & 2   & \checkmark   &   \ding{55}   &   \ding{55}    \\
\multicolumn{1}{r}{\textbf{P\textit{n}P}}                  & \ding{55}                & Camera-to-shape transformation ($\mathbb{SE}(3)$)  & Rigid     & 1 &  \ding{55} & \checkmark  &   N/A \\
\\
\multicolumn{1}{r}{\textbf{NRS\textit{f}M}}                & \checkmark                 & Point depths in camera coordinates and partial pose for the orthographic case \citep{gotardo2011kernel,hamsici2012learning,dai2014simple,kumar2022organic}${}^{\mathparagraph}$            & Deformable                             & $P > 1$\textcolor{revCol2}{${}^{\dagger\dagger}$}                      & \checkmark &  N/A  &  \ding{55}  \\
\\
\multicolumn{1}{r}{\textbf{S\textit{f}T}}                  & \checkmark                 & Point depths in camera coordinates and partial pose for orthographic cases \cite{wang2014robust, zhu2015single, zhou20153d}${}^{\mathparagraph}$                           & Deformable    & 1   &  \ding{55}    & \checkmark  & \ding{55}$^{\ddagger}$  \\
\\
\multicolumn{1}{r}{\textbf{GPA}}                  & \checkmark$^{\dagger}$                 & Non-rigid transformation between shapes in $\mathbb{R}^n$ for input in $\mathbb{R}^n$ & Deformable    & $P > 1$   &  \checkmark    & N/A  & \ding{55}  \\
\\
\multicolumn{1}{r}{\textbf{Generalised-\y{sft}}}              & \checkmark                 & Point depths in camera coordinates, camera-to-shape transformation ($\mathbb{SE}(3)$) and relative motion of cameras for any combination of orthographic and perspective projections & Deformable   & $P \geq 1$  &  \ding{55}   & \checkmark  & \checkmark   \\ \hline

\bottomrule
\end{tabular}
\end{adjustbox}
\caption{Comparison of the proposed \y{sft} with generalised camera problem with similar problems in computer vision}
\label{tab_comp}
\begin{itemize}
\item[\tiny $\dagger$] \tiny A mean shape is estimated
\item[\tiny \textcolor{revCol2}{$\dagger\dagger$}] \tiny \textcolor{revCol2}{Typically, $P \gg 1$}
\item[\tiny $\ddagger$] \tiny Although repeated \y{sft} across patches is an option but it does not lead to useful results (details in \cref{sec_nec_gen})
\item[\tiny $\mathparagraph$] Pose from orthographic projection recovers 5-\y{dof} pose, translation along $Z$-axis is irrecoverable; no method recovers pose from perspective projection 
\end{itemize}
\end{table}

\vspace{2mm}
\noindent \textbf{Additional input.}~In multi-camera setup, there could exist cases where the number of template-to-image correspondences are simply insufficient for \y{sft}, e.g.: in objects with textures which are hard to match, and if the object has deformed significantly from its template shape. In such a case of scant correspondence, it is usual to seek additional information to aid reconstruction. We utilise silhouettes, which are relatively easy to extract and has many extraction methods with modern approaches \cite{ascenso2020review}, to improve the reconstruction accuracy when correspondences are insufficient. Modifying the convex solution of purely correspondence-based \y{sft}, we propose an iterative, step-wise globally optimal refinement using silhouettes, that succeeds to reconstruct strong deformations from very few correspondences and silhouettes.

\vspace{2mm}
\noindent \textbf{Summary of our contributions}. We summarise our contributions below:
\begin{itemize}
    \item Conceptually, we introduce the generalised camera model to \y{sft} and establish the need for studying \y{sft} with generalised cameras
    \item We propose a convex solution for the \y{sft} problem with multiple cameras when the pose of each camera are known, which we term the \y{ns} problem
    \item We propose a convex solution for the \y{sft} problem and, with the aid of a local rigidity prior, a convex solution for the camera-pose estimation problem (jointly with \y{sft}) with multiple cameras when the pose of all such cameras are unknown, which we term the \y{nsc} problem    
    \item We demonstrate a method to utilise silhouettes to boost the accuracy of \y{ns}/\y{nsc} in the presence of few correspondences, in iterative steps, with guaranteed convergence to some local minima, each of these steps being globally optimal in their own domain
\end{itemize}

\noindent We validate our proposed solution on many synthetic and real data and demonstrate that our method provides the first accurate solution to the \y{ns} and \y{nsc} problem and improves upon the accuracy of \y{ns}/\y{nsc} solution when silhouette information is available.

\section{Previous work}
We discuss the state-of-the-art for correspondence-based and silhouette-based \y{sft} approaches.

\subsection{Correspondence based S\textit{f}T}
\y{sft} is most commonly solved with point-correspondences between a template and a 2D image, therefore we begin our literature survey with point-correspondence based \y{sft}. We first group the existing methods based on their underlying approach and then discuss their applications and the unsolved challenges in such applications.

\subsubsection{Existing approaches for correspondence based S\textit{f}T}
Methodologically, the existing approaches are grouped into three categories: the \textit{first} are the `traditional' ones that uses, invariably, some variant of classical optimisation to solve \y{sft}, the \textit{second} are the learning-based approaches and the \textit{third} are the methods using the generalised camera model which are, technically, not \y{sft} methods, since they are confined to rigid objects, but are nonetheless strongly relevant in the context of our pursuit for \y{sft} with generalised cameras.

\vspace{2mm}
\noindent \textbf{Traditional approaches}. With the advent of computer vision in \y{mcp} techniques, the initial ideas for \y{sft} were driven by \y{ssm} to model the deformation of templates and use classical optimisation to fit these templates to the observed data \cite{zia2013detailed, wang2014robust, zhu2015single, zhou20153d}. Later, many new methods were developed with the assumption that the imaged object tends to be a surface, hence surface-based differential constraints such as isometry, conformity or equiareality could be used to drive the \y{sft} \cite{bartoli2015shape, chhatkuli2016stable, fuentes2021texture, casillas2019equiareal, gallardo2020shape, parashar2019local, parashar2017isometric}.  Such surface-based approaches allowed \y{sft} on larger deformations as long as they obeyed their respective physical model, but were inapplicable to many deformations that breaks physics, e.g.: strong shearing, stretching, torsion, etc., and were also inapplicable to volumetric or articulated objects that cannot be represented as a surface. There exists another group of methods that model the structural mechanics of deformations allowing for reconstruction or tracking of volumetric objects \cite{malti2015linear, kairanda2022f, sengupta2021visual}, but the true structural mechanics of deforming objects tend to be highly non-linear and complicated, resulting in spurious local-minimas.

\vspace{2mm}
% need to talk about SSMs, to cite: ambellan2019invariant
\noindent \textbf{Learning-based approaches}. Alternatively, deep learning methods for \y{sft} have rapidly expanded in recent years \cite{orumi2019unsupervised, fuentes2021texture, fuentes2022deep, tretschk2023state}. These methods perform well on datasets with extensive training samples. However, they still face significant challenges with domain gaps and, more crucially, are not suitable for arbitrary data with limited training samples, hindering their broader adoption.

\vspace{2mm}
\noindent \textbf{Generalised camera in \y{sft}}. While rigid methods for object and camera pose recovery have widely contributed to the tracking and localisation problem with generalised cameras \cite{grossberg2005raxel, pless2003using}, \y{sft} with generalised cameras are yet to be addressed. The rigid methods \cite{li2008linear} for generalised cameras have diverse modalities including methods proposing 2-point minimal solutions for deriving the essential matrix \cite{hee2013motion},  6-point minimal solver for the relative-pose \cite{henrikstewenius2005solutions}, calibration of generalised cameras with three intersecting reference planes \cite{nishimura2015linear}, addressing the pose synchronisation problem \cite{ibuki2011visual, kim2009motion},  pose from known relative rotation angles \cite{martyushev2020efficient}, methods specialised for near-central camera systems \cite{choi2023noniterative}, and in visual \y{slam} systems \cite{kaveti2023design}. It is unsurprisingly essential for such methods using generalised cameras to be adapted to a non-rigid world. This is especially necessary given that many of the applications of \y{sft}, as we discuss hereafter, involves the generalised camera. The conventional approach for dealing with such problems of non-rigid registration in a multi-camera setup (which leads to the generalised camera) is to solve the problem with a non-linear or learning-based approach \cite{furukawa2009dense, gall2009motion, maheshwari2023transfer4d}, which leaves out the cases with sparse correspondences and/or cases with no or few training samples.

\subsubsection{Applications of correspondence based S\textit{f}T}
A thorough review of the applications of \y{sft} is, unfortunately, beyond the scope of our article (we refer the interested readers to some highly detailed domain specific surveys \cite{tian2023recovering, morales2021survey} or a survey of general non-rigid reconstruction \cite{tretschk2023state} methods and applications), but we certainly want to highlight some specific applications that motivates our study of \y{sft} with the generalised camera.

\glsreset{mcp}\vspace{2mm}
\noindent \textbf{\y{mcp}}. While modern industrial MoCap techniques \cite{plantard2017validation, bortolini2020motion, mehrizi2018computer} uses multi-modal sensing involving camera based sensors, \y{imu} and other sensors (e.g.: time-of-flight cameras), the use of monocular cameras are still relevant and an active area of research \cite{cai2019exploiting, pavllo20193d, cheng2019occlusion, wang2020motion, liu2020attention}. The target application of many such \y{mcp} systems involve fitting a template to the captured motion, i.e., an \y{sft} problem, e.g.: computerised rotoscoping \cite{baker2020history, li2016roto++, sturman1994brief}, animated avatars and faces \cite{hartbrich2023eye, grewe2021statistical}, etc. Owing to the widespread proliferation of consumer grade handheld cameras, novel applications such as \y{mcp} with mobile phones are emerging. While almost all of these methods are learning based due to widespread improvement of human body models and the addition of many large scale dataset \cite{wang2021deep}, there still remains the necessity of improved methods for capturing motion of entities without large scale training data, e.g.: \y{mcp} of animals \cite{naik2020animals, zhang2023animaltrack, yao2023openmonkeychallenge} and low-cost, consumer-grade \y{mcp} of arbitrary deformable and/or articulated shapes in sports \cite{suo2024motion} and entertainment \cite{baker2020history} applications without an industrial \y{mcp} rig.

\vspace{2mm}
\noindent \textbf{Medical imaging with colour cameras}. Use of colour cameras for \y{sft} in medical imaging involves applications such as endoscopy \cite{sengupta2024totem, zhang2020template, sengupta2024specular}, laparoscopy \cite{modrzejewski2019vivo, prokopetc2015automatic}, image-guided surgical assistance \citep{barcali2022augmented, fida2018augmented}, maxillo-facial analysis \cite{wang2019practical, park2024automatic}, dermal visualisations \cite{waked2022visualizing, francese2021mobile}, etc. \y{sft} on such medical images enable \y{ar} and visualisation on human internal organs, which are invaluable tools for improving outcome of surgical and invasive diagnosis and procedures. While many modern approaches for \y{sft} on medical images are learning based \cite{venkatesan2021virtual} in-keeping with the modern trend, some \y{sft} use-cases suffer from acute lack of training data (e.g.: endoscopy), leading to difficulties for learning-based methods. 

\vspace{2mm}
\noindent \textbf{X-rays}. \y{sft} from a single X-ray image has been well studied (e.g.: \cite{zhang2021unsupervised}). However, many recent \y{sft} and 3D reconstruction methodologies from X-rays are based on a multi-view setup, i.e., biplanar X-rays \citep{gajny2019quasi, bousigues20233d, girinon2020quasi} where the assumption is the presence of some stereo-corresponding pairs between two orthogonal views and a generic template of the observed bone-structure. However, such correspondences are not guaranteed to be adequate and may not be available in many cases. Moreover, most of these approaches estimate the shape with the assumption of rigidity (e.g.: \citep{gajny2019quasi, zhang2024rigorous}), but many approaches also estimate the shape irrespective of deformations (e.g.: \citep{mitton20003d, nerot20153d}). There has been approaches where such \y{sft} approaches utilised the shape-contour/silhouette information \citep{boussaid20113d, verstraete2017accuracy}. However, all such approaches are non-convex and are prone to being stuck in local minima. Recent approaches for \y{sft} from biplanar X-rays have seen a boom in learning based approaches \citep{an2021robust, shakya2023benchmarking, van2022deep, gao20233dsrnet, kyung2023perspective, chen2024automatic, cheng2023sdct}. But most of the available literature shows the importance of extensive, highly domain-specific training required for such reconstructions, e.g.: chest \citep{cheng2023sdct, kyung2023perspective}, spine \citep{chen2024automatic, gao20233dsrnet}, femur \citep{van2022deep}, and cranium \citep{an2021robust}; implying that domain adaptation remains problematic, e.g.: as highlighted in \citep{shakya2023benchmarking}.

\subsection{Silhouettes in S\textit{f}T}

Silhouette-based \y{sft} has received much attention, primarily from the application of fitting human body model to silhouettes \cite{cheung2005shape, ilic2007implicit, mugaludi2021aligning}. Silhouette-based \y{sft} with a 3D template, which is not necessarily restricted to human body models, follow two major paradigms: the \textit{first} paradigm initialises the pose of the template with some heuristic alignment and then uses some non-convex optimisation to iteratively refine this pose \cite{cheung2005shape, ilic2007implicit,da2010iterative, kong2017using, sunkel2007silhouette, saito2014model}, even till recent times \citep{perez2024alignment}, the most prevalent method for this initial alignment remains a \y{pca} \citep{abdi2010principal} based approach, which can suffer from many inaccuracies depending on viewing directions; the \textit{second} paradigm uses deep-learning based priors, which are very specific to their particular domains \cite{dibra2016hs, mugaludi2021aligning, liu2022concise}. \cite{prasad2005fast, menudetmodel} are examples of convex methods for silhouette-based \y{sft}, but only does so with the help of pre-defined, matchable textures. In the absence of correspondences or some alignment heuristics, the silhouette-based \y{sft} problem remains understandably unsolvable up to global optimality; the utility of silhouettes are therefore primarily to boost the accuracy of some correspondence based method.

\section{Background}
We begin by explaining the background of the \y{sft} and camera-pose estimation problem that we study in this article in three steps: 1) we first recap the generalised camera model, 2) we define the \y{sft} and camera-pose estimation problems with the generalised camera, 3) we highlight why the generalised camera based \y{sft} is an important problem to solve and how the existing solutions are insufficient in many applications. The idea of representing multiple cameras as one generalised camera is due to \citep{pless2003using} and has been widely used in the rigid context. We begin by recalling this generalised camera model and then expound why we believe it is necessary to study the \y{sft} problem with generalised cameras. But first, we need to summarize some basic notations used throughout the article.

\vspace{2mm}
\noindent \textbf{Notations.} For any matrix $\mathbf{X} \in \mathbb{R}^{m \times n}$, we use $\mathrm{vec}(\mathbf{X}) \in \mathbb{R}^{mn}$ to denote the vectorisation of the matrix $\mathbf{X}$ by stacking along the row-major order. $\mathbf{1}_{m \times n}$ denotes a $[m \times n]$ dimensional matrix of all ones. Similarly, $\mathbf{0}_{m \times n}$ denotes a $[m \times n]$ dimensional matrix of all zeros. $\mathbbm{1}_m$ denotes an $[m \times m]$ unit matrix. $\mathcal{S}_+^m$ denotes the manifold of $[m \times m]$ Positive Semi-Definite (PSD) matrices. For any vector $\mathbf{v}_a$, the term $\mathbf{v}_{a,i}$ denotes its $i$-th element; for the matrix $\mathbf{X}_b$, the term $\mathbf{X}_{b,i,j}$ denotes the element at $i$-th row and $j$-th column, $\mathbf{X}_{b,i,[j:j']}, ~ j < j'$ denotes the vector of elements from $\mathbf{X}_{b,i,j}$ to $\mathbf{X}_{b,i,j'}$ and $\mathbf{X}_{b,[i:i'],[j:j']}, ~ i < i',j < j'$ denotes the sub-matrix of $\mathbf{X}_b$ with rows indexed from $i$ to $i'$ and columns indexed from $j$ to $j'$.

\subsection{The generalised camera}\label{intro_gen}
The generalised camera model abstracts away the path traced by light rays for reaching the optic sensor of the camera. Indeed, the geometry of the path of light rays are multifarious and varies between projection models and optical configuration of the camera. Generalised cameras, on the other hand, models the optics as a collection of cones emanating from some points in space. Classically, these cones are parametrised by their direction $(\phi, \theta)$, their aspect ratio $(f_a, f_b)$, and the direction of their eccentricity $\Upsilon$, originating from some point $\mathbf{C} \in \mathbb{R}^3$. But for geometric analysis, we can ignore this conic structure (since it relates to photometric image formation) and focus only on the ray that the pixel samples. A convenient way to denote these rays in space is with Pl\"{u}cker vectors \citep{plucker1865xvii}. The Pl\"{u}cker vectors for a line are defined by the tuple $(\mathfrak{d} \in \mathbb{S}^2, \mathfrak{d}')$ of its direction and moment vector. For any point $\mathbf{P} \in \mathbb{R}^3$ lying on this line, $\mathfrak{d}' = \mathfrak{d} \times \mathbf{P}$ and all points on this line can be denoted by $\{(\mathfrak{d} \times \mathfrak{d}') + \beta \mathfrak{d}, ~\forall \beta \in \mathbb{R}\}$. Importantly, for any point $\mathbf{P}' \in \mathbb{R}^3$ that does not lie on this line, the orthogonal distance of $\mathbf{P}'$ to the line is:
\begin{equation}
     d(\mathbf{P}', \mathfrak{d}) \propto \mathfrak{d} \times \mathbf{P}'.
\end{equation}
\noindent If such a 3D point $\mathbf{P}$ projects to a 2D point with normalised, homogeneous coordinates $\hat{\mathbf{p}}$ in some camera with its optic centre at $\mathbf{C}$, we denote by $\Pi(\mathbf{P})$, the projection of $\mathbf{P}$. The resulting reprojection error, if any, is $\propto \mathfrak{d} \times (\mathbf{P} - \mathbf{C})$ agnostic of the actual projection model of the camera, as long as $\mathfrak{d}$ is an unit vector along the \y{sl} of $\hat{\mathbf{p}}$, i.e., $\mathfrak{d} = \frac{\hat{\mathbf{p}}}{\|\hat{\mathbf{p}}\|}$.

\vspace{2mm}
\noindent \textbf{Silhouettes with generalised cameras}. When it comes to silhouettes viewed from a camera, known models for silhouette formation from a given 3D object demand some conventional projection model; silhouettes from generalised camera are not known in the literature. To deal with silhouettes in the generalised camera setup, we introduce the notion of `viewpoint'. A viewpoint is a position in space from which a bunch of keypoints have been observed, we impose the mere requirement that all keypoints from any given viewpoint must intersect a point in $\mathbb{R}^3$ which are either identical or are at known position w.r.t each other, e.g., the origin of \y{sl}s in a perspective camera. Given $P$ viewpoints, each viewpoint is allowed to have $n_x$ keypoints for $x \in [1,P]$ and each $n_x$ keypoints for the $x$-th viewpoint must be identifiable by a direction vector and the origin of the $x$-th viewpoint. Therefore, the classical generalised camera setup \citep{pless2003using} is realised with $n_x = 1, ~\forall x \in [1,P]$. Note that a viewpoint need not necessarily induce perspective projection (which only happens for intersection at a common point), and neither does a perspective camera necessarily imply a viewpoint (e.g.: for $n_x = 1$), hence the need for the distinction between `viewpoints' and cameras. For orthographic projection, the viewpoint origin is set to the origin of one canonical keypoint, the origin of all other \y{sl}s are at known translation from this canonical keypoint origin. Thus, if silhouettes are needed, an obvious requirement is $n_x \gg 1$, and one can define some $n'_x \leq n_x$ keypoints which are tangent to a 3D object and forms the silhouette at the $x$-th viewpoint.

\subsection{\y{sft} with generalised camera}
\y{sft} is formally defined as the problem of reconstructing the 3D shape of a deformable object from correspondences between a 3D undeformed template and 2D keypoints on an image; the `template' being a known 3D pointcloud. An extension of such \y{sft} setup is the case with generalised cameras, where each keypoint from the generalised camera have correspondences to the 3D template. Therefore, the problem which we are interested to study in this article is formally defined as follows:
\glsreset{nsc}
\begin{definition}{\y{nsc}.}\label{defn_nsc}
    Given observation of a static scene across $P$ different viewpoints with a generalised camera consisting of
    $P$ set of $n_x, ~\forall x \in [1,P]$ correspondences between a non-rigid 3D template $\mathbf{V} \in \mathbb{R}^{3 \times N}$ and $\mathbf{p}_x \in \mathbb{R}^{2 \times n_x}$ 2D normalised keypoints on the $x$-th viewpoint, \y{nsc} is the problem of determining the deformed shape and pose of $\mathbf{V}$ and the pose of all $\{n_x\}$ \y{sl}s expressed as a direction vector and a point on the \y{sl}. 
\end{definition}

\noindent  Our definition of \y{nsc} does not make a distinction between the special cases of generalised projections, e.g.: the `locally-central' case where a group of \y{sl}s are considered fixed in some local coordinate system, or the `axial' case where all projections are assumed to intersect a common line or axis. Instead, our definition of \y{nsc} considers the most generic setup of generalised camera, i.e., no special constraints on the \y{sl}. It must however be noted that there exists applications where the pose of the $P$ camera are known apriori, hence we define an alternative problem formulation that strives to recover only the deformed shape of the template $\mathbf{V}$, given by:

\glsreset{ns}
\begin{definition}{\y{ns}.}\label{defn_ns}
    Given observation of a static scene across $P$ different viewpoints with a generalised camera consisting of
    $P$ set of $n_x, ~\forall x \in [1,P]$ correspondences between a non-rigid 3D template $\mathbf{V} \in \mathbb{R}^{3 \times N}$ and $\mathbf{p}_x \in \mathbb{R}^{2 \times n_x}$ 2D normalised keypoints on the $x$-th viewpoint and with known pose of all $\{n_x\}$ \y{sl}s, expressed as a direction vector and a point on the \y{sl}, \y{ns} is the problem of determining the deformed shape and pose of $\mathbf{V}$. 
\end{definition}
\noindent The classical \y{sft} problem with one-view and either perspective or orthogonal model is contained as a special-case of \y{ns} with $P=1$; in such a case, we assume all \y{sl}s to pass through world origin (perspective) or the normalised keypoints (orthographic), following usual convention of \y{sft} methods, and solve only for the deformed shape of $\mathbf{V}$. While \y{ns} can be considered a somewhat simplified variant of \y{nsc} due to lesser unknown parameters, there are some details of \y{nsc} which are worth taking a closer look to motivate the study of \y{sft} under generalised camera. 

\subsection{Advantages of \y{sft} with generalised camera}\label{sec_nec_gen}
A trivial solution for \y{nsc} could be to solve \y{sft} in each of the $P$ viewpoints and rigidly align the reconstructed shape aposteriori, giving us the camera pose. However, such a simplistic interpretation can be detrimental: 1) as the number of correspondences begin to sparsify, or 2) if the global deformation of the object consists of multiple, disjoint local deformations and the correspondences are also restricted to smaller patches. To understand the implications of sparsity and spatial extent of correspondences on observability of global deformation, we begin with a very intuitive observation about the possible minimal requirements for classical \y{sft}:
\begin{lemma}\label{lemma_min_4}
    To accurately perform \y{sft} with correspondences between a 3D template and image with a perspective camera, there exists a number for \y{mnc}, denoted by $\psi \geq 4$ as long as the deformation of the 3D template has at least one \y{dof}.
\end{lemma}
\begin{proof}
    For rigid objects, we need a minimum of 3 point correspondences to localize a shape in $\mathbb{R}^3$. For \y{sft}, assuming only 1 \y{dof} for deformation (which is impractically low, but nonetheless, the trivial theoretical lower-bound), we need at least one more point correspondence to solve the problem (since correspondences increase in integer steps), therefore 4 correspondences are the minimum. 
\end{proof}
\noindent Thankfully, an exact formulation of \y{mnc} is unimportant for the rest of our discussion, but it is necessary to establish that there exists some \y{mnc} for every \y{sft} problem. While the intuition behind \cref{lemma_min_4} is straightforward, for \y{ns} and \y{nsc}, \cref{lemma_min_4} allows us to establish some minimum necessary requirements on the input, later explained in \cref{sol_ns} and \cref{sec_nsc_sol}, which are not known or required for the generalised camera in rigid settings. More importantly, it is known that the number of available correspondences affect the global shape reconstruction accuracy of many purely correspondence-based \y{sft} approaches (e.g.: experiments from \citep{zhou20153d, bartoli2015shape, chhatkuli2016stable} offers some evidence). We offer additional empirical evidence to support the dependency of global shape reconstruction accuracy on correspondences in \cref{taueta}, leading to \cref{assum_basic} formalised as follows:
\begin{assume}\label{assum_basic}
There exists cases where at least \y{mnc} $\tau \in \mathbb{Z}_+$ is required for solving the \y{sft} problem to a minimum desired accuracy of $\eta$, evaluated across the entire shape to be reconstructed, such that for any number of correspondences $\tau' < \tau$, the shape reconstruction accuracy changes to  $\eta' < \eta$. We denote such a threshold by \y{emnc}. 
\end{assume}
\noindent The case of \cref{assum_basic} arises frequently in realistic \y{sft} problems as the correspondences become sparser, e.g.: for many modern \y{sft} methods on mildly deforming objects, it can be argued that the difference in accuracy between 100 and 110 correspondences is negligible, but the same cannot be said for 5 and 15 correspondences. Moreover, in real-world cases involving large/volumetric objects captured with non-overlapping image patches, it could be expected that \cref{assum_basic} gets violated with even stronger requirements, e.g.: to achieve $\eta$ accuracy, it is not sufficient to just have $\tau$ correspondences, but there exists additionally some requirements about the pose from which these $\tau$ correspondences are captured and some necessary sampling distribution might also need to be considered. Therefore, it is but a mild assumption to consider the existence of \y{emnc}, especially when sparser correspondences are concerned. We utilise \y{emnc} to highlight the necessity of generalised \y{sft} with an example.
% Moreover, we impose certain intuitive assumptions on \y{emnc} to simplify the rest of this discussion. The first assumption deals with the invariability of \y{emnc} with viewing direction:
% \begin{assume}{Direction invariance of \y{emnc}.}\label{assum_homogeneous}
%     We assume every 3D template $\mathbf{V}$ to have the same \y{emnc} irrespective of viewing direction.
% \end{assume}
% \noindent The other important assumption is about the additive nature of \y{emnc}
% \begin{assume}{Additive sufficiency of \y{emnc}.}\label{assum_additive}
%     If a certain object has \y{emnc} for a single image, then for  $P$ images each with $n_x, ~\forall x \in [1,P]$ correspondences each, the necessary \y{emnc} for the combined setup remains lower bound by $\tau$, i.e., $\sum_{x=1}^P n_x = \tau$ is a necessary, but not sufficient, criterion for solvability of the resulting multi-image problem.
% \end{assume}
% \noindent \Cref{assum_homogeneous} and \cref{assum_additive} allows us some important insights into the solvability of \y{nsc} setup with the trivial solution of multiple conventional \y{sft} across each view.

\vspace{2mm}
\noindent \textbf{A minimal example of \y{sft} with multiple cameras}. To understand the various possibilities arising due to \y{emnc} on \y{sft}, it is sufficient to analyse the case of $P=2$, this is generalisable to any value of $P>2$. 

Let us imagine a toy problem consisting of a 2D template with correspondences to two images of the deformed template, not necessarily unique. The requirement is to solve the combined \y{sft} problem up to some accuracy $\eta$ and $\tau$ is the minimum number of correspondence required to achieve this. There can be \textit{four} possibilities in terms of template-to-image correspondence: \rom{1}) there exists at least \y{emnc} correspondences between template and image in both views (SPV$_2$) separately, \rom{2}) there exists at least \y{emnc} correspondences between template and image in only one view (SPV$_1$), \rom{3}) there does not exist \y{emnc} correspondence between template and image in any of the views separately, but together, it is exceeds the \y{emnc} (SPV$_0$), and \rom{4}) there does not exist \y{emnc} correspondence between template and image in any of the views separately, and together, it remains below the \y{emnc} (SPV$_{\emptyset}$). In terms of stereo correspondences between images, there can be \textit{three} possibilities: \rom{1}) no stereo-correspondences (SC$_\emptyset$), \rom{2}) less than 3 stereo correspondences (SC$_{< 3}$), and \rom{3}) 3 or more stereo correspondences (SC$_{\geq 3}$) but less than 5, to avoid opening additional possibilities of trivial solutions from the five-point algorithm \cite{nister2004efficient}.

These combination of template-to-image and stereo correspondences create a total of 12 different possibilities, among which, as we demonstrate below, at least five possibilities happen to be highly interesting:
\begin{enumerate}
    \item \textbf{SPV$_2$/SPV$_1$ + SC$_{\geq 3}$}. Solvable; a possible trivial solution is to solve \y{sft} in the view with sufficient template-to-image correspondences (any one for SPV$_2$), followed by \y{pnp} \cite{gao2003complete} for the other view
    \item \textbf{SPV$_2$ + SC$_{< 3}$/SC$_{\emptyset}$}. Solvable; a possible trivial solution is to solve \y{sft} for both views separately followed by Horn's method \cite{horn1987closed} for estimating relative camera pose
    \item \textbf{SPV$_1$ + SC$_{< 3}$/SC$_\emptyset$}. Solvable if the sum of number of template-to-image correspondences and image-to-image correspondences in the image with less than \y{mnc} correspondences is 3 or more; unsolvable otherwise. For the solvable case, a possible trivial solution is to solve \y{sft} in the view with SPV$_1$ and merge the template-to-image and image-to-image correspondences together to solve the \y{pnp}
    \item \textbf{SPV$_0$ + SC$_{\geq 3}$/SC$_{< 3}$/SC$_\emptyset$ and SPV$_{\emptyset}$ + SC$_{\geq 3}$/SC$_{< 3}$}. Sum of number of template-to-image correspondences and image-to-image correspondences being 4 or more on the two images combined is a necessary criteria for solvability (due to \cref{lemma_min_4}). Sufficient criteria for solvability is unclear. \textit{Existing methods are insufficient when solvable}
    \item \textbf{SPV$_{\emptyset}$ + SC$_\emptyset$}. Unsolvable up to $\eta$ precision
\end{enumerate}

\noindent Clearly, the case of SPV$_0$ + SC$_{\geq 3}$/SC$_{< 3}$/SC$_\emptyset$ and SPV$_{\emptyset}$ + SC$_{< 3}$ cannot be readily solved with existing methods. However, at least SPV$_{\emptyset}$ + SC$_{\geq 5}$ can be solved with existing \y{sfm} techniques, therefore not sufficiently interesting. But among the SPV$_0$ + SC$_{\geq 3}$/SC$_{< 3}$/SC$_\emptyset$ cases, we focus on the most challenging case of non stereo corresponding images, i.e.,  SPV$_{\geq 0}$ + SC$_\emptyset$.

% \vspace{2mm}
% \noindent \textbf{The utility of \y{nsc}/\y{ns} problem}. Given the lack of methods to precisely compute \y{mnc}, the case of SPV$_0$ + SC$_{\geq 3}$/SC$_{< 3}$/SC$_\emptyset$ or SPV$_{\emptyset}$ + SC$_{\geq 3}$/SC$_{< 3}$ cannot be determined apriori for any given \y{sft} setup, but these five cases offer hints about cases where \y{nsc}/\y{ns} is invaluable. Since SPV$_0$ or SPV$_{\emptyset}$ are not known precisely but such cases will inevitably arise and the current methods are insufficient for solving it. The deteriorating accuracy of standard \y{sft} approaches in and around the case of SPV$_0$ + SC$_{\geq 3}$/SC$_{< 3}$/SC$_\emptyset$ or SPV$_{\emptyset}$ + SC$_{\geq 3}$/SC$_{< 3}$ has been empirically verified by us, see our experimental results for details, e.g.: \cref{sec_exp_gsftp}. To simplify the discussion, we focus on the most challenging case of SPV$_{\emptyset}$, i.e., \y{ns} and \y{nsc} with no-stereo correspondences between the images. 

\vspace{2mm} \noindent \textbf{A toy example}. We utilise a toy example in $\mathbb{R}^2$ to expand upon the utility of \y{nsc} problem concretely in the case of SPV$_{\geq 0}$ + SC$_\emptyset$. As shown in \cref{fig_toy}, we suppose there exists a circle, represented by a semi-dense collection of 2D points, which is the \textit{template shape} in \rom{1}. An upper subregion of \rom{1}, shown in blue, corresponds to a subregion of the deformed shape in \rom{2}, this is {\tt correspondence-1} in some reference frame. Another lower subregion of \rom{1}, shown in red, corresponds to another subregion of the deformed shape in \rom{2}, given as {\tt correspondence-2} in another reference frame. Obviously, given the blue and red patches are at opposite ends of the shape in \rom{2}, no camera situated outside of \rom{2} can observe both the patches in one image\footnote{Excluding the case of stereographic projection of \rom{2} by inscribing it inside a circumcircle, in which case the projection may be surjective and does not represent any commonly known use-case, thus safely ignored}, irrespective of their projection model. Existing \y{sft} methods can only solve two \y{sft} problems separately, lacking the means to combine multiple reference frames, for the two sets of correspondences each; assuming they are accurate, we get two reconstructed shapes shown in \rom{3} and \rom{4}, none of which are similar to the \y{gt} shape \rom{2}. Obviously, there exists the trivial solution of combining \rom{3} and \rom{4}, a reasonable choice for combining shapes is to obtain an Euclidean mean shape \cite{bai2022procrustes} of \rom{3} and \rom{4}, which gives us the bold-red shape in \rom{5}, which is still not the \y{gt} shape \rom{2}, shown as the superimposed black shape in \rom{5}. In this article, we show, with extensive experimental validation, that it is possible to reconstruct the \y{gt} shape for \y{ns} and \y{nsc} in any real-world use-case corresponding to \cref{fig_toy}.

\begin{wrapfigure}{r}{0.3\textwidth}
\begin{overpic}[width=0.3\textwidth]{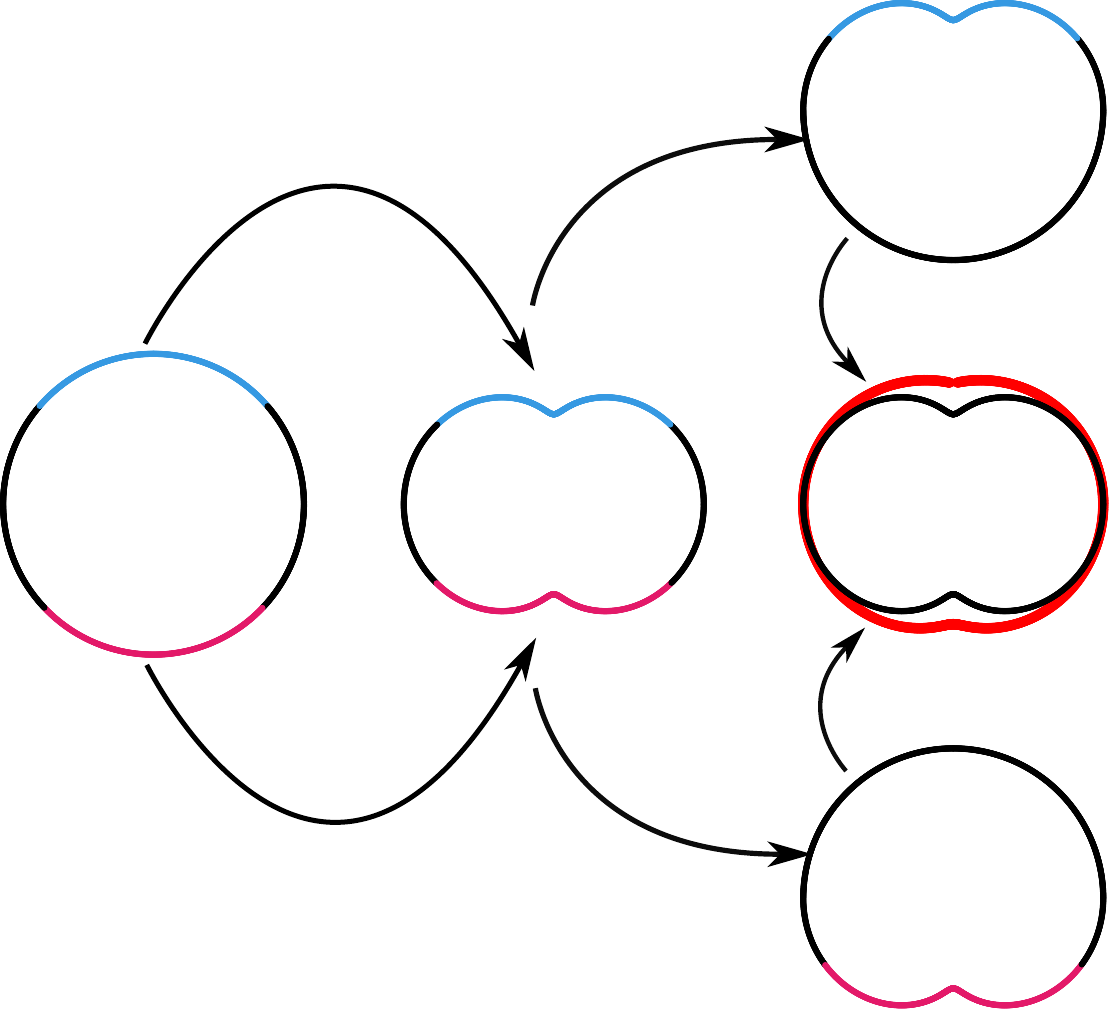}
\put(12,44){\small{\rom{1}}}
\put(47.4,44){\small{\rom{2}}}
\put(83,44){\small{\rom{5}}}
\put(83,78){\small{\rom{3}}}
\put(83,10){\small{\rom{4}}}
\put(13,77){\tiny{Correspondence-1}}
\put(13,13){\tiny{Correspondence-2}}
\put(55,70){\tiny{\y{sft}-1}}
\put(55,22){\tiny{\y{sft}-2}}
\end{overpic} 
\caption{A toy example showing the necessity of \y{ns}/\y{nsc}}
\label{fig_toy}
\end{wrapfigure}

Our article concerns the solution for the SPV$_{\geq 0}$ + SC$_\emptyset$ case, a special $\mathbb{R}^2$-case of which is shown in the toy example of \cref{fig_toy}, but generalised to any number of images captured from any number of views with any projection model. The methodological description follows.

\section{Methodology}\label{sec_method}
The presentation of our methodological contribution is organised as follows: \textit{first} we describe our problem setup thoroughly in \cref{prob_setup}, \textit{second} we formalise our problem statement for \y{ns} in \cref{subsec_nsc_first} followed by the solution in \cref{sol_ns}, \textit{third} we formalise our problem statement for \y{nsc} in \cref{sec_unknown} followed by the solution in \cref{sec_nsc_sol}, and \textit{fourth} we present a silhouette-boosted \y{ns} and \y{nsc} problem in \cref{pron_silh_nsc} and its solution in \cref{sec_synth_sol}.

\subsection{Problem setup}\label{prob_setup}
We first describe two important components of the problem setup, the \textit{object model} including parametrisation of shape deformation and the \textit{image formation model}.

\vspace{2mm}
\noindent \textbf{Deformation model}. We use the well-know SSM \citep{davies2008statistical} to represent the deformable shape of the object being imaged. The shape is described as a pointcloud, where $\bar{\mathbf{P}} \in \mathbb{R}^{3 \times N}$ is the mean shape and the basis pointcloud is $\mathbf{P}_{i} \in \mathbb{R}^{3 \times N} \quad \forall i \in [1,M]$, assuming we have $M$ basis shapes, all coordinates in object reference frame. Assuming $\mathbf{R} \in \mathbb{SO}(3)$ to be the object to world coordinate rotation and $\mathbf{t} \in \mathbb{R}^3$ to be the associated translation, the shape can be represented as $\mathbf{V} = \mathbf{R}\Big(\bar{\mathbf{P}} + \sum_{i=1}^M w_i \mathbf{P}_{i}\Big) + \mathbf{t} = \mathbf{R} \mathbf{Q} + \mathbf{t}$, where $\{w_i\}$ are the basis weights and $\mathbf{Q}$ is the deformed shape. $\mathbf{Q}$ is represented in the local coordinates of the object; this object-centric reference frame is denoted by $\mathbb{O}$. 

\textcolor{revCol1}{Physics-based deformation models - such as isometry (length-preservation), conformality (angle-preservation), equiareality (area-preservation), and volume preservation - are highly appealing for their applicability, yet remain intractable for joint shape and pose estimation within a convex framework. We expand upon this challenge in \cref{app_challenge_phy}.}

\vspace{2mm}
\noindent \textbf{Image-formation model}. We assume that there are $P$ viewpoints which are imaging a shape, the centre of these viewpoints are at $\mathbf{C}_x \in \mathbb{R}^{3}$ w.r.t some world coordinates $\mathbb{W}$ and the collection of \y{sl}s at the $x$-th viewpoint is denoted as $\mathcal{I}_x$. We further impose the following mild assumption on the \y{ns}/\y{nsc} setup:
\begin{assume}\label{assum_1}
    The $Z$-coordinates of the every 3D keypoint observed by every $x$-th viewpoint, expressed in their respective reference frame centred at $\mathbf{C}_x$, is in $\mathbb{R}_+$.
\end{assume}

\noindent The input correspondences from the shape, expressed as projected 2D points, are always normalised with the known intrinsics, expressed as $\{\mathbf{v}_x \in \mathbb{R}^{2 \times n_x}\}$, obtained with some known extrinsics $({}^{C_x}\mathbf{R}_{W}  \in \mathbb{SO}(3), {}^{C_x}\mathbf{t}_{W} \in \mathbb{R}^3 )$, i.e.,
\begin{equation}
        \mathbf{v}_x = \Pi_x \Big( {}^{C_x}\mathbf{R}_{W} \mathbf{V} +  {}^{C_x}\mathbf{t}_{W}\Big) = \Pi \big( g_x( \mathbf{V} )\big), \quad \forall ~x \in [1,P],
\end{equation}
\noindent where $\Pi_x(\cdot)$ is the suitable projection function for the $x$-th viewpoint and $g_x(\cdot)$ represents the corresponding rigid transformation. 
\begin{figure}[t]
\centering
\begin{overpic}[unit=1mm,scale=1.2]{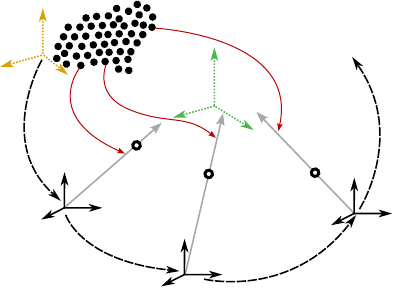}
\put(5,47.3){\small{$\mathbb{O}$}}
\put(39,37){\small{$\mathbb{W}$}}
\put(15,12.7){$\mathbf{C}_1$}
\put(38.0,-0.5){$\mathbf{C}_2$}
\put(72.5,12){$\mathbf{C}_3$}
\put(63,48){\small{$P$ cameras}}
\put(0,0){(a)}
\end{overpic}\hspace{20mm}
\begin{overpic}[abs,unit=1mm,scale=0.85]{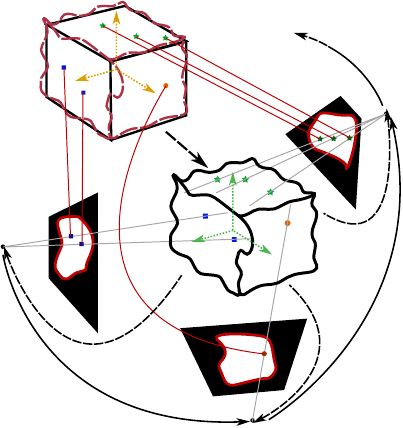}
\put(-2,27){$\mathbf{C}_1$}
\put(36,-2.5){$\mathbf{C}_2$}
\put(57,45){$\mathbf{C}_3$}
\put(37,57){\small{$P$ cameras}}
\put(5.5,40){$\Omega_1$}
\put(18,28){$\Omega_2$}
\put(33,52){$\Omega_3$}
\put(28,41){$(\mathbf{R}, \mathbf{t})$}
\put(16,10){$g_1^{-1}$}
\put(45.5,17){$g_2^{-1}$}
\put(49,26){$g_3^{-1}$}
\put(30,28){\small{$\mathbb{W}$}}
\put(17.8,50.7){\small{$\mathbb{O}$}}
\put(0,0){(b)}
\end{overpic} 
\caption{{\color{revCol1}(a) The \y{nsc} setup employs a generalised camera model, with inputs comprising correspondences between known 3D template points (black filled circles) in the object frame $\mathbb{O}$ and direction-only sightlines (\y{sl}s) defined in local coordinate frames of $P$ viewpoints. The deformed object lies in an unknown configuration within the world frame $\mathbb{W}$. The goal is to estimate the 3D locations of these template points in $\mathbb{W}$ (black hollow circles), along with the global positions ${\mathbf{C}_1, \ldots, \mathbf{C}_P}$ in $\mathbb{W}$. (b) The \y{ns} configuration instantiates this setup with $P$ calibrated perspective cameras and known poses ${g_1, \ldots, g_P}$. In the silhouette-augmented variant (cf. \cref{eqn_pt_based_silhBoost_final}), each image additionally provides silhouette information. The objective is to reconstruct keypoint positions in the world frame $\mathbb{W}$.}}
\end{figure}

\subsection{Correspondence-based \y{ns} and \y{nsc}}\label{pt_based}
We assume as input, $P$ set of two-way point correspondences between $\mathcal{I}_x$ and the mean shape $\bar{\mathbf{P}}$, one set for each viewpoint. The correspondences are specified as $P$ set of tuples $\Omega_x = \Big\{\big(j, j',\big), ~j \in [1,N], ~j' \in [1, n_x]\}$, with $|\Omega_x| = n_x$, and denotes that $\mathbf{Q}_j \in \mathbb{R}^3$ projects to the keypoint $\mathbf{p}_{x,j'} \in \mathbb{R}^2$ in $\mathcal{I}_x$ for all $x \in [1, P]$.

\subsubsection{Problem statement for \y{ns}}\label{subsec_nsc_first} 

We wish to determine the optimal shape basis weights $\{w_i\}$ and the object to world coordinate transformation $\mathbf{R}$ such that $g_x(\mathbf{V}_j)$ projects to $\mathbf{p}_{x,j_x}$ in $\mathcal{I}_x$, for all $x \in [1,P]$. In \y{ns}, all $\{g_x\}$ are known. Posing the projective requirements as an L$_1$-cost, the problem statement for \textit{correspondence-based \y{nsc} with known extrinsics} is:
\begin{equation}\label{eqn_pt_based}
    \begin{gathered}
       \min_{\mathbf{R}, \mathbf{t}, \{w_i\}}  \rho(\bar{\mathbf{P}}, \mathbf{Q}) + \sum_{x = 1}^P \sum_{(j,j') \in \Omega_x} \gamma\Big(\Pi\big(g_x(\mathbf{V}_j)\big), \mathbf{p}_{x,j'}\Big), \\
        \text{s.t.:} \quad \mathbf{V}_j = \mathbf{R} \mathbf{Q}_j + \mathbf{t}, ~\mathbf{Q} = \Big(\bar{\mathbf{P}} + \sum_{i=1}^M w_i \mathbf{P}_{i}\Big), \quad \mathbf{R} \in \mathbb{SO}(3).
     \end{gathered}
\end{equation}
\noindent where $\rho(\bar{\mathbf{P}}, \mathbf{Q})$ is a regularizer preserving local rigidity, defined as $\rho(\bar{\mathbf{P}}, \mathbf{Q}) = \sum_i w_i^2$. Importantly, apart from imparting local rigidity on the reconstructed shape, $\rho(\bar{\mathbf{P}}, \mathbf{Q})$ also acts as an object pose prior in \cref{eqn_pt_based}, which can be verified by noticing that removal of $\rho(\bar{\mathbf{P}}, \mathbf{Q})$ from \cref{eqn_pt_based} results in multiple ambiguous solutions to $(\mathbf{R}, \mathbf{t})$. On the other hand, $\gamma\Big(\Pi\big(g_x(\mathbf{V}_j)\big), \mathbf{p}_{x,j'}\Big)$ is a projective cost penalising the deviation of the projected model point $\Pi\big(g_x(\mathbf{V}_j)\big)$ from $\mathbf{p}_{x,j'}$, which we define using the L$_1$-norm of a vector product:
\begin{equation}\label{reproj_cost}
    \gamma\Big(\Pi\big(g_x(\mathbf{V}_j)\big), \mathbf{p}_{x,j'}\Big) = \| \big(g_x(\mathbf{V}_j) - \mathbf{C}_x\big)  \times \mathfrak{d}_{x,j'}\|_1,
\end{equation}
where $\mathfrak{d}_{x,j'} \in \mathbb{S}^2$ is the directional component of Pl\"{u}cker vectors from the viewpoint centre directed towards the input keypoint $\mathbf{p}_{x,j'}$ on the image plane and $\mathbf{C}_x \in \mathbb{R}^3$ is the $x$-th viewpoint centre; $(\mathfrak{d}_{x,j'}, \mathbf{C}_x)$ are data terms trivially derivable from $\{\mathbf{p}_x\}$ and $\{g_x\}$. However, a simple but important observation from \cref{eqn_pt_based} is as follows:
\begin{theorem}\label{theorem_sft_1}
    The solution to \cref{eqn_pt_based} is identical to the classical one-view \y{sft} problem.
\end{theorem}
\begin{proof}
    First we assign a arbitrary viewpoint, say the first viewpoint, as an `anchor'. Now, let the relative transformation of the $x$-th viewpoint w.r.t the anchor be $\tilde{\mathbf{R}}_x, \tilde{\mathbf{t}}_x$. We transform all the directional component of Pl\"{u}cker vectors of the non-anchor viewpoints as $\tilde{\mathbf{R}}_x^{\top} \mathfrak{d}_{x,j'}$ and set $\mathbf{C}_x = \mathbf{0}_3, \forall x$. Therefore, the $j$-th keypoint of the transformed shape $\tilde{\mathbf{V}}$ is given as:
    \begin{equation}
        \tilde{\mathbf{V}}_j = \tilde{\mathbf{R}}_x^{\top} (\mathbf{R}\mathbf{Q}_j + \mathbf{t}) - \tilde{\mathbf{R}}_x^{\top}\tilde{\mathbf{t}}_x.
    \end{equation}
\noindent Since all $\{\tilde{\mathbf{R}}_x, \tilde{\mathbf{t}}_x\}$ are known in the \y{ns} setup, the proof is complete.
\end{proof}
\noindent An important outcome of \cref{theorem_sft_1} is:
\begin{lemma}\label{lem_gsft_min}
    A sufficient criterion for achieving $\eta$ accuracy from \cref{eqn_pt_based} is $\sum_x n_x = \tau$, if \y{emnc} applies to the deformable object being reconstructed and \cref{assum_basic} holds. 
\end{lemma}
\noindent The proof of \cref{lem_gsft_min} is a straightforward combination of \cref{assum_basic} with \cref{theorem_sft_1}, hence omitted, but \cref{lem_gsft_min} leads us to an interesting observation about the \y{ns} problem setup. Namely, there are no minimum requirements for the minimum number of correspondences from each viewpoint (the only requirement of $\sum_x n_x = \tau$ is globally across viewpoints), therefore even one correspondence from one camera might add to the accuracy of \y{ns}, which is an useful outcome. Indeed, we later offer experimental validation of \cref{lem_gsft_min} in \cref{sec_gsft}.

But first, we present the solution to \cref{eqn_pt_based}, which remains a non-convex problem due to the non-linear constraints on $\mathbf{R}$ and the final convex solution needs some additional discussion, expanded below in \cref{sol_ns}. 

% The world to viewpoint transformation $g_x(\cdot)$ and world to object transformation $(\mathbf{R}, \mathbf{t})$ could be combined into one per-viewpoint transformation, but we prefer the formulation of \cref{eqn_pt_based} since it allows for easy separation of the \y{ns} and \y{nsc} problem. Moreover, if the object to object pose is also known, removing the object pose parameters $(\mathbf{R}, \mathbf{t})$ from \cref{eqn_pt_based} is trivial.

\subsubsection{Solution for \y{ns}}\label{sol_ns} 
There is but one challenge for posing the solution of \cref{eqn_pt_based} as a convex optimisation, we need to linearise the rotational constraints in \cref{eqn_pt_based}, such that \cref{reproj_cost} can also be linearized. This exists many well-known methods for linearizing rotation in the context of \y{sft}, but we offer a brief recap in the following paragraph.

\vspace{2mm}
\noindent \textbf{Parameterising rotations}.~When considering rotations as a parameter of optimisation, there exists many options including axis-angle representations \citep{parsons1995inability, giulietti2007optimal, chen2014rotation}, representation on the Lie manifolds \citep{jacobson1979lie, govindu2004lie, gilmore2006lie}, quaternions \citep{zhang1997quaternions, ude1998nonlinear, brambley2020unit} and their \y{sdp} based relaxations \citep{yang2019quaternion}, representation on rotation or Stiefel manifolds \citep{krishnan2007optimisation, ma1998motion, kumar2018scalable, boumal2014manopt}, and \y{sdp} based relaxations of rotation matrices \citep{chaudhury2015global, briales2017convex, iglesias2020global}. The existence of Lasserre's hierarchy to linearise \y{sos} polynomials \citep{anjos2011handbook} helps the case of \y{sdp} based relaxations, since it conveniently fits a wide variety of costs and constraints, which is essential for solving \cref{eqn_pt_based}. 

\vspace{2mm}
\noindent \textbf{\y{sdp} solution.}~We chose to utilise the \y{sdp} based approach to linearise the rotational constraints. First we define a Gram matrix $\Delta_e = \xi_e \xi_e^{\top}$, where $\xi_e  = \big(1, \mathrm{vec}(\mathbf{R})^{\top}, w_1, \hdots, w_M\big)^{\top}$. Following the strategy of Lasserre's hierarchy for \y{sos} polynomials, we can write all the costs/constraints in \cref{eqn_pt_based} as linear combination of the elements of $\Delta_e$. $\Delta_e$ is a Gram matrix, therefore a \y{psd} matrix by construction, hence ideally suited to be posed as an \y{sdp} problem. Therefore, one initial \y{sdp} based solution for the problem in \cref{eqn_pt_based} is:
\begin{equation}\label{eqn_init_sdp_known}
    \begin{gathered}
        \min_{\Delta_e, \mathbf{t}} ~\epsilon \omega_a(\Delta_e) + \sum_{x = 1}^P \sum_{(j,j') \in \Omega_x}\omega_b \Big(\Delta_e, \mathbf{p}_{x,j'}, \bar{\mathbf{P}}_j, \{\mathbf{P}_{i,j}\}, \mathbf{t}\Big), \\ \text{s.t.:} \quad \omega_{c, r}(\Delta_e) = 1, ~\omega_{d, r}(\Delta_e) = 0, ~\Delta_{e,1,1} = 1, \Delta_e \in \mathcal{S}_+^{10+M} \quad \forall ~r \in [1,3], ~i \in [1,M].
    \end{gathered}
\end{equation}
% In \cref{eqn_init_sdp_known}, $\omega_a$, $\omega_b$, $\{\omega_c\}$ and  $\{\omega_d\}$ are sampling functions that create specific linear combinations of $\Delta_e$ (details in \cref{appen_sample}), we explain their significance sequentially: $\omega_a$ leads to $\varrho(\bar{\mathbf{P}}_j, \mathbf{Q}_j)$ from the diagonal elements of $\Delta_e$, $\omega_b \Big(\Delta_e, \mathbf{p}_{x,j'}, \bar{\mathbf{P}}_j, \{\mathbf{P}_{i,j}\}, \mathbf{t}\Big)$ leads to the reprojection cost in \cref{reproj_cost}, $\{\omega_c\}$ and $\{\omega_d\}$ is derived from the rotational constraint $\mathbf{R}\mathbf{R}^{\top} = \mathbbm{1}_3$. Specifically, the reprojection cost, $\omega_b$, directly stems from the generalised camera model introduces in \cref{intro_gen}, as opposed to standard perspective camera.
\noindent {\color{revCol1} In \cref{eqn_init_sdp_known}, the functions $\omega_a$, $\omega_b$, ${\omega_c}$, and ${\omega_d}$ are sampling operators that generate specific linear combinations of the entries of $\Delta_e$ (see \cref{appen_sample} for details). Their roles are as follows: $\omega_a$ extracts the diagonal elements of $\Delta_e$ to yield $\varrho(\bar{\mathbf{P}}_j, \mathbf{Q}j)$; $\omega_b\Big(\Delta_e, \mathbf{p}{x,j'}, \bar{\mathbf{P}}j, {\mathbf{P}{i,j}}, \mathbf{t}\Big)$ corresponds to the reprojection cost described in \cref{reproj_cost}; and the sets ${\omega_c}$ and ${\omega_d}$ arise from enforcing the rotational constraint $\mathbf{R}\mathbf{R}^{\top} = \mathbbm{1}_3$. Notably, the reprojection term $\omega_b$ is derived from the generalised camera model introduced in \cref{intro_gen}, in contrast to the conventional perspective projection model.}

\vspace{2mm}
\noindent\textbf{Low-rank solution}. The presence of noise or shortcomings of the deformation model (e.g.: \y{ssm} with inadequately learned basis) can cause \y{sdp} formulations of the nature of \cref{eqn_init_sdp_known} to converge to spurious high-rank solutions, i.e., $\mathrm{rank}(\Delta_e) > 1$. This is caused by failure of strict feasibility in cases with such data or model issues \cite{roux2018validating} and remains an undesirable side-effect of such \y{sdp} formulations. As a remedy, one tries to minimise the rank of the \y{psd} matrix, $\Delta_e$ in our case, in conjunction with all the other costs and constraints. Rank being non-convex, the usual alternative is to minimise its convex conjugate, the nuclear norm, which is equal to the trace for a \y{psd} matrix. Thankfully, the trace also occurs in the formulation of our rigidity prior, simplifying our primal in the optimisation.
% This leads to our following proposition:
% \begin{prop}\label{prop_trace_arap}
%     For some $\epsilon' \in (0, \epsilon]$, $\epsilon' \mathrm{tr}(\Delta_e)$ is the convex conjugate of both $\mathrm{rank}(\Delta_e)$ and the \y{arap} regularizer $\rho(\bar{\mathbf{P}}_j, \mathbf{Q}_j)$.
% \end{prop}
% \begin{proof}
%     Trace of \y{psd} matrices being a convex conjugate of rank is well known due to \cite{fazel2002matrix, candes2008exact}. Matrix trace being a convex conjugate of \y{arap} is from \cref{lemm_arap_aurrogacy}.
% \end{proof}

Combining trace minimisation with \cref{eqn_init_sdp_known} leads to our final solution to \cref{eqn_pt_based}:
\begin{equation}\label{eqn_final_sdp_known}
    \begin{gathered}
        \min_{\Delta_e, \mathbf{t}} ~\epsilon' \mathrm{tr}(\Delta_e) + \sum_{x = 1}^P \sum_{(j,j') \in \Omega_x}\omega_b \Big(\Delta_e, \mathbf{p}_{x,j'}, \bar{\mathbf{P}}_j, \{\mathbf{P}_{i,j}\}, \mathbf{t}\Big), \\ \text{s.t.:} \quad \omega_{c, r}(\Delta_e) = 1, ~\omega_{d, r}(\Delta_e) = 0, ~\Delta_{e,1,1} = 1, \Delta_e \in \mathcal{S}_+^{10+M} \quad \forall ~r \in [1,3], ~i \in [1,M].
    \end{gathered}
\end{equation}
\noindent \Cref{eqn_final_sdp_known} is a convex optimisation problem, solved with standard SDP solvers using the CVX library \citep{cvx, gb08} with Mosek \citep{mosek} on Matlab.

Due to \cref{theorem_sft_1}, it may appear that nothing new is achieved in \cref{eqn_final_sdp_known} over and above the many existing solutions to classical \y{sft}. However, such a conclusion would be misleading because: 1) \cref{eqn_final_sdp_known} already enables us to solve \y{sft} with diverse projection models (e.g.: object pose from \y{sft} with perspective projection was previously unknown, so is the combination of perspective and orthographic cameras for \y{sft}), and 2) \cref{eqn_final_sdp_known} sets the background for studying the significantly more ill-posed \y{nsc} problem in \cref{sec_unknown}.

\subsubsection{Problem statement for \y{nsc}} \label{sec_unknown}

In this second sub-variant, i.e., correspondence based \y{nsc}, we follow the same problem setup as \cref{subsec_nsc_first}, but we assume all $\{g_x\}$ are unknown. We denote the pose of each viewpoint by its position w.r.t $\mathbb{W}$ given by $\mathbf{C}_x$ as before and its unknown orientation $\mathbf{R}_x$. We therefore obtain the following problem statement for the \textit{correspondence-based \y{nsc}}:
\begin{equation}\label{eqn_pt_based_unext}
    \begin{gathered}
       \min_{(\mathbf{R}, \mathbf{t}),\{\mathbf{R}_x\},\{\mathbf{C}_x\},\{w_i\} }   \rho(\bar{\mathbf{P}}, \mathbf{Q}) + \sum_{x = 1}^P \sum_{(j,j') \in \Omega_x} \gamma_P\Big(\Pi\big(\mathbf{R}\mathbf{Q}_j + \mathbf{t})\big), \mathbf{p}_{x,j'}\Big), \\
        \text{s.t.:} \quad \mathbf{Q} = \Big(\bar{\mathbf{P}} + \sum_{i=1}^M w_i \mathbf{P}_{i}\Big), \quad \{\mathbf{R}_x \in \mathbb{SO}(3), ~\forall x \in [1,P]\}.
     \end{gathered}
\end{equation}
\noindent \Cref{eqn_pt_based_unext} is identical to \cref{eqn_pt_based} except modifying the reprojection cost to 
\begin{equation}\label{eqn_modif_rep}
    \gamma_P\Big(\Pi\big(\mathbf{R}\mathbf{Q}_j + \mathbf{t})\big), \mathbf{p}_{x,j'}\Big) = \| \big(\mathbf{R}\mathbf{Q}_j + \mathbf{t} - \mathbf{C}_x\big) \times \mathbf{R}_x \mathfrak{d}_{x,j'}\|_1
\end{equation}
\noindent \Cref{eqn_modif_rep} is needed because for \y{nsc} because the camera poses $\big(\{\mathbf{R}_x\}, \{\mathbf{C}_x\}\big)$ are unknown. For the case of $P=1$, the problem statement in \cref{eqn_pt_based_unext} becomes the classical one-view \y{sft} problem with clear analogies to similar problems in the literature \cite{hejrati2012analyzing, ramakrishna2012reconstructing, zia2013detailed, wang2014robust, zhu2015single, zhou20153d, bartoli2015shape, chhatkuli2016stable, fuentes2021texture, casillas2019equiareal, gallardo2020shape, parashar2019local, parashar2017isometric}. But importantly, although the definition of $\rho$ remains identical to \cref{eqn_pt_based}, it has an important side-effect, namely, it acts as a object-pose prior in \cref{eqn_pt_based_unext}, similar to \y{ns}. 

\subsubsection{Solution for \y{nsc}} \label{sec_nsc_sol}
The key idea here is to solve $P$ simultaneous one-view \y{sft} problems with rigidity constraints between all the $P$ shapes, but to understand the necessity of such a formulation, it is important to set forth some intricacies of \cref{eqn_pt_based_unext}. In this subsection, we first analyse these intricacies and then slowly proceed to the final solution formulation. 

While the key problem in \cref{eqn_pt_based_unext} may sound intuitive, there remains some caveats to its solvability, beginning with the following:

\begin{theorem}[Gauge freedom in \cref{eqn_pt_based_unext}] \label{threorem_1} The viewpoint absolute pose $(\{\mathbf{R}_x\},\{\mathbf{C}_x\})$ is not uniquely solvable in \cref{eqn_pt_based_unext}, i.e., there exists multiple ambiguous rigid transformations that solve \cref{eqn_pt_based_unext}. 
\end{theorem}
\begin{proof}
   See \cref{theo_2}.
\end{proof}

\noindent Thankfully, there exists a way put for resolving this ambiguity from \cref{threorem_1}: this can be done by first assuming that all viewpoints are centred at origin, i.e., $\mathbf{C}_x = \mathbf{0}_{3 \times 1}$ and then, recovering the actual pose of viewpoints in a second step. We formalize this new strategy for resolving the pose ambiguity in \cref{eqn_pt_based_unext} as follows:
\begin{prop}[Gauge fixing of \cref{eqn_pt_based_unext}] \label{prop_gauge_fix}
     $(\{\mathbf{R}_x\},\{\mathbf{C}_x\})$ becomes uniquely solvable in \cref{eqn_pt_based_unext} by considering all viewpoints centred at origin, i.e., $\mathbf{C}_x = \mathbf{0}_{3 \times 1}$ and suitably inverse-transforming the shape $\mathbf{Q}$ to match the origin, given by the reformulation of \cref{eqn_pt_based_unext} as:
 \begin{equation}\label{eqn_pt_based_unext_reform}
   \begin{gathered}
       \min_{(\mathbf{R}, \mathbf{t}),\{(\mathbf{R}_x, \mathbf{t}_x)\},\{w_i\} }  \rho(\bar{\mathbf{P}}, \mathbf{Q}) + \sum_{x = 1}^P \sum_{(j,j') \in \Omega_x} \gamma\bigg(\Pi\Big( \mathbf{R}_x^{\top}\big(\mathbf{R}\mathbf{Q}_j + \mathbf{t} - \mathbf{t}_x)\big)\Big), \mathbf{p}_{x,j'}\bigg), \\
        \text{s.t.:} \quad \mathbf{Q} = \Big(\bar{\mathbf{P}} + \sum_{i=1}^M w_i \mathbf{P}_{i}\Big), \quad \mathbf{C}_x = \mathbf{0}_{3 \times 1}, \quad \{\mathbf{R}_x \in \mathbb{SO}(3), ~\forall x \in [1,P]\}.
     \end{gathered}
\end{equation}
\noindent where the reprojection cost $\gamma(\cdot, \cdot)$ is given by \cref{reproj_cost}.
\end{prop}
\begin{proof}
    See \cref{prop_proof}.
\end{proof}

\noindent With the reformulation of \cref{prop_gauge_fix}, the \y{nsc} problem is now solved in two consecutive steps: \rom{1}) we use the object to viewpoint inverse transformation $\{(\mathbf{R}_x^{\top}, -\mathbf{R}_x^{\top}\mathbf{t}_x)\}$ of the template $\{\mathbf{Q}\}$ in each viewpoint's local reference frame and solve multiple, simultaneous \y{sft} problems convexly at each viewpoint with a rigidity constraint between each of the transformed shape pose, and \rom{2}) with the solution of these simultaneous \y{sft} problems, we assign the first viewpoint as the `anchor' and estimate the relative pose of all viewpoints with a simple inverse-of-inverse transformation, i.e., $\{(\mathbf{R}_x, \mathbf{t}_x)\}$. Importantly, these two steps described above are non-iterative, arriving at the global solution in one pass. 

To remove redundancies from our parameters of optimisation, we assign $\mathbf{R}_x^{\top} \leftarrow \mathbf{R}_x^{\top}\mathbf{R}$ and $-\mathbf{t}_x \leftarrow \mathbf{R}_x^{\top}(\mathbf{t} - \mathbf{t}_x)$ with a slight abuse of notation. For \y{nsc}, we can rewrite 
\cref{eqn_pt_based_unext} as:
\begin{equation}\label{eqn_pt_basd_unextpt1}
    \begin{gathered}
       \min_{\{\mathbf{R}_x\}, \{\mathbf{t}_x\},\{w_i\}}  \sum_{x=1}^P \rho(\bar{\mathbf{P}}, \mathbf{Q}_{x}) + \sum_{x = 1}^P \sum_{(j,j') \in \Omega_x} \gamma\Big(\Pi\big(\mathbf{R}_x\mathbf{Q}_j - \mathbf{t}_x)\big), \mathbf{p}_{x,j'}\Big), \\
        \text{s.t.:} \quad \mathbf{Q} = \Big(\bar{\mathbf{P}} + \sum_{i=1}^M w_i \mathbf{P}_{i}\Big), \forall i \in [1,M],    
        ~ \{\mathbf{R}_x \in \mathbb{SO}(3), ~\forall x \in [1,P]\}.
     \end{gathered}
\end{equation}

\noindent \textbf{\y{sdp} formulation}. In \cref{eqn_pt_basd_unextpt1}, we have $P$-copies of $\{\mathbf{R}_x\}, \{\mathbf{t}_x\}$ which are multiplied with $\{w_i\}$ in $\gamma(\cdot, \cdot)$. To linearise these quadratic polynomials in $\{\mathbf{R}_x\}, \{\mathbf{t}_x\}$, and $\{w_i\}$ in \cref{eqn_pt_basd_unextpt1}, we propose the solution in two steps:
\begin{itemize}
    \item We first duplicate $\{w_i\}$ into $P$ redundant copies $\{\mathbf{w}_x \in \mathbb{R}^M, ~\forall x \in [1, P]\}$ and introduce new equality constraints between each of these weights
    \item With these $P$ redundant copies of $\{\mathbf{w}_x\}$, we introduce $P$ new Gram matrices that linearises the cost $\gamma$
\end{itemize}

\noindent Thus, for the first step of introducing $P$ redundant copies $\{\mathbf{w}_x \in \mathbb{R}^M\}$ into the problem at \cref{eqn_pt_basd_unextpt1}, the new problem writes as:
\begin{equation}\label{eqn_pt_based_unext_2}
    \begin{gathered}
       \min_{\{\mathbf{R}_x\}, \{\mathbf{t}_x\},\{\mathbf{w}_x\}}  \sum_{x=1}^P \rho(\bar{\mathbf{P}}, \mathbf{Q}_{x}) + \sum_{x = 1}^P \sum_{(j,j') \in \Omega_x} \gamma\Big(\Pi\big(\mathbf{R}_x\mathbf{Q}_j - \mathbf{t}_x)\big), \mathbf{p}_{x,j'}\Big), \\
        \text{s.t.:} \quad \mathbf{Q}_x = \Big(\bar{\mathbf{P}} + \sum_{i=1}^M \mathbf{w}_{x,i} \mathbf{P}_{i}\Big), ~\mathbf{w}_{x,i} = \mathbf{w}_{x',i}, \forall i \in [1,M], (x, x') \in [1, P], x \neq x',        
        \\ \{\mathbf{R}_x \in \mathbb{SO}(3), ~\forall x \in [1,P]\},
     \end{gathered}
\end{equation}
\noindent where the equality constraints $\mathbf{w}_{x,i} = \mathbf{w}_{x',i}$ ensures that the redundant copies of weights remain identical to each other. Next, we re-utilise the \y{sdp} formulation of \cref{eqn_final_sdp_known} for solving \cref{eqn_pt_based_unext_2}, but parametrise the optimisation with $P$ \y{psd} matrices $\{\Delta_{x}\}, ~\forall x \in [1,P]$ for every $P$ viewpoints with $\Delta_x = \xi_x \xi_x^{\top}$, where 
\begin{equation}
    \xi_x  = \big(1, \mathrm{vec}(\mathbf{R}_x)^{\top}, \mathbf{w}_{x,1}, \hdots, \mathbf{w}_{x,M}\big)^{\top},
\end{equation}
\noindent the translation parameter $\mathbf{t}_x$ is repeated $P$ times as well. As an additional requirement, we need the object to necessarily lie in-front of the viewpoint's centre, therefore the minimum depth along $Z$-axis in the viewpoint reference frame is constrained at $f_x$, an empirical minimum depth (which can be the focal distance of the $x$-th viewpoint if it happens to represent a perspective camera). Therefore, the solution of the first sub-problem writes as:
\begin{equation}\label{eqn_final_sdp_unknown}
    \begin{gathered}
        \min_{\{\Delta_{x}\}, \{\mathbf{t}_x\}} ~\sum_{x=i}^P\epsilon' \mathrm{tr}(\Delta_x) + \sum_{x = 1}^P \sum_{(j,j') \in \Omega_x}\omega_b \Big(\Delta_x, \mathbf{p}_{x,j'}, \bar{\mathbf{P}}_j, \{\mathbf{P}_{i,j}\}, -\mathbf{t}_x\Big), \\ \text{s.t.:} \quad \omega_{c, r}(\Delta_x) = 1, ~\omega_{d, r}(\Delta_x) = 0, ~\Delta_{x,1,1} = 1, ~\mathbf{t}_{x,3} \geq f_x, ~\Delta_x \in \mathcal{S}_+^{10+M}, ~\Delta_{x, 1, [11:10+M]} = \Delta_{x', 1, [11:10+M]}, \\ \forall ~r \in [1,3], ~i \in [1,M], ~(x, x') \in [1,P], ~x \neq x'.
    \end{gathered}
\end{equation}
% The last constraint $\Delta_{x, 1, [11:10+M]} = \Delta_{x', 1, [11:10+M]}$ ensures that the basis weights obtained in all the views remain identical, thereby ensuring rigidity between viewpoints. \Cref{eqn_final_sdp_unknown} is a convex problem and its implementation follows the same strategy of \cref{eqn_final_sdp_known}. As in \y{ns}, $\omega_b$ allows us to define the reprojection cost w.r.t any arbitrary line in space, thus staying true to the generalised camera setup.
\noindent {\color{revCol1}The final constraint, $\Delta_{x, 1, [11:10+M]} = \Delta_{x', 1, [11:10+M]}$, enforces consistency of basis weights across all views, thereby maintaining rigidity across viewpoints. The formulation in \cref{eqn_final_sdp_unknown} remains convex and follows an implementation strategy analogous to that of \cref{eqn_final_sdp_known}. As in \y{ns}, the operator $\omega_b$ facilitates the definition of reprojection error with respect to arbitrary lines in space, faithfully adhering to the generalised camera framework and extending beyond conventional perspective models.}

\vspace{2mm}
\noindent \textbf{Recovering the pose of each $\{\mathbf{C}_x\}$}. Given that the optimal deformed shape from the solution of \cref{eqn_final_sdp_unknown} is denoted by $\mathbf{Q}_x^{\ast} = \mathbf{Q}^{\ast}$ (since all shapes are constrained to be identical) and the optimal poses are denoted by $\big\{(\mathbf{R}^{\ast}_x, \mathbf{t}^{\ast}_x),~\forall x \in [1,P]\big\}$, the pose of each viewpoint w.r.t object and world reference is simply the inverse pose $\Big\{\big((\mathbf{R}^{\ast}_x)^{\top}, -(\mathbf{R}^{\ast}_x)^{\top}\mathbf{t}^{\ast}_x\big),~\forall x \in [1,P]\Big\}$ w.r.t the object coordinate.

\vspace{2mm}
\noindent \textbf{Side effect of our \y{sdp} formulation}.~Due to \cref{prop_gauge_fix} and the subsequent reformulation in \cref{eqn_pt_based_unext_2}, we have introduced a new side-effect into the solution which indeed seems unavoidable, namely:
\begin{lemma}\label{lem_acc}
    A necessary condition for achieving $\eta$ accuracy from \cref{eqn_pt_based_unext_2} is $n_x = \tau, ~\forall x \in [1,P]$, if \y{emnc} applies to the deformable object being reconstructed following \cref{assum_basic} and the primal cost in \cref{eqn_pt_based_unext_2} models the deformation accurately.
\end{lemma}
\begin{proof}
See \cref{proof_lem_3_app}.
\end{proof}

\noindent However, \cref{lem_acc} is unsurprising due to the ill-posedness of the original \y{nsc} problem in \cref{eqn_pt_based_unext} and does not offer any practical drawback apart from requiring `sufficient' correspondences from each viewpoint.

% writes as the following feasibility problem:
% \begin{equation}\label{second_sub_prob}
% \begin{gathered}
%   \min_{\{\mathbf{R}_x\}, \{\mathbf{t}_x\} }  0, \\ \text{s.t.:} \quad \mathbf{R}_1 \mathbf{Q}^{\ast} + \mathbf{t}_1 = \mathbf{R}_{x'} \mathbf{Q}^{\ast} + \mathbf{t}_{x'},~\forall x' \in [2, P] ~\mathbf{R}_1 = \mathbf{R}'^{\ast}_1,~\mathbf{t}_1 = \mathbf{t}'^{\ast}_1, ~\{\mathbf{R}_{x'}\} \in \mathbb{SO}(3), ~\forall x' \in [2,P].
%   \end{gathered}
% \end{equation}
% \noindent While a principled solution to \cref{second_sub_prob} can certainly be obtained with another round of \y{sdp}, we simplify the solution with the standard \y{svd} based solution to the classic Wahba's problem \cite{wahba1965least} of finding optimal rotation between pointclouds, since, by construction (using the same $\mathbf{Q}^{\ast}$), there cannot exist noise or outliers in the data. For the sake of completeness, we recap this well-known closed-form solution in \cref{sec_wahba}.

\subsection{Silhouette-boosted \y{ns} and \y{nsc}}\label{silh_formation_model}
{\color{revCol1}While solving the \y{ns}/\y{nsc} problem in conjunction with silhouette information, we define the input silhouette as direction component of Pl\"{u}cker vectors $\{\mathcal{C}_x \in \mathbb{R}^{3 \times S_x}\}$ defined by a sequence of $S_x$ direction vectors from $\mathcal{I}_x$ for all $x \in [1,P]$. Importantly, we assume the presence of unoccluded silhouettes as a necessary input to our method. We denote by $\mathfrak{g}(\cdot)$, a function that acts on $g_x( \mathbf{V} )$ to produce its silhouette vectors $\tilde{\mathcal{C}}_x \in \mathbb{R}^{3 \times S'_x}$, defined by a sequence of $S'_x \leq S_x$ direction vectors from $\mathcal{I}_x$. We define a silhouette map $\mu_x:\mathbb{R}^{3 \times S_x} \mapsto \mathbb{R}^{3 \times S'_x}$ acting on $\mathcal{C}_x$ to produce a sub-sequence of $S'_x$ vectors from $\mathcal{C}_x$. Therefore a similarity measure between the input silhouette vectors and the silhouette vectors from the modelled object can be defined as $\kappa\big(\mu_x(\mathcal{C}_x), \tilde{\mathcal{C}}_x\big)$ which penalises deviation of $\mu_x(\mathcal{C}_x)$ from $\tilde{\mathcal{C}}_x$.}

%\noindent \textcolor{gray}{We quickly delineate some properties of such a silhouette similarity measure as $\kappa$ (without loss of generality, we follow the problem setup given in \cref{subsec_nsc_first}).
% \begin{lemma}\label{lem_ambiguity_sol}
%     The global optima of the solution space of a problem of the nature of:
%     \begin{equation}\label{hypo_prob}
%         \min_{\mathbf{R}, \mathbf{t}, \{w_i\}} \sum_{x=1}^P\kappa\big(\mu_x(\mathcal{C}_x), \tilde{\mathcal{C}}_x\big)
%     \end{equation}
%     is not guaranteed to be unique.
% \end{lemma}
% \begin{proof}
%     If $(\hat{\mathbf{R}}, \hat{\mathbf{t}}, \{\hat{w}_i\})$ be a global minima for \cref{hypo_prob}, for some arbitrary rotation $\mathbf{R}'$ and a very small translation $\mathbf{t}'$, there can exist weights $\{w'_i\}$ such that:
%     \begin{equation}
%         \hat{\mathbf{R}}\Big(\bar{\mathbf{P}} + \sum_{i=1}^M \hat{w}_i \mathbf{P}_{i}\Big) + \hat{\mathbf{t}} \approx \mathbf{R}'\hat{\mathbf{R}}\Big(\bar{\mathbf{P}} + \sum_{i=1}^M w'_i \mathbf{P}_{i}\Big) + \hat{\mathbf{t}} + \mathbf{t}',
%     \end{equation}
%     therefore uniqueness of globally optimal solution cannot be guaranteed.
% \end{proof}}
Silhouette matching is a problem with a highly ambiguous solution space. If $(\hat{\mathbf{R}}, \hat{\mathbf{t}}, \{\hat{w}_i\})$ is a global optima for 
\begin{equation}\label{hypo_prob}
        \min_{\mathbf{R}, \mathbf{t}, \{w_i\}} \sum_{x=1}^P\kappa\big(\mu_x(\mathcal{C}_x), \tilde{\mathcal{C}}_x\big),
\end{equation}
\noindent there can certainly exist some arbitrary rotation $\mathbf{R}'$ and translation $\mathbf{t}'$ with associated weights $\{w'_i\}$ such that:
\begin{equation}\label{e_ssm_silh}
    \hat{\mathbf{R}}\Big(\bar{\mathbf{P}} + \sum_{i=1}^M \hat{w}_i \mathbf{P}_{i}\Big) + \hat{\mathbf{t}} \approx \mathbf{R}'\hat{\mathbf{R}}\Big(\bar{\mathbf{P}} + \sum_{i=1}^M w'_i \mathbf{P}_{i}\Big) + \hat{\mathbf{t}} + \mathbf{t}'.
\end{equation}
Certainly, the same can be said about \y{sft} in general, which is one of the reasons why regularisers exist, i.e., to induce an unique solution to an otherwise ambiguous problem. But the combination of regularisers and silhouette similarity costs like \cref{hypo_prob} can be notoriously difficult to solve up to global optimality; no known convex solutions exist for such a problem. Given such a challenge for the purely silhouette based \y{sft} problem, we re-utilise a common strategy \cite{cheung2005shape, ilic2007implicit, bottino2001silhouette, cheung2003shape, da2010iterative, kong2017using, sunkel2007silhouette, saito2014model} for silhouette based reconstruction, i.e., we use silhouettes in conjunction with some keypoint correspondences, hence the term `silhouette-boosting'. 

\vspace{2mm}
\noindent \textbf{Unified problem statement}. Silhouette-boosted \y{sft} therefore only allows us to search for a solution that lies within a close vicinity of the solution for purely correspondence based \y{ns}/\y{nsc} in their respective solution space, eliminating the possibility for a global solution. Therefore, to simplify our methodological description, it is possible to merge the \y{ns} and \y{nsc} problem statements in case of silhouette-boosting. While the premise of \y{ns} remains unchanged, for \y{nsc} problem, we solve for the \y{nsc} solution using the method described in \cref{sec_unknown} with just the correspondences and obtain the viewpoint pose. Thereafter, the silhouette boosting for \y{nsc} is no different from \y{ns}.

\vspace{2mm}
\noindent \textbf{A model for silhouette formation}. We utilise the classical concept of `$\alpha$-shapes' \cite{edelsbrunner1983shape} for extracting silhouettes from a set of projected 2D points, which can be trivially derived from $\mathfrak{g}(\cdot)$. For a sufficiently large $N$, there must exist some $\alpha$ for which the direction vectors in $\mathfrak{g}(\hat{\mathbf{p}})$ are identical to the direction vectors comprising the true silhouette of an arbitrary 2D pointcloud $\hat{\mathbf{p}}$. Therefore, we add the following assumption to our problem setup:
\begin{assume}\label{assume_density}
    The \y{ssm} used as the object model in \cref{e_ssm_silh}: \rom{1}) has a sufficiently high $N$  such that the silhouette recovered by \cite{edelsbrunner1983shape} from the 3D projection of such a shape closely resembles the true silhouette, and \rom{2}) represents a single connected object.
\end{assume}
\noindent E.g.: in our experiments with silhouettes, typical values of $N$ were maintained at $>10^3$.
\Cref{assume_density} is mild since: \rom{1}) a very sparse object model can easily be densified up to arbitrary resolution as a simple pre-processing step, as long as the triangulated boundary of the model is known, \rom{2}) \cref{assume_density} puts no additional requirement on the keypoint correspondences, which are free to be highly scant without affecting the silhouette model.

\subsubsection{Problem statement for silhouette-boosted \y{ns}}\label{pron_silh_nsc}
Given the correspondences and silhouettes, our objective is therefore to do \y{ns} as in \cref{subsec_nsc_first} for the correspondences and minimise the silhouette similarity measure $\kappa(\cdot, \cdot)$ for the silhouettes. The alignment cost of the correspondences is weighted by a parameter $\lambda$ to promote alignment of the silhouettes. Therefore, our initial formulation for the \textit{silhouette-boosted \y{ns}} problem is:

\begin{equation}\label{eqn_pt_based_silhBoost}
    \begin{gathered}
       \min_{\mathbf{R}, \mathbf{t}, \{w_i\}, \{\mu_x\}}  \rho(\bar{\mathbf{P}}, \mathbf{Q}) + \lambda \sum_{x = 1}^P \sum_{(j,j') \in \Omega_x} \gamma\Big(\Pi\big(g_x(\mathbf{V}_j)\big), \mathbf{p}_{x,j'}\Big) + (1 - \lambda)\sum_{x=1}^P \sum_{q=1}^{S'_x} \kappa\bigg(\mu_x(\mathcal{C}_x), \mathfrak{g}\Big(\Pi\big(g_x(\mathbf{V})\big)\Big)\bigg) , \\
        \text{s.t.:} \quad \mathbf{V}_j = \mathbf{R} \mathbf{Q}_j + \mathbf{t}, ~\mathbf{Q} = \Big(\bar{\mathbf{P}} + \sum_{i=1}^M w_i \mathbf{P}_{i}\Big), \quad \mathbf{R} \in \mathbb{SO}(3), ~\lambda \in [0,1].
     \end{gathered}
\end{equation}
\noindent With $\kappa(\cdot, \cdot)$, we want to minimise a reprojection error between the silhouette of the projected object model and the input silhouette, therefore a convenient option is to reuse the alignment cost $\gamma(\cdot, \cdot)$. Thus, our silhouette-boosted \y{ns} problem writes as:
\begin{equation}\label{eqn_pt_based_silhBoost_final}
   \begin{gathered}
       \min_{\mathbf{R}, \mathbf{t}, \{w_i\}, \{\mu_x\}}   \rho(\bar{\mathbf{P}}, \mathbf{Q}) + \sum_{x = 1}^P \Bigg(\lambda \sum_{(j,j') \in \Omega_x} \gamma\Big(\Pi\big(g_x(\mathbf{V}_j)\big), \mathbf{p}_{x,j'}\Big) + (1 - \lambda)\sum_{q=1}^{S'_x} \gamma\bigg(\mathfrak{g}\Big(\Pi\big(g_x(\mathbf{V})\big)\Big), \mu_x(\mathcal{C}_x)\bigg)\Bigg) , \\
        \text{s.t.:} \quad \mathbf{V}_j = \mathbf{R} \mathbf{Q}_j + \mathbf{t}, ~\mathbf{Q} = \Big(\bar{\mathbf{P}} + \sum_{i=1}^M w_i \mathbf{P}_{i}\Big), \quad \mathbf{R} \in \mathbb{SO}(3), ~\lambda \in [0,1].
     \end{gathered}
\end{equation}
\noindent There are multiple sources of non-convexity in \cref{eqn_pt_based_silhBoost_final} including the aforementioned rotational constraints as well as the now introduced correspondence function $\{\mu_x\}$ and the silhouette function $\mathfrak{g}(\cdot)$.

\subsubsection{Solution for silhouette-boosted \y{ns}}\label{sec_synth_sol}
Our strategy for solving \cref{eqn_pt_based_silhBoost_final} is to break it down into sub-steps, all of which provides an optimal solution to their respective sub-problems, and to iterate over these steps leading to a global solution. We use an \y{sdp} formulation, similar to \cref{eqn_final_sdp_known}, for making the rotational constraints convex. Inside each iteration, beginning from an initial pose, we compute the silhouette and determine its correspondence with the object model and finally, refine the shape to fit the silhouette better.  We denote all values at the $t$-th iteration with a bracketed exponent $\cdot^{(t)}$. The steps of our iterative solution are explained below:

\noindent \textbf{\rom{1}. Initialisation}. First, we begin by disregarding the silhouette information and solving \cref{eqn_final_sdp_known} with just the input correspondences. This directly leads to an initial estimate of $\mathbf{R}^{(t)}$, $\mathbf{t}^{(t)}$, and $\{w_i^{(t)}\}$ for $t=1$.

\noindent \textbf{\rom{2}. Silhouette function}. Given the putative shape and pose from $\mathbf{R}^{(t)}$, $\mathbf{t}^{(t)}$, and $\{w_i^{(t)}\}$, we use $\alpha$-shapes to determine the direction vectors of silhouette $\tilde{\mathcal{C}}^{(t)}_x$, i.e.:
\begin{equation}\label{alg_eq_1}
    \tilde{\mathcal{C}}^{(t)}_x = \mathfrak{g}\bigg(\Pi\Big(g_x\big(\mathbf{R}^{(t)} (\bar{\mathbf{P}} + \sum_{i=1}^Mw_i^{(t)}\mathbf{P}_i)\big) + \mathbf{t}^{(t)}\Big)\bigg).
\end{equation}

\noindent \textbf{\rom{3}. Correspondence function}. The correspondence between $\tilde{\mathcal{C}}^{(t)}_x$ and $\mathcal{C}_x$ is given as the smallest vector product between the directions vectors in $\tilde{\mathcal{C}}^{(t)}_x$ and the input $\mathcal{C}_x$, the resulting sub-sampled direction vectors of $\mathcal{C}_x$ is denoted by $\mu_x^{(t)}(\mathcal{C}_x)$. We utilise $\mu_x^{(t)}(\mathcal{C}_x)$ to create a silhouette correspondence tuple $\vartheta_x^{(t)} = \{ (s, s') \in [1, S'_x]^2 \}$ such that every $\tilde{\mathcal{C}}^{(t)}_{x,s}$ and $\mathcal{C}_{x, s'}$ are in correspondence if $(s, s') \in \vartheta_x^{(t)}$ and $\varphi(s)$ gives the index of $\tilde{\mathcal{C}}^{(t)}_{x,s}$ in the mean shape $\bar{\mathbf{P}}$.

\noindent \textbf{\rom{4}. Silhouette-boosted shape and pose}. The final shape and pose of the object to be reconstructed is obtained for the $t$-th iteration as:
\begin{equation}\label{eqn_final_sdp_known_silh}
    \begin{gathered}
        \min_{\Delta_e^{(t)}, \mathbf{t}^{(t)}} ~\epsilon' \mathrm{tr}(\Delta_e^{(t)}) + \sum_{x = 1}^P \Bigg( \lambda \sum_{(j,j') \in \Omega_x}\omega_b \Big(\Delta_e^{(t)}, \mathbf{p}_{x,j'}, \bar{\mathbf{P}}_j, \{\mathbf{P}_{i,j}\}, \mathbf{t}^{(t)}\Big)  \\+ (1-\lambda) \sum_{(s,s') \in \vartheta_x^{(t)}}\omega_b \Big(\Delta_e^{(t)}, \mathcal{C}_{x,s'}, \bar{\mathbf{P}}_{\varphi(s)}, \{\mathbf{P}_{i,\varphi(s)}\}, \mathbf{t}^{(t)}\Big)   \Bigg), \\ \text{s.t.:} \quad \omega_{c, r}(\Delta_e^{(t)}) = 1, ~\omega_{d, r}(\Delta_e^{(t)}) = 0, ~\Delta_{e,1,1}^{(t)} = 1, ~\Delta_e^{(t)} \in \mathcal{S}_+^{10+M} \quad \forall ~r \in [1,3], ~i \in [1,M], ~\lambda \in [0,1],
    \end{gathered}
\end{equation}
\noindent and \cref{eqn_final_sdp_known_silh} is a convex solution.

Starting from the step \rom{1}, i.e., the initial solution from \cref{eqn_final_sdp_known}, we iterate over steps \rom{2}, \rom{3}, and \rom{4} till we reach the convergence criteria $\vartheta_x^{(t)} \backslash \vartheta_x^{(t-1)} = \emptyset$. We observe the following from our proposed iterative solution:
\begin{lemma}
    For some sufficiently large value of $t$ and for $P \geq 2$ with $|\Omega_x| \geq 1, ~\forall x \in [1,P]$, if at each iteration $rank(\Delta_e^{(t)}) = 1$ and the solution of \cref{eqn_final_sdp_known_silh} obeys Slater's strong duality, the proposed iterative solution always converges to some local solution of \cref{eqn_pt_based_silhBoost_final}
\end{lemma}
\begin{proof}
    The initialisation from \cref{eqn_final_sdp_known} is convex. The silhouette function $\mathfrak{g}(\cdot)$ follows the method of $\alpha$-shapes which is optimal for some given value of $\alpha$ \cite{edelsbrunner1983shape}. The correspondence function has a simple closed-form. The final shape and pose reconstruction in \cref{eqn_final_sdp_known_silh} is convex, the strong duality and rank-1 requirement ensures there are no spurious high-rank solution and verifies the problem is indeed feasible. Therefore we have split \cref{eqn_pt_based_silhBoost_final} into multiple globally optimal steps without altering the costs and constraints of \cref{eqn_pt_based_silhBoost_final}.  For $P=1$, the feasible set of \cref{eqn_pt_based_silhBoost_final} is unbounded due to the scaling effect, i.e., for $P=1$ all solutions are up to scale. But for $P > 1$, the scale is bounded, so is the solution set of \cref{eqn_pt_based_silhBoost_final} as long as $|\Omega_x| \geq 1, ~\forall x \in [1,P]$. Therefore, following well-known properties of block relaxation \cite{hong2015unified}, the solution is guaranteed to converge to some local solution.
\end{proof}

\noindent {\color{revCol1}The iterative silhouette-boosted \y{ns} solution can be thus summarized as given in \cref{alg:silh_boosted_ns}.

\begin{algorithm}[h]
\caption{Silhouette-Boosted \y{ns}}
\label{alg:silh_boosted_ns}
\begin{algorithmic}[1]
{\color{revCol1} \State \textbf{Input:} 2D point correspondences $\{\Omega_x\}$, silhouette directions $\{\mathcal{C}_x\}$, mean shape $\bar{\mathbf{P}}$, shape basis $\{\mathbf{P}_i\}$
\State \textbf{Output:} Estimated pose $\mathbf{R}^{(t)}$, translation $\mathbf{t}^{(t)}$, shape weights $\{w_i^{(t)}\}$

\State \textbf{Initialisation:} \hfill \Comment{Step \rom{1}}
\State \hskip1em Solve \cref{eqn_final_sdp_known} using only $\{\Omega_x\}$ to obtain initial estimates $\mathbf{R}^{(1)}$, $\mathbf{t}^{(1)}$, $\{w_i^{(1)}\}$
\State Set $t \gets 1$

\Repeat:
    \State \textbf{Silhouette Function:} \hfill \Comment{Step \rom{2}}
    \State \hskip1em Compute silhouette directions using $\alpha$-shapes: \cref{alg_eq_1}

    \State \textbf{Correspondence Function:} \hfill \Comment{Step \rom{3}}
    \State \hskip1em Compute $\mu_x^{(t)}(\mathcal{C}_x)$ by matching direction vectors in $\tilde{\mathcal{C}}^{(t)}_x$ to $\mathcal{C}_x$
    \State \hskip1em Construct correspondence tuple $
    \vartheta_x^{(t)} = \{ (s, s') \in [1, S'_x]^2 ~|~ \tilde{\mathcal{C}}^{(t)}_{x,s} \leftrightarrow \mathcal{C}_{x,s'} \}$
    
    \State \hskip1em Define mapping $\varphi(s)$ for indexing into $\bar{\mathbf{P}}$

    \State \textbf{Silhouette-Boosted Shape and Pose Update:} \hfill \Comment{Step \rom{4}}
    \State \hskip1em Solve the convex SDP to update shape and pose: \cref{eqn_final_sdp_known_silh}

    \State $t \gets t + 1$
\Until: {$\vartheta_x^{(t)} \backslash \vartheta_x^{(t-1)} = \emptyset$} }

\State \textcolor{revCol1}{ \Return: final estimates $\mathbf{R}^{(t)}$, $\mathbf{t}^{(t)}$, and $\{w_i^{(t)}\}$}
\end{algorithmic}
\end{algorithm}
}

\section{Experimental validation}\label{exp_valid}

We now delve into the experimental validation of our proposed methodologies. We begin by describing the various dataset we use for validation in section \cref{sec_data_descr}, Next, we delineate the baseline methods we use to benchmark the performance of our proposed method in section \cref{sec_comp_met}. This is followed by the results from \y{ns} in section \cref{sec_gsft} and the results from \y{nsc} from section \cref{sec_exp_gsftp}. Finally, results from silhouette-boosted \y{ns} has been presented in section \cref{sec_exp_gsft_sb}.

\subsection{Data description}\label{sec_data_descr}
First, we describe the various data used in our experiments.

\vspace{2mm}
\noindent \textbf{Synthetic, volumetric data.} For synthetic tests on 3D, non-planar models, we use three standard triangulated mesh models: the {\tt Stanford bunny}\footnote{graphics.stanford.edu/data/3Dscanrep} (also abbreviated to just {\tt bunny}), the 3D scanned bust of {\tt Nefertiti} \cite{huppertz2009nondestructive} and the animated cow {\tt Spot} \cite{crane2013robust}. For each shape, we simulate 5 configurations of freeform deformation \cite{joshi2007harmonic} using the Blender\footnote{blender.org} software and use these deformed shapes to generate the \y{ssm}. For silhouette formation, the models are densified with randomly placed internal keypoints, which are mapped between their deformed configurations barycentrically. The number of internal points added are 23840, 24625, and 23844 for {\tt Stanford bunny}, {\tt Nefertiti}, and {\tt Spot} (resp.). The shape vertices are represented in \y{au}, the mean shapes for {\tt Stanford bunny}, {\tt Nefertiti}, and {\tt Spot} have a largest diagonal of 4.15 \y{au}, 3.62 \y{au}, and 2.58 \y{au} (resp.).

\vspace{2mm}
\noindent \textbf{Synthetic, planar data.} We simulated a paper-like, ruled isometric surface using the method of \cite{perriollat2013comp}. For building the \y{ssm}, we used 10 randomly deformed configuration to generate a semi-dense pointcloud of 100 points each. The coordinates are expressed in \y{au}, a typical largest diagonal of the simulated surface is {$\sim$0.35 \y{au}}.

\vspace{2mm}
\noindent \textbf{Real, \y{mcp}.} We use the {\tt CMU-MoCap} data \cite{cmumocap} for the articulated shapes, particularly the \textit{running} motion of {\tt subject-9} and {\tt subject-35}. We use all the 3D marker positions from the Vicon\footnote{vicon.com} captured data (instead of just the human skeleton). We create two separate \y{ssm}s from the two subjects. 

\vspace{2mm}
\noindent \textbf{Real, articulated objects.} We utilise the fourteen objects presented in the RBO dataset \cite{1806.06465} for testing our method on realistic, household objects. We test our approaches on all the fourteen objects in the dataset, termed {\tt book}, {\tt cabinet}, {\tt cardboard-box}, {\tt clamp}, {\tt folding-rule}, {\tt globe}, {\tt ikea}, {\tt ikea-small}, {\tt laptop}, {\tt microwave}, {\tt pliers}, {\tt rubik's-cube}, {\tt treasure-box}, and {\tt tripod}. We refer the interested readers to the project website\footnote{tu-rbo.github.io/articulated-objects} of \cite{1806.06465} for details on these objects. All dimensions are specified in $m$.

\vspace{2mm}
{\color{revCol1}\noindent\textbf{Real, $\varphi$-\y{sft} data}. This $\varphi$-\y{sft} dataset from \cite{kairanda2022f} contain 9 scenes of deforming fabrics, all of these scenes contain \y{gt} pointcloud and triangulated, textured mesh aligned with the first undeformed configuration of the fabric.}

\vspace{2mm}
\noindent \textbf{Real, planar data.} Owing to the unavailability of planar data with correspondences and silhouette that also suits the \y{ns} problem setup, we utilise a deformed sheet of a paper imaged across five different view points with a calibrated, perspective camera\footnote{This data is a modification of the \textit{Bramante39M} dataset: \href{https://github.com/agnivsen/Bramante39M}{github.com/agnivsen/Bramante39M}}. On this sheet of paper exists 40 keypoints, chosen semi-automatically, and tracked across the images. We use standard \y{sfm} methods \cite{hartley2003multiple, nousias2019large} to reconstruct the \y{gt} of these keypoints, the scale of the reconstruction is recovered by comparison with known dimensions of the paper. The silhouette of the paper is annotated, the depth of the silhouettes remain unknown and unnecessary for our method. We repeat the same reconstruction approach across another eight different poses of deformation (without silhouettes), to obtain eight different \y{gt} for these forty keypoints, allowing use to compute an \y{ssm} of deformation from these eight set of \y{gt} keypoints. The experimental setup is shown in \cref{silh_real}b, all data in $m$. We term this dataset, the {\tt deformed-paper} data.

\vspace{2mm}
{\color{revCol1}\noindent\textbf{Real, biplanar X-rays.}~ Despite recent advances in reconstruction methods leveraging biplanar X-rays, few publicly available datasets simultaneously satisfy the following four criteria: (1) availability of real biplanar X-ray images, (2) imaging of non-trivial structures - i.e., excluding isolated small nodules or 1-dimensional manifolds such as catheters and veins, (3) provision of calibrated intrinsic parameters of the X-ray imaging system, and (4) availability of a corresponding 3D template. Thankfully, the biplanar videoradiography dataset introduced by \cite{welte2022biplanar} meets all four of these requirements. It captures \textit{in-vivo} imaging of a patient’s foot during walking and jogging within a biplanar X-ray setup. The dataset includes calibrated camera intrinsics as well as \y{ct}-derived 3D models of the calcaneus, talus, and tibia, which serve as \y{gt} reference shapes.
}

\subsection{Compared methods}\label{sec_comp_met}

Although there are plenty of recent \y{sft} methods that show impressive performance on real-world data, to the best of our knowledge, there does not exist any method solving either the \y{ns} or \y{nsc} problem, hence incomparable to baseline approaches. Thus, we validate \y{ns} and \y{nsc} solutions (with and without silhouette boosting) on many real dataset and report the results while comparing them against a `trivial solution' (described below), which involves multiple, repeated \y{sft} by state of the art approaches across all viewpoints and combining them trivially. We also validate our method in the special case of $P=1$, which is the classical \y{sft} setup, just to demonstrate that we are at par with state of the art approaches in this classical problem setup.

\vspace{2mm}
\noindent \textbf{Comparison setup.}~We have two distinct setup for comparison:
\begin{itemize}
    \item \textbf{Trivial solution to \y{ns}/\y{nsc}}. Given correspondences from $P$ viewpoints, this comparison setup repeats \y{sft} $P$ times, once for each viewpoint. The final reconstruction is obtained by trivially combining the $P$ set of reconstruction with \y{gt} camera poses while the reconstruction accuracy is reported as the mean of accuracies across $P$ viewpoints
    \item \textbf{Special case of $P=1$}. This is the classical \y{sft} setup and therefore allows us to compare our method against a wide variety of baseline methods.
\end{itemize}

\vspace{2mm}
\noindent\textbf{Baseline approaches.}~We have two broad categories of methods we compare against:
\begin{enumerate}
\item \noindent\textbf{\y{sft} with \y{ssm}.}~We use the method from \cite{zhou20153d} as a baseline for \y{sft} method that utilises \y{ssm}, we term it {\tt Prox.-ADMM} since it uses an \y{admm} based solution of a \y{sdp} problem. 

\item \noindent\textbf{Surface-based~\y{sft}.}~We compared against six \y{sft} methods that are based on physical constraints, forming the surface-based baseline for our \y{sft} problem. They are {\tt BGCCP12} for the closed-form local isometric solution from~\cite{bartoli2012template}, {\tt BGCCP12IR} for {\tt BGCCP15} followed by non-linear refinement from~\cite{brunet2014monocular}, {\tt CPB14} for the normal-integration method from~\cite{chhatkuli2014stable}, {\tt CPB14-R} for {\tt CPB14} followed by non-linear refinement from~\cite{brunet2014monocular}, {\tt SOCP} for the Second-Order Cone Programming (SOCP) based method from \cite{brunet2011monocular}, {\tt CF} for the solution following Cartan's formulation from \cite{parashar2019local}, and {\tt RobuS\textit{f}T} for an adaptation\footnote{Non-real-time adaptation of \cite{shetab2024robusft} was developed in consultation with the authors} of the method of \cite{shetab2024robusft}. We term these combined methods of {\tt BGCCP12I}, {\tt BGCCP12IR}, {\tt CPB14I}, {\tt CPB14IR}, {\tt SOCP}, {\tt CF}, and {\tt RobuS\textit{f}T} as the \textit{surface-based} \y{sft} methods.
\end{enumerate}

\vspace{2mm}
\noindent\textbf{Excluded methods.}~While \y{sft} is an active field of research with many recently proposed approaches, we had to avoid comparisons with some of the recent methods due to one or more of the following factors:
\begin{itemize}
    \item \textbf{Dissimilar input}.~We exclude methods that requires inputs which are not keypoint correspondences and/or silhouettes, e.g.: photometric information or depth data
    \item \textbf{Trajectory priors}.~We exclude methods that requires the input data to follow a continuous trajectory; all the dataset used in our validation are either already separated by wide pose-baseline or suitably sub-sampled (temporally) to have wide pose-baseline 
    \item \textbf{Domain specificity}.~We exclude methods that are specific to particular domains of objects, e.g.: human or hand tracking/reconstruction approaches; our proposed \y{ns}/\y{nsc} methods require no domain-specific data beyond a valid \y{ssm}
    \item \textbf{Public availability of code}.~We exclude methods without publicly available code (and not trivially replicable either)
\end{itemize}
\noindent \Cref{tab_sft_baseline} shows a list of some recent approaches that were excluded from comparison and the reason for doing so.

\vspace{2mm}
\begin{threeparttable}[t]
\begin{adjustbox}{width=0.9\textwidth,center} 
\begin{tabular}{c|cccc}
\hline
\rowcolor[HTML]{C0C0C0} 
\cellcolor[HTML]{C0C0C0}                                          & \multicolumn{4}{c}{\cellcolor[HTML]{C0C0C0}\textbf{Reason for exclusion}}                                                     \\ \cline{2-5} 
\rowcolor[HTML]{C0C0C0} 
\multirow{-2}{*}{\cellcolor[HTML]{C0C0C0}\textbf{Method}} & \textbf{Dissimilar input} & \textbf{Trajectory priors} & \textbf{Domain specificity} & \textbf{Unavailability of public code} \\ \hline
\textbf{$\varphi$ - \y{sft}} (\cite{kairanda2022f})& Yes & Yes   & No  & No  \\
\textbf{Texture based deep-\y{sft}} (\cite{fuentes2021texture, fuentes2022deep})      & Yes\tnote{$\dagger$}   & No & No  & Yes \\
\textbf{DECAF} (\cite{DecafTOG2023}) & Yes & Yes   & Yes & Yes  \\
\textbf{Physics-guided \y{sft}} (\cite{stotko2024physics}) & Yes  & Yes & No  & No   \\ 
\textbf{Contour constrained \y{sft}} (\cite{tan2024improved}) & Yes\tnote{$\dagger$} & No & No & Yes \\
\textbf{2D-3D registration in medical imaging} (\cite{shanmuganathan2024two, zhang2024rigorous, espinel2024keyhole}) & No & No & No\tnote{$\ddagger$} & Yes \\
\hline
\end{tabular}
\end{adjustbox}
\caption{List of some recent \y{sft} approaches that were excluded from compared baseline and the reason for their exclusion (for a method to be a valid baseline approach, all four parameters need to be `no')}
\begin{tablenotes}
\item[$\dagger$] \tiny Requires textured template; although for \cite{tan2024improved}, texture is not strictly necessary
\item[$\ddagger$] \tiny Although \cite{shanmuganathan2024two, zhang2024rigorous, espinel2024keyhole} have been validated on specific human organs/bones each, we assume they can be generalised to any shape
\end{tablenotes}
\label{tab_sft_baseline}
\end{threeparttable}
\vspace{2mm}

\noindent\textbf{Metric.}~Since we always have some correspondences between the reconstruction and the \y{ssm} or input model in all cases, we use \y{rms} as the preferred metric for evaluation of accuracy. For evaluating the accuracy of camera absolute pose, we use the L$_2$-norm of the difference in pose. For evaluating the accuracy of camera orientation, we use the mean Euler angle (in degrees) of the orientation error. {\color{revCol1} While experimenting with the $\varphi$-\y{sft} dataset \cite{kairanda2022f}, to facilitate a consistent comparison with prior works utilizing the same dataset (e.g., \cite{kairanda2022f, stotko2024physics}), we adopt the \y{cd} metric. We follow the precise definition of \y{cd} as presented in equation (10) of \cite{kairanda2022f} and equivalently in equation (18) of \cite{stotko2024physics}.}

\subsection{Results}\label{sec_res}

We now describe our experimental results, highlighting both the quantitative and qualitative aspects, as well as some analysis of the presented results.

\subsubsection{\y{ns}}\label{sec_gsft}
The first set of our experimental results are from the \y{ns} case, where the pose of the multiple camera centres are known apriori.

\paragraph{\y{mcp} with \y{ns}} 
Using the {\tt CMU-MoCap} dataset, we create the \y{ns} setup in five configurations. Each configuration has between two and hundred viewpoints, and in each configuration, each viewpoint has the same number of non-stereo correspondences which always sums (across viewpoints) to one-hundred. The number of viewpoints and per-viewpoint correspondences for each configuration are detailed below:
\begin{itemize}
    \item \textbf{Configuration 1}. Two viewpoints, fifty correspondences each
    \item \textbf{Configuration 2}. Three viewpoints, thirty-three correspondences on the first two and thirty-four on the third viewpoint
    \item \textbf{Configuration 3}. Four viewpoints, twenty-five correspondences each
    \item \textbf{Configuration 4}. Ten viewpoints, ten correspondences each
    \item \textbf{Configuration 5}. One hundred viewpoints, one correspondence each
\end{itemize}
\noindent {\color{revCol1} The deformed shapes in the dataset are randomly roto-translated in $\mathbb{R}^3$, with Euler angles sampled uniformly from the range $[-90^\circ, 90^\circ]$ along each of the three principal axes. Viewpoints are randomly distributed around the target shape to simulate diverse camera configurations. For each of the five experimental setups, reconstructions are repeated 50 times to ensure statistical reliability. Representative results for {\tt subject-9} and {\tt subject-35} are presented in \cref{fig:gsft_sub_9} and \cref{fig:gsft_sub_35}, respectively.}

True to our theoretical observation, the accuracy for \y{ns} remains highly similar, despite splitting correspondences across viewpoints in space. In the absence of any directly comparable competing method, we use the trivial solution of \y{ns}, i.e., baseline method of \y{sft} repeated across all viewpoints in the five configurations. In this case of the {\tt CMU-MoCap} dataset, since there exists no surface representation of the tracked object, the baseline method used in {\tt Prox.-ADMM}. The results from the trivial solution of repeated \y{sft} across viewpoints suffer a strong drop in accuracy and, as must be expected, devolves to a degenerate solution in the fifth configuration. Some sample qualitative results comparing the results from \y{ns} with the trivial solution of repeated \y{sft}s across viewpoints are shown in \cref{fig_sub_9_gsft_qual} and \cref{fig_sub_35_gsft_qual}.

\begin{figure}[b]
\centering
\begin{subfigure}{0.4\textwidth}
  \centering
  \includegraphics[width=0.65\textwidth]{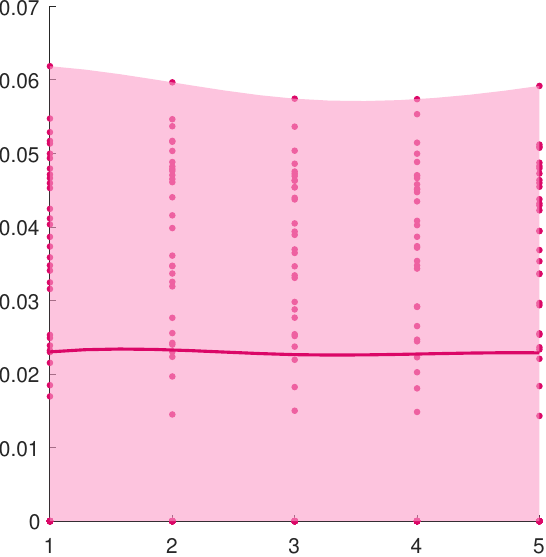}
  \caption{{\tt Subject-9} across five configurations}
  \label{fig:gsft_sub_9}
\end{subfigure}%
\begin{subfigure}{0.4\textwidth}
  \centering
  \includegraphics[width=0.65\textwidth]{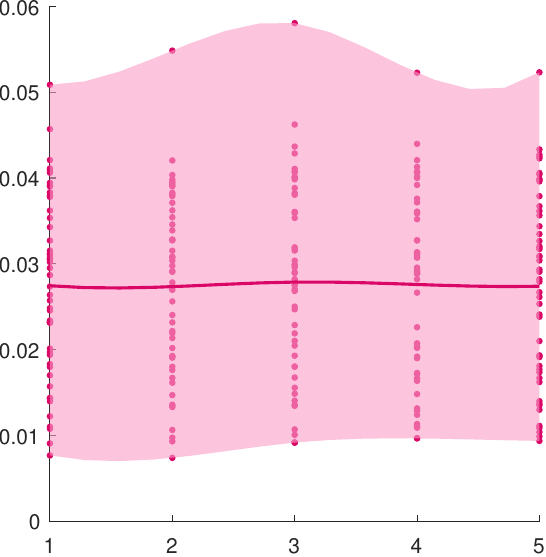}
  \caption{{\tt Subject-35} across five configurations}
  \label{fig:gsft_sub_35}
\end{subfigure}%
\caption{{\tt Subject-9} and {\tt subject-35}, viewed across five different configurations. Each configuration is repeated 50 times, the ordinate of the scatter plots denote the accuracy in terms of the \y{rms} of each reconstruction, the red line denotes their mean value}
\label{fig:gsft_quant}
\end{figure}

\begin{figure}[h]
\centering
  \subfloat{%
    \begin{overpic}[trim=0 20 0 10,clip=true,width=.18\linewidth]{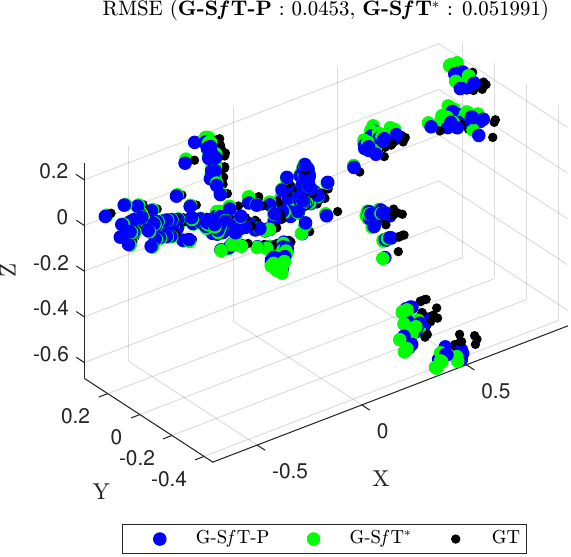}
    \end{overpic}
  }
  \subfloat{%
    \begin{overpic}[trim=0 20 0 10,clip=true,width=.18\linewidth]{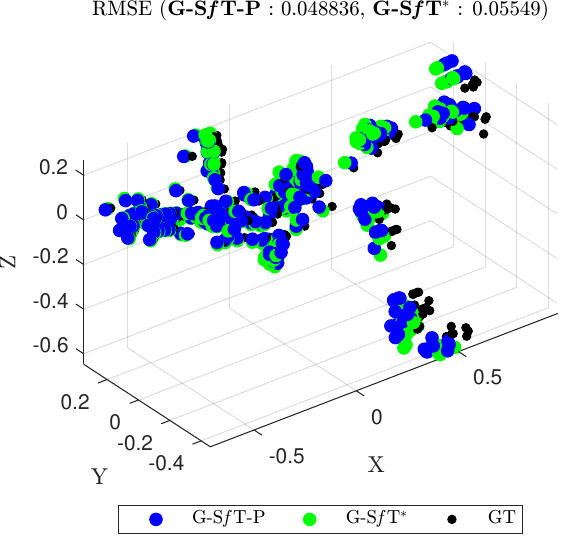}
    \end{overpic}
  }
  \subfloat{%
    \begin{overpic}[trim=0 20 0 10,clip=true,width=.18\linewidth]{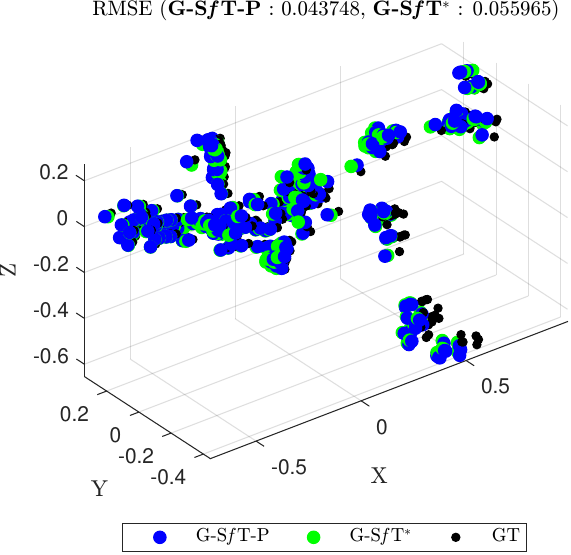}
    \end{overpic}
  }
  \subfloat{%
    \begin{overpic}[trim=0 20 0 10,clip=true,width=.18\linewidth]{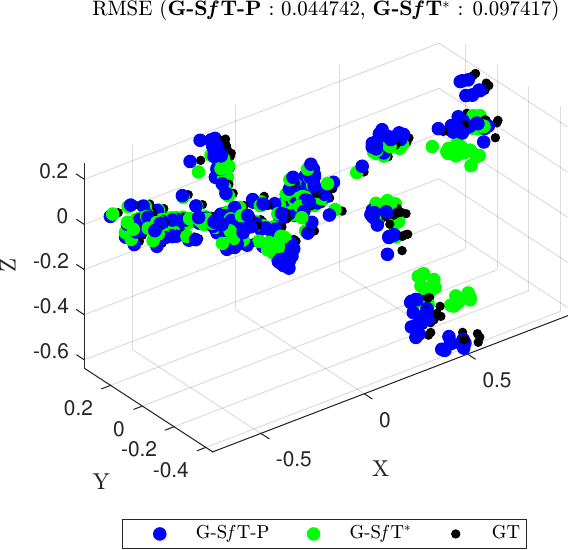}
    \end{overpic}
  }
  \subfloat{%
    \begin{overpic}[trim=0 20 0 10,clip=true,width=.18\linewidth]{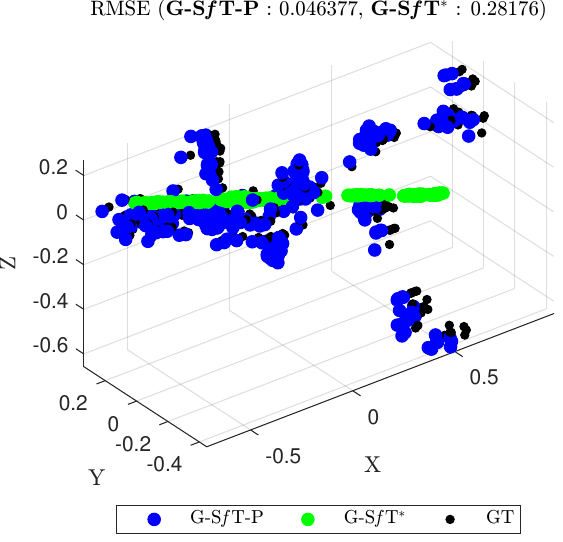}
    \end{overpic}
  }

  \subfloat{%
    \begin{overpic}[trim=0 20 0 10,clip=true,width=.18\linewidth]{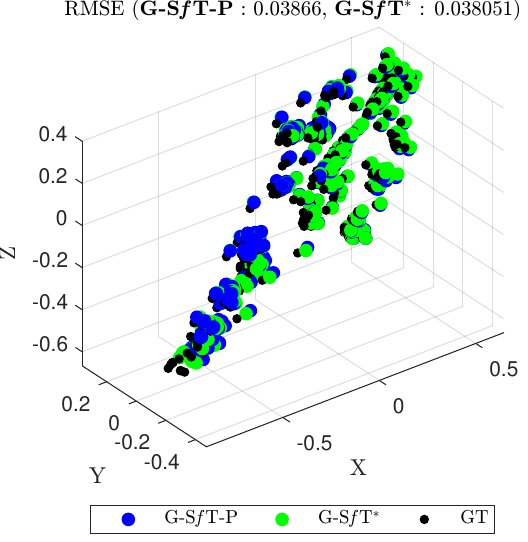}
    \end{overpic}
  }
  \subfloat{%
    \begin{overpic}[trim=0 20 0 10,clip=true,width=.18\linewidth]{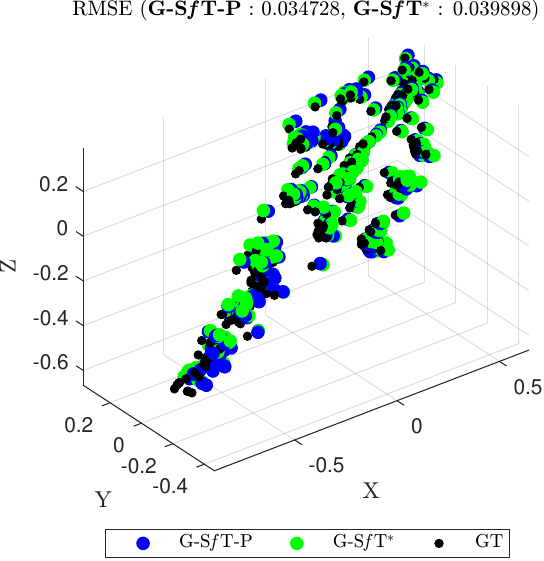}
    \end{overpic}
  }
  \subfloat{%
    \begin{overpic}[trim=0 20 0 10,clip=true,width=.18\linewidth]{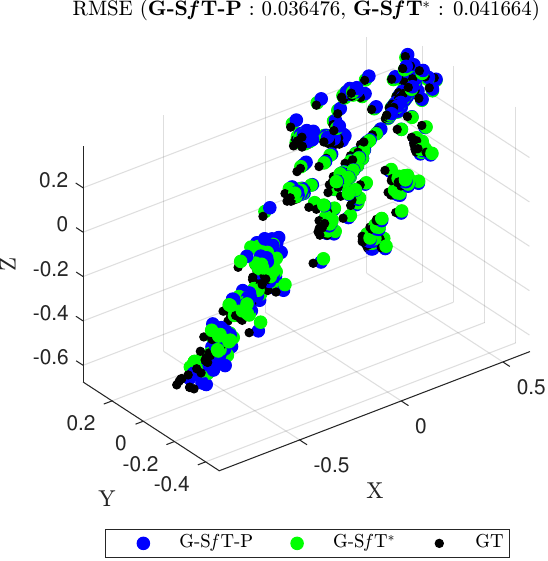}
    \end{overpic}
  }
  \subfloat{%
    \begin{overpic}[trim=0 20 0 10,clip=true,width=.18\linewidth]{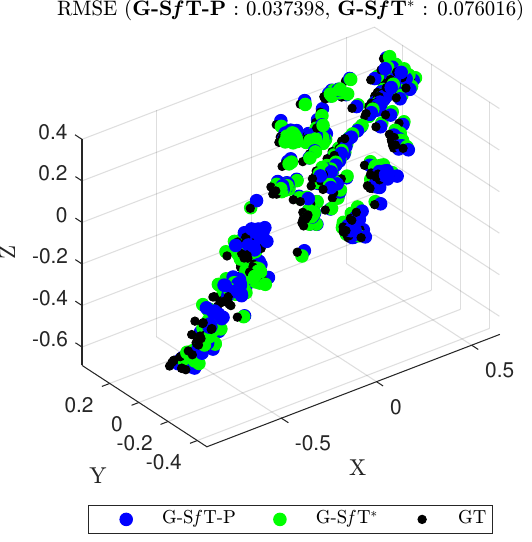}
    \end{overpic}
  }
  \subfloat{%
    \begin{overpic}[trim=0 20 0 10,clip=true,width=.18\linewidth]{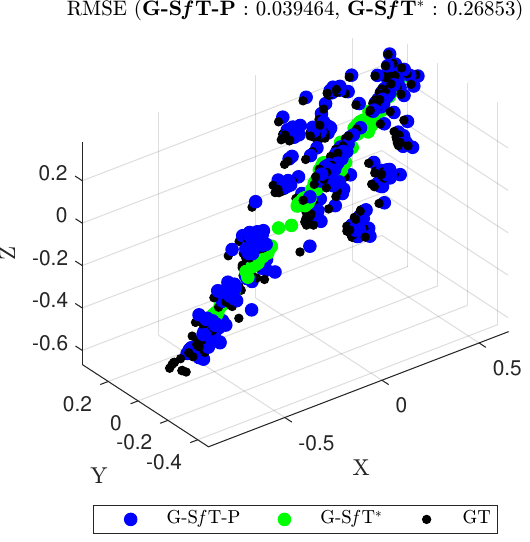}
    \end{overpic}
  }

   \subfloat{%
    \begin{overpic}[trim=0 20 0 10,clip=true,width=.18\linewidth]{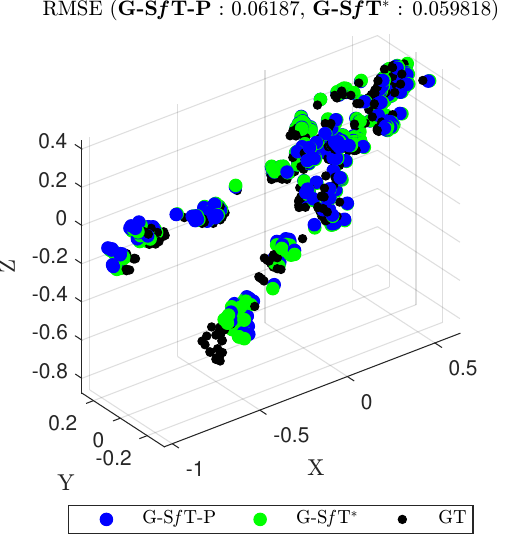}
    \end{overpic}
  }
  \subfloat{%
    \begin{overpic}[trim=0 20 0 10,clip=true,width=.18\linewidth]{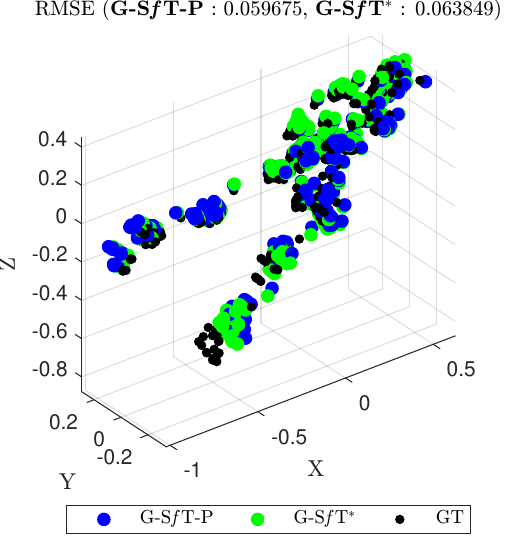}
    \end{overpic}
  }
  \subfloat{%
    \begin{overpic}[trim=0 20 0 10,clip=true,width=.18\linewidth]{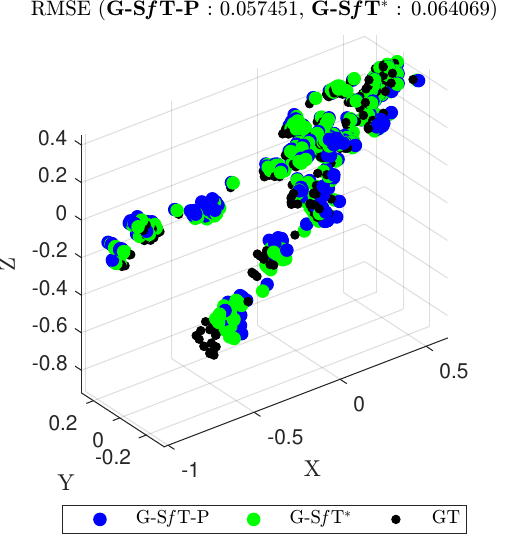}
    \end{overpic}
  }
  \subfloat{%
    \begin{overpic}[trim=0 20 0 10,clip=true,width=.18\linewidth]{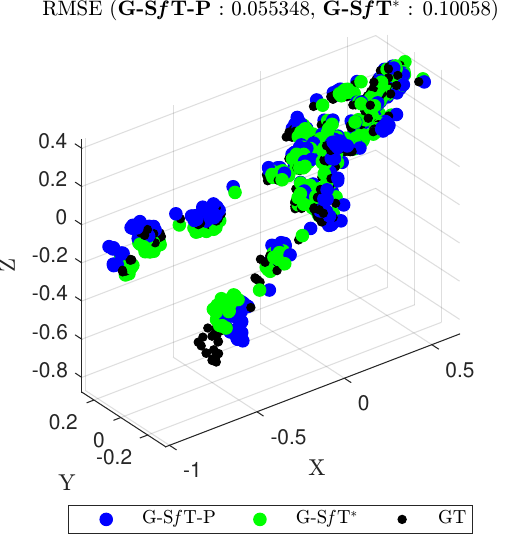}
    \end{overpic}
  }
  \subfloat{%
    \begin{overpic}[trim=0 20 0 10,clip=true,width=.18\linewidth]{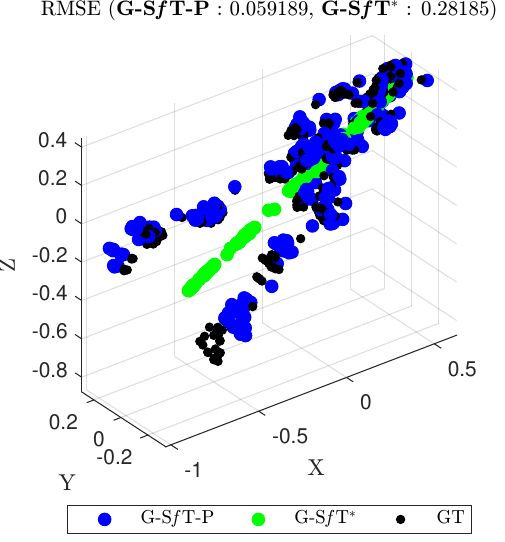}
    \end{overpic}
  }
\caption{Results from the experiments with {\tt Subject-9} of the {\tt CMU-MoCap} dataset with multiple viewpoints. The five columns show configurations one through five. At each configuration, the experiments have been repeated fifty times with randomised pose; the three rows show three random samples of these fifty repeats. In the plots, the \textbf{black} 3D keypoints are the \y{gt}, the \textcolor{blue}{\textbf{blue}} 3D keypoints are from \y{ns} and the \textcolor{green}{\textbf{green}} 3D keypoints are the combined reconstruction obtained by repeated \y{sft} across all the viewpoints in the configuration. Repeated \y{sft}s show roughly the same accuracy as \y{ns} for the first three configurations but the accuracy of \y{sft} is significantly lower than \y{ns} in configuration four and \y{sft} devolves to a degenerate solution at configuration five.}
\label{fig_sub_9_gsft_qual}
\end{figure}

\begin{figure}[h]
\centering
  \subfloat{%
    \begin{overpic}[trim=0 20 0 10,clip=true,width=.18\linewidth]{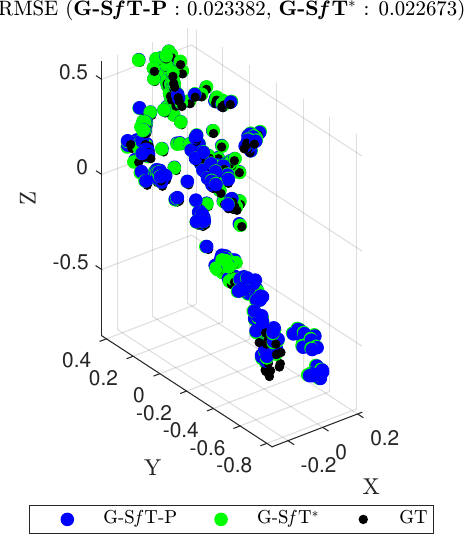}
    \end{overpic}
  }
  \subfloat{%
    \begin{overpic}[trim=0 20 0 10,clip=true,width=.18\linewidth]{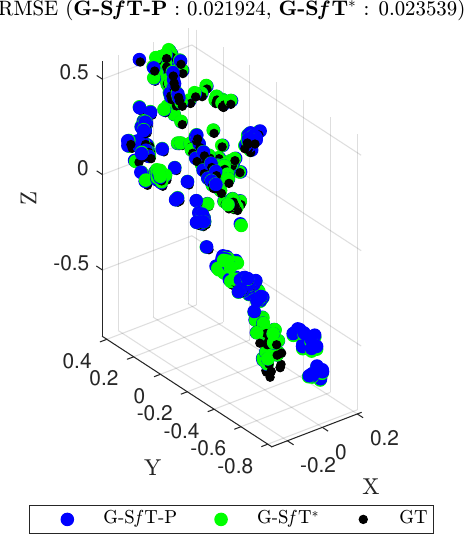}
    \end{overpic}
  }
  \subfloat{%
    \begin{overpic}[trim=0 20 0 10,clip=true,width=.18\linewidth]{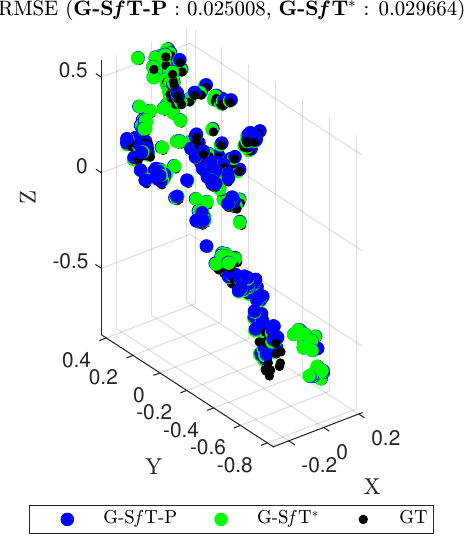}
    \end{overpic}
  }
  \subfloat{%
    \begin{overpic}[trim=0 20 0 10,clip=true,width=.18\linewidth]{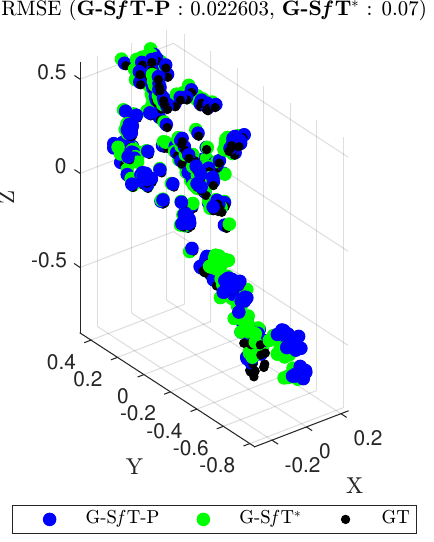}
    \end{overpic}
  }
  \subfloat{%
    \begin{overpic}[trim=0 20 0 10,clip=true,width=.18\linewidth]{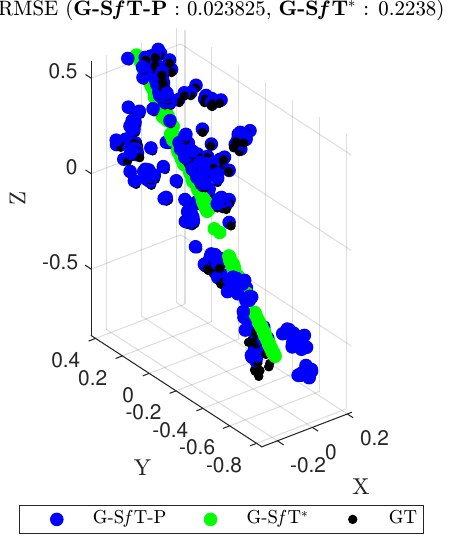}
    \end{overpic}
  }

  \subfloat{%
    \begin{overpic}[trim=0 20 0 10,clip=true,width=.18\linewidth]{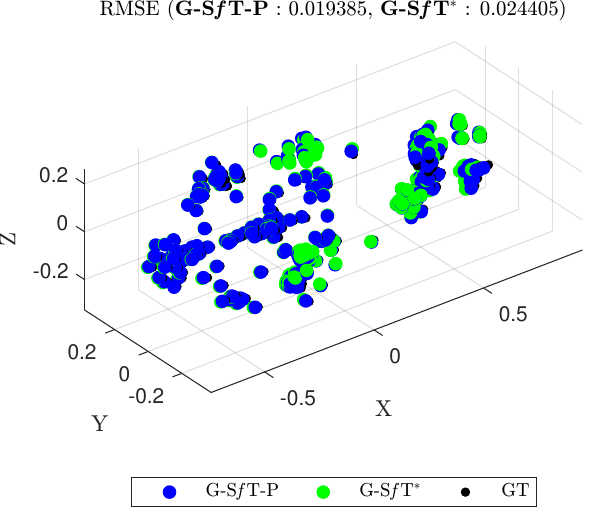}
    \end{overpic}
  }
  \subfloat{%
    \begin{overpic}[trim=0 20 0 10,clip=true,width=.18\linewidth]{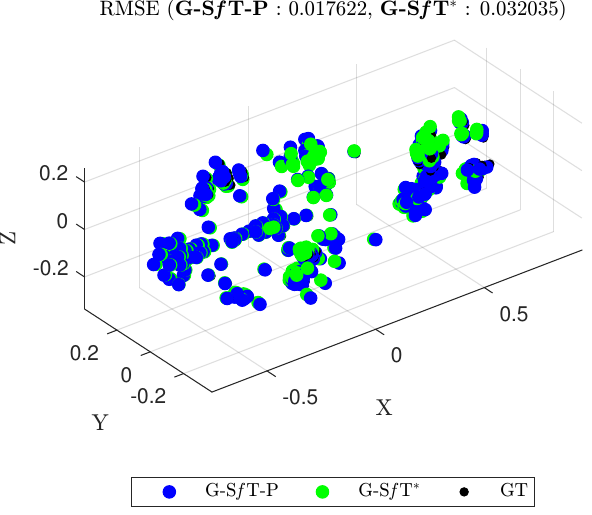}
    \end{overpic}
  }
  \subfloat{%
    \begin{overpic}[trim=0 20 0 10,clip=true,width=.18\linewidth]{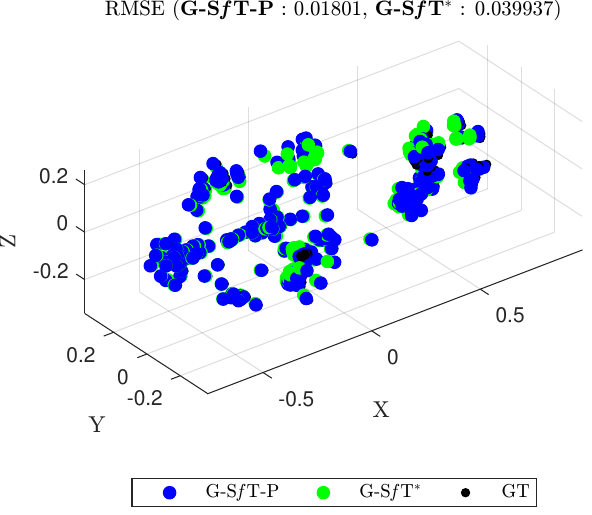}
    \end{overpic}
  }
  \subfloat{%
    \begin{overpic}[trim=0 20 0 10,clip=true,width=.18\linewidth]{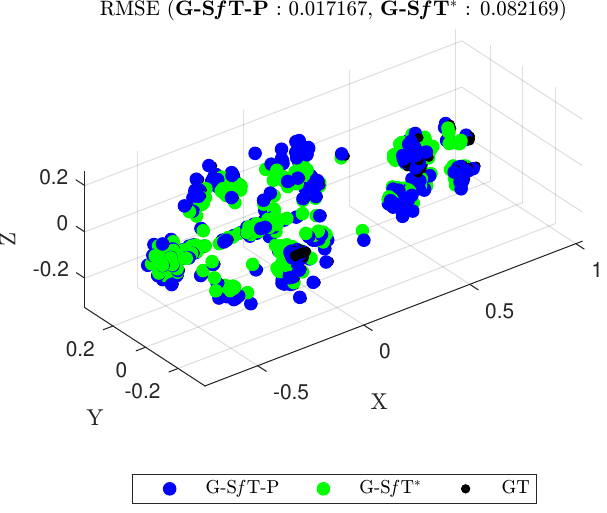}
    \end{overpic}
  }
  \subfloat{%
    \begin{overpic}[trim=0 20 0 10,clip=true,width=.18\linewidth]{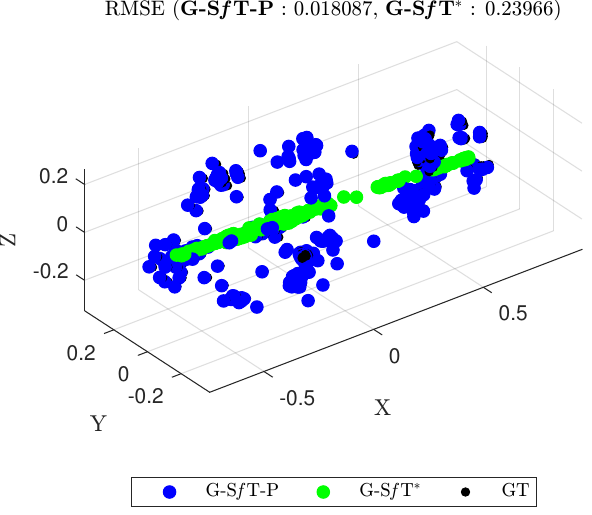}
    \end{overpic}
  }

   \subfloat{%
    \begin{overpic}[trim=0 20 0 15,clip=true,width=.18\linewidth]{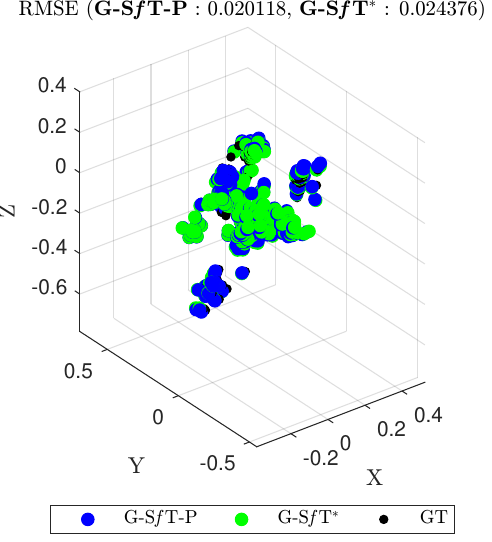}
    \end{overpic}
  }
  \subfloat{%
    \begin{overpic}[trim=0 20 0 15,clip=true,width=.18\linewidth]{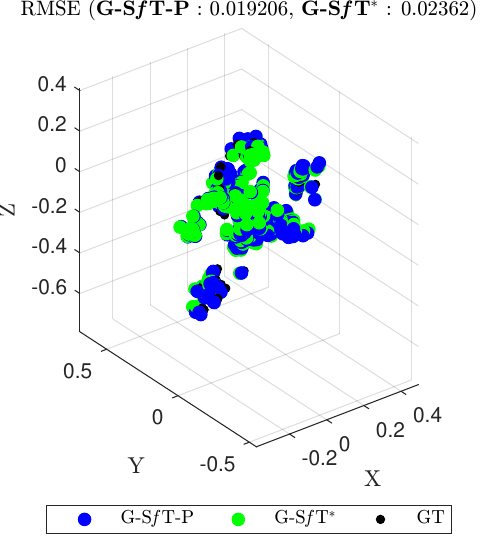}
    \end{overpic}
  }
  \subfloat{%
    \begin{overpic}[trim=0 20 0 15,clip=true,width=.18\linewidth]{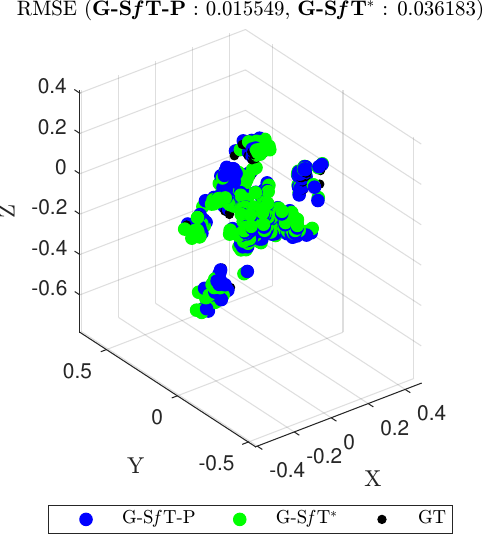}
    \end{overpic}
  }
  \subfloat{%
    \begin{overpic}[trim=0 20 0 15,clip=true,width=.18\linewidth]{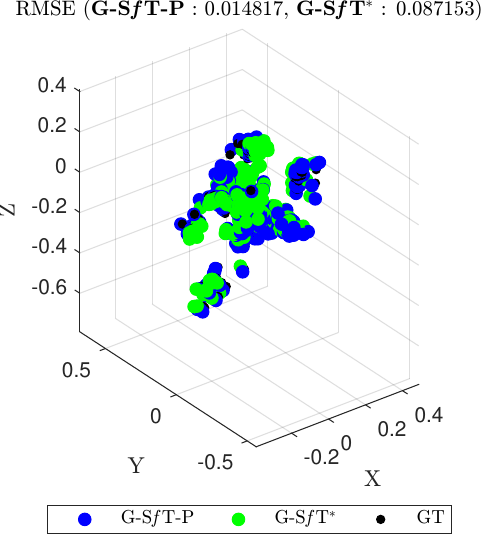}
    \end{overpic}
  }
  \subfloat{%
    \begin{overpic}[trim=0 20 0 15,clip=true,width=.18\linewidth]{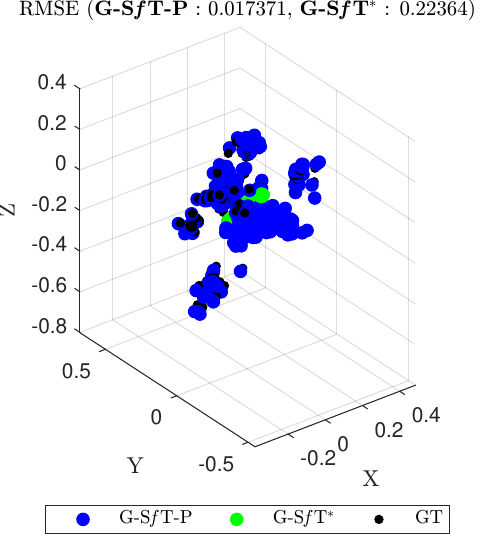}
    \end{overpic}
  }
\caption{Results from the experiments with {\tt Subject-35} of the {\tt CMU-MoCap} dataset with multiple viewpoints. The five columns show configurations one through five. At each configuration, the experiments have been repeated fifty times with randomised pose; the three rows show three random samples of these fifty repeats. In the plots, the \textbf{black} 3D keypoints are the \y{gt}, the \textcolor{blue}{\textbf{blue}} 3D keypoints are from \y{ns} and the \textcolor{green}{\textbf{green}} 3D keypoints are the combined reconstruction obtained by repeated \y{sft} across all the viewpoints in the configuration. Repeated \y{sft}s show roughly the same accuracy as \y{ns} for the first three configurations but the accuracy of \y{sft} is significantly lower than \y{ns} in configuration four and \y{sft} devolves to a degenerate solution at configuration five.}
\label{fig_sub_35_gsft_qual}
\end{figure}

\paragraph{Additional analysis of \y{ns}}\label{sec_addnl_analysis}

To demonstrate that the proposed \y{ns} methodology is at par with the state of the art approaches in the classical \y{sft} setup, we offer many additional results with the assumption of $P=1$, i.e., a special-case of \y{ns} with one viewpoint, allowing us to freely compare against many existing methods.

\subparagraph{Synthetic, volumetric objects.} We use the \y{ssm} from {\tt Stanford bunny}, {\tt Nefertiti}, and {\tt Spot} to generate arbitrary shapes with basis weights drawn from an uniform random distribution in the range $[0,1]$ and pose drawn from Euler angles sampled from an uniform random distribution in the range $[-90^{\circ},90^{\circ}]$ from the object pose of the mean shape. We sample random keypoints from the object model as correspondences. We use only {\tt Prox.-ADMM} for the comparison, since the surface-based \y{sft} methods are unsuitable for keypoints sampled from the interior of a volumetric object, all experiments done with $P=1$.

\begin{figure}[h]
\centering
\begin{subfigure}{.16\textwidth}
  \centering
  \includegraphics[width=0.76\textwidth]{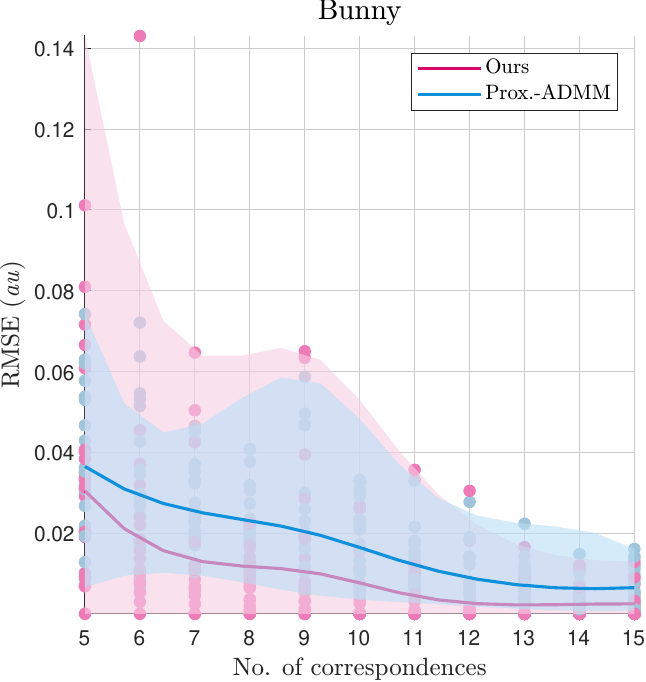}
  \caption{}
  \label{fig:bunny_noiseless}
\end{subfigure}%
\begin{subfigure}{.16\textwidth}
  \centering
  \includegraphics[width=0.76\textwidth]{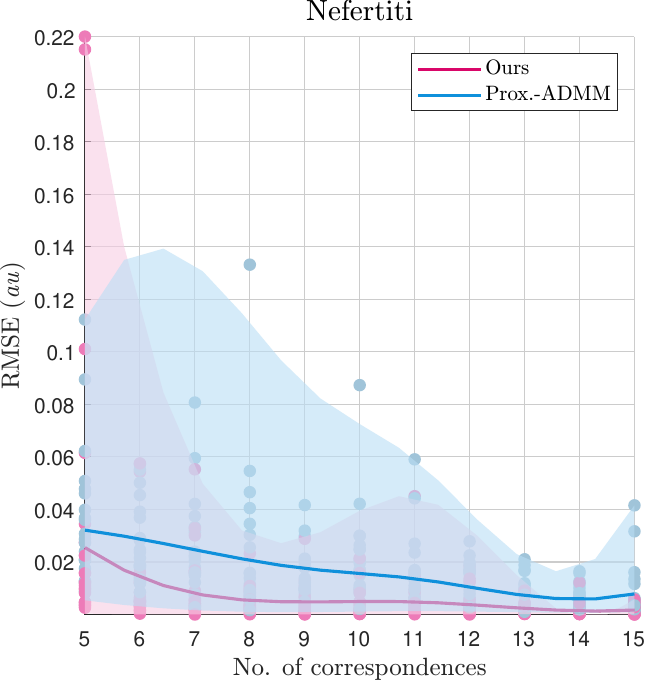}
  \caption{}
  \label{fig:nefertiti_noiseless}
\end{subfigure}
\begin{subfigure}{.16\textwidth}
  \centering
  \includegraphics[width=0.76\textwidth]{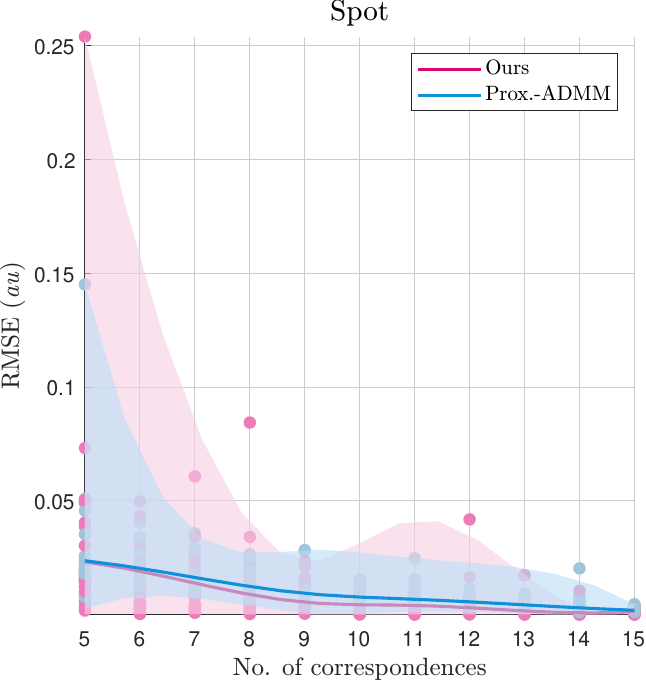}
  \caption{}
  \label{fig:spot_noiseless}
\end{subfigure}
\begin{subfigure}{.16\textwidth}
  \centering
  \includegraphics[width=0.76\textwidth]{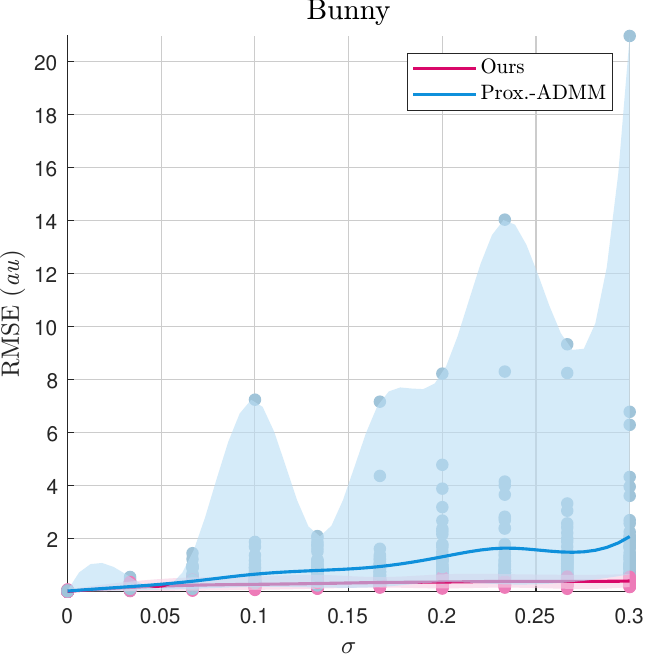}
  \caption{}
  \label{fig:bunny_noisy_7}
\end{subfigure}%
\begin{subfigure}{.16\textwidth}
  \centering
  \includegraphics[width=0.76\textwidth]{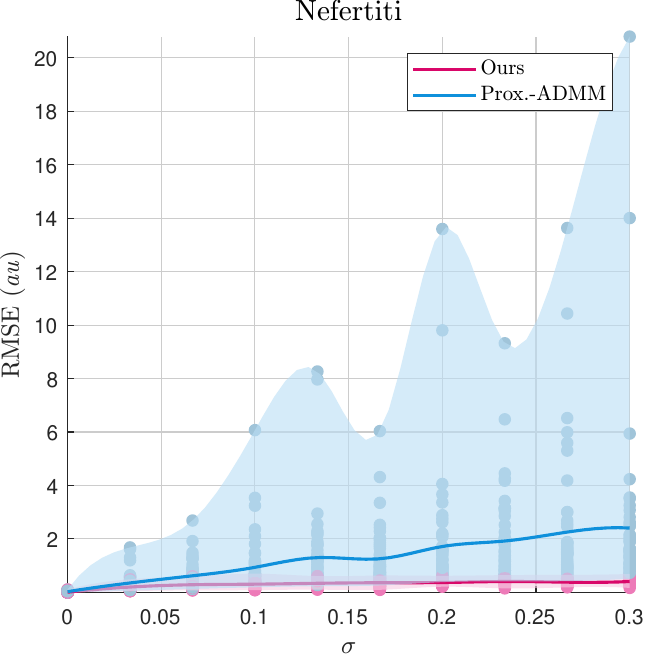}
  \caption{}
  \label{fig:nefertiti_noisy_7}
\end{subfigure}
\begin{subfigure}{.16\textwidth}
  \centering
  \includegraphics[width=0.76\textwidth]{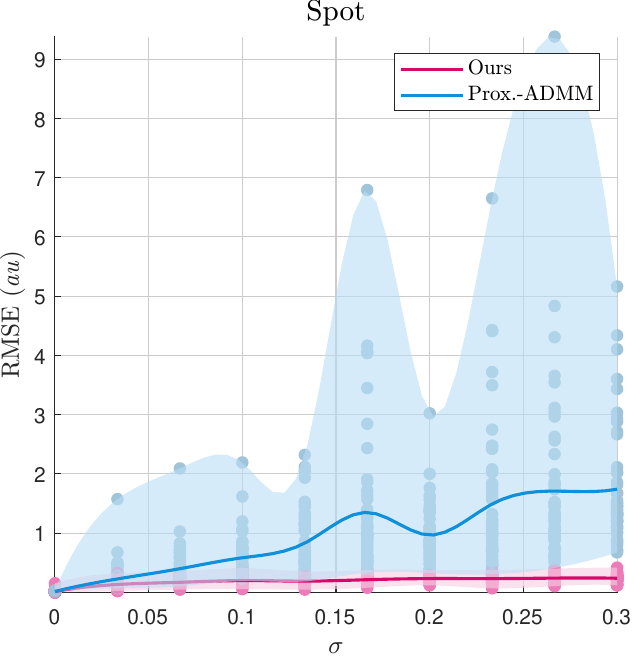}
  \caption{}
  \label{fig:spot_noisy_7}
\end{subfigure}

\begin{subfigure}{.2\textwidth}
  \centering
  \includegraphics[width=0.76\textwidth]{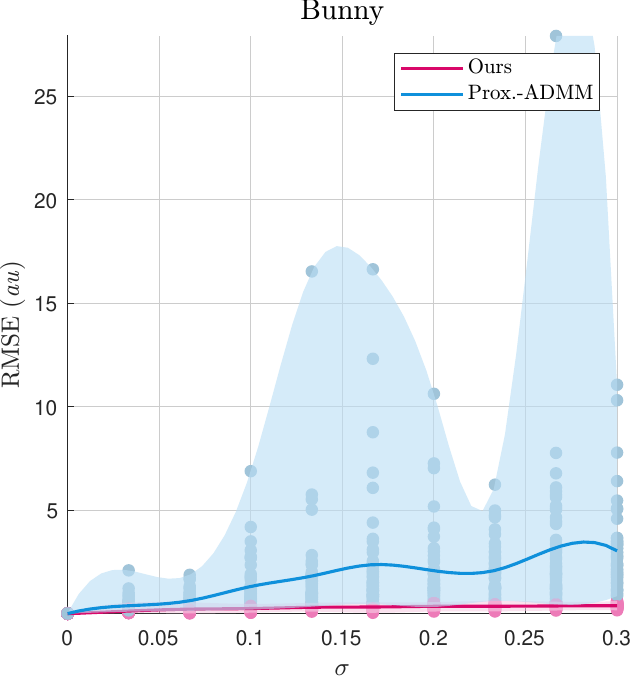}
  \caption{}
  \label{fig:bunny_noisy_12}
\end{subfigure}%
\begin{subfigure}{.2\textwidth}
  \centering
  \includegraphics[width=0.76\textwidth]{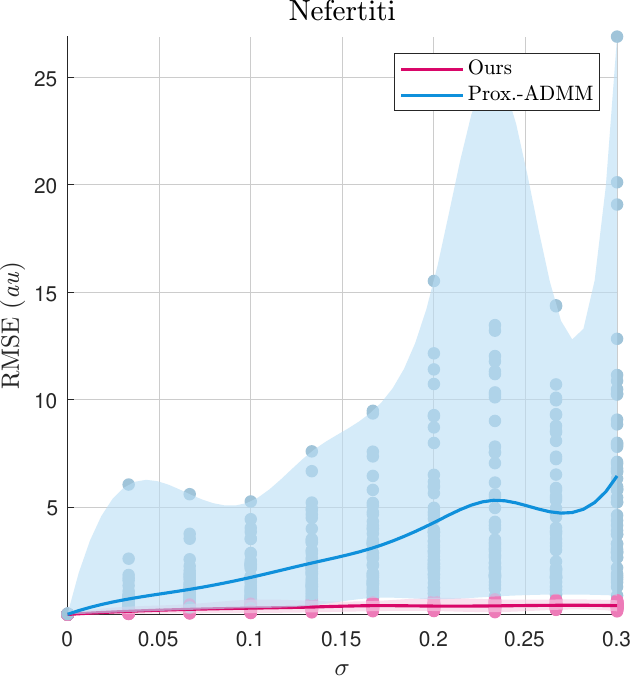}
  \caption{}
  \label{fig:nefertiti_noisy_12}
\end{subfigure}
\begin{subfigure}{.2\textwidth}
  \centering
  \includegraphics[width=0.76\textwidth]{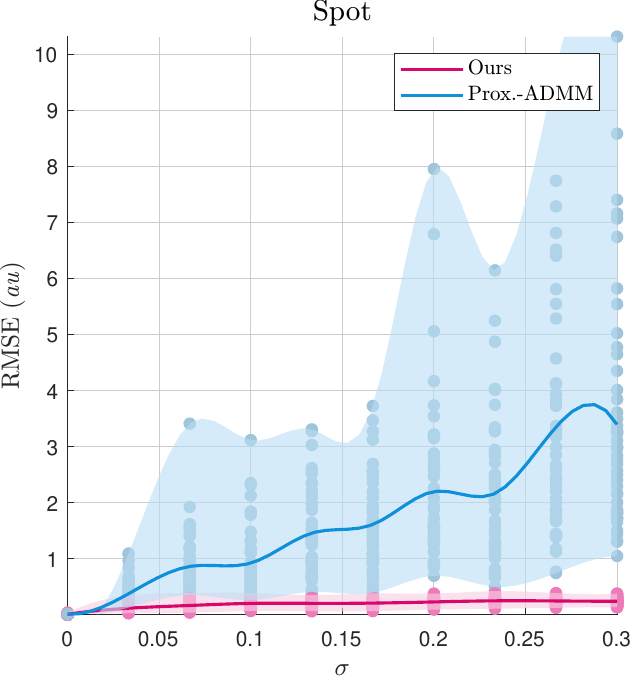}
  \caption{}
  \label{fig:spot_noisy_12}
\end{subfigure}
\caption{Results from synthetic experiments using the 3D models {\tt Stanford bunny}, {\tt Nefertiti}, and {\tt Spot}. (a), (b), and (c) shows results with no added noise on data with correspondences increasing from 5 through 15, each experiment in repeated twenty times. (d), (e), and (f) shows results with \y{sd} of additive noise on data for seven correspondences, each experiment in repeated fifty times. (g), (h), and (i) shows results with \y{sd} of additive noise on data for twelve correspondences, each experiment in repeated fifty times. The Y-axis is \y{rms}, all values in \textit{au}. The shaded area represents the range of accuracy values, the values in between steps are interpolated smoothly with B-splines, the mean values are denoted by the bold curve.}
\label{fig:noiseless}
\end{figure}

We use two setup for the experiments, the \textit{first} uses increasing number of keypoint correspondences from 5 through 15 with no noise added. The results are shown in \cref{fig:bunny_noiseless}, \cref{fig:nefertiti_noiseless}, and \cref{fig:spot_noiseless}. For the \textit{second} setup, we sample the case of 7 and 12 correspondences and add noises to the keypoint correspondences to test the resilience of the two methods to added noise. It is customary for such experiments to add the noise on the projected images/keypoints, but since we strive to remain agnostic to projection methods ({\tt Prox.-ADMM} accepts only orthographic projections), we decided to add the noise directly on the 3D keypoints before synthetically projecting them. For 7 keypoint correspondences, we add a noise sampled from the normal distribution with increasing \y{sd} as shown in \cref{fig:bunny_noisy_7}, \cref{fig:nefertiti_noisy_7}, and \cref{fig:spot_noisy_7}. Similar results for 12 correspondences are shown in \cref{fig:bunny_noisy_12}, \cref{fig:nefertiti_noisy_12}, and \cref{fig:spot_noisy_12}. 

In the noise-free cases of \cref{fig:bunny_noiseless}, \cref{fig:nefertiti_noiseless}, and \cref{fig:spot_noiseless}, both our proposed method and {\tt Prox.-ADMM} seems to be reasonably close in performance, however, our method does have a slightly lesser mean \y{rms} for {\tt Stanford bunny} and {\tt Nefertiti}. In the noisy cases of \cref{fig:bunny_noisy_7}, \cref{fig:nefertiti_noisy_7}, \cref{fig:spot_noisy_7}, \cref{fig:bunny_noisy_12}, \cref{fig:nefertiti_noisy_12}, and \cref{fig:spot_noisy_12}, \y{ns} significantly outperforms {\tt Prox.-ADMM} as the \y{sd} of the noise goes higher. Even at lower \y{sd}s, {\tt Prox.-ADMM} already produces significant outliers, attesting the improved noise tolerance of \y{ns} in this special case of $P=1$. 

\subparagraph{Synthetic, planar data.} We use the simulated paper-like, developable surfaces from \cite{perriollat2013comp} and test our method against the surface-based \y{sft} methods. The \y{ssm} is built up from samples created by the same surface generator. We use two testing conditions, one with \textit{milder} deformations where the deformation angle between successive rulings of the synthetic, developable surface is confined to an uniform distribution drawn from $[-5^{\circ}, 5^{\circ}]$, results shown in \cref{fig:iso_synth_mild}, and \textit{stronger} deformations where the deformation angle between successive rulings of the synthetic, developable surface is confined to an uniform distribution drawn from $[-30^{\circ}, 30^{\circ}]$. 

\begin{figure}[b]
\centering
\begin{subfigure}{0.33\textwidth}
  \centering
  \includegraphics[width=0.9\textwidth]{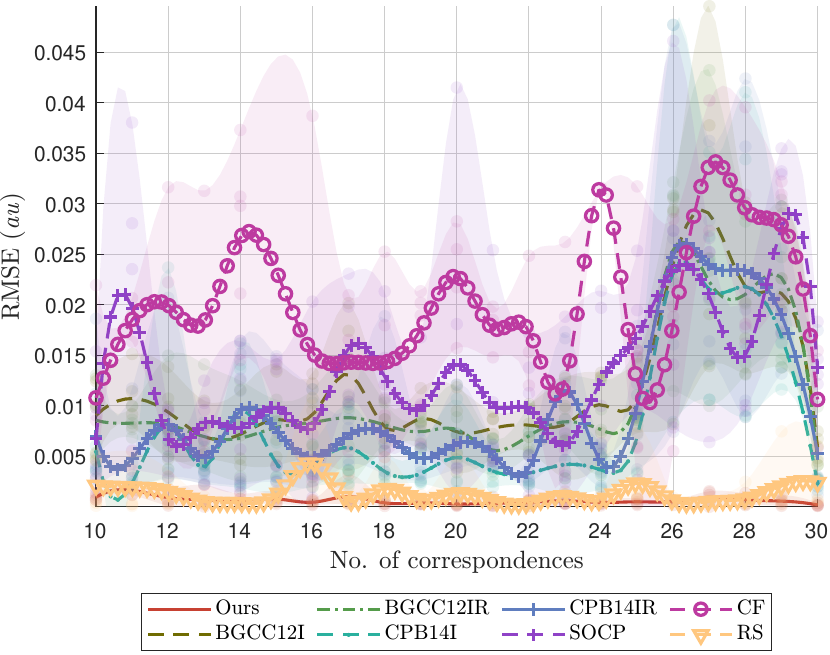}
  \caption{Mild deformation}
  \label{fig:iso_synth_mild}
\end{subfigure}%
\begin{subfigure}{0.33\textwidth}
  \centering
  \includegraphics[width=0.9\textwidth]{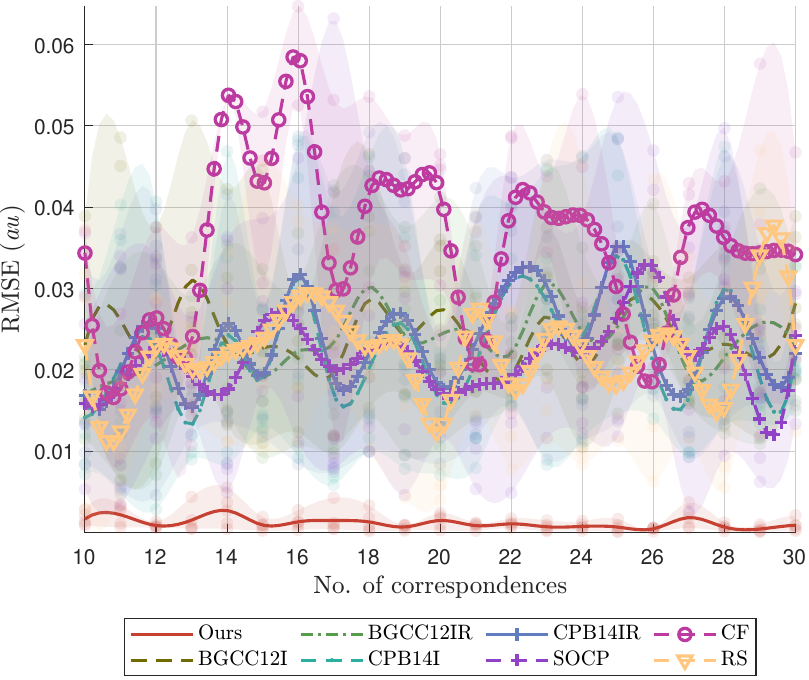}
  \caption{Strong deformation}
  \label{fig:iso_synth_strong}
\end{subfigure}%
\begin{subfigure}{0.33\textwidth}
  \centering
  \includegraphics[width=0.9\textwidth]{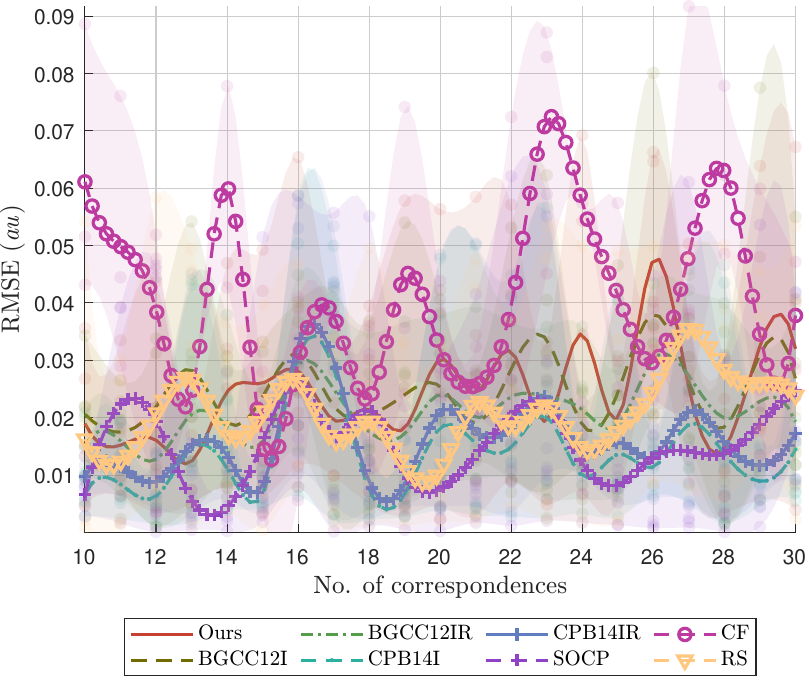}
  \caption{Strong deformation (random shape)}
  \label{fig:iso_synth_strong_oob}
\end{subfigure}%
\caption{Results for the synthetic, planar surface. (a) and (b) shows the results of our \y{ns} approach on mild and strong deformation (resp.), as compared to the state-of-the-art, surface-based \y{sft} methods with deformation targets generated by assigning random weights to the \y{ssm} used by \y{ns}. (c) shows the results on strong deformation, but with randomly generated new shapes, not necessarily related to the learned \y{ssm} of \y{ns}}
\label{fig:iso_synth}
\end{figure}

For the \textit{first} set of experiments, the testing target is generated by creating a random shape from the learned \y{ssm} with random weights, and assigning the resulting shape some random pose in space. The results on milder and stronger deformations are shown in \cref{fig:iso_synth_mild} and \cref{fig:iso_synth_strong} respectively. We compare against all available surface-based methods, i.e., {\tt BGCCP12}, {\tt BGCCP12IR}, {\tt CPB14I}, {\tt CPB14IR}, {\tt SOCP}, {\tt CF}, and {\tt RS}. Clearly, our proposed method beats all compared methods in this setting. In the milder case, our proposed \y{nsc} has a mean RMSE (across all correspondence steps) of $6.33 \times 10^{-4}$ \textit{au}, whereas the second best method of {\tt RS} has a mean RMSE of $1.31 \times 10^{-4}$ \textit{au}, a 52.3\% decrease of error. In the stronger case, our proposed \y{nsc} has a mean RMSE (across all correspondence steps) of $1.15 \times 10^{-3}$ \textit{au}, whereas the second best method of {\tt SOCP} has a mean RMSE of $2.16 \times 10^{-2}$ \textit{au}, a 94.6\% decrease of error.

For the \textit{second} set of experiments, the testing target is generated by creating a random shape separately from the learned \y{ssm} using the same simulator from \cite{perriollat2013comp} for stronger deformations, and assigning the resulting shape some random pose in space. The results are shown in \cref{fig:iso_synth_strong_oob}. The performance of our proposed approach declines in this setting and trails in accuracy as compared to some of the compared methods. This is unsurprising since the surface-based methods utilise physics based priors and, for strong deformations outside of the learned shape-space of \y{ssm}, has some clear advantages. However, this advantage is small ($<1 \times 10^{-2}$~\textit{au} compared to the best method of {\tt CPB14I}) and must be considered in the context of $P=1$ being a special case for \y{ns}; for any $P \geq 2$, none of these compared \y{sft} methods are applicable.

\subparagraph{Real data, \y{mcp}.}~We show the results on the {\tt CMU-MoCap} dataset for the subjects 9 and 35 in \cref{fig:art_real_noisy}. The provided \y{mcp} 3D keypoints in the dataset are randomly sampled, producing correspondences numbering 5 though 40, incremented in steps of 5. The \y{mcp} 3D keypoints are randomly roto-translated in space and projected with to generate the data. Comparison is just with {\tt Prox.-ADMM}, since none of the surface-based methods are applicable for this data due to articulation and the absence of a single, continuous surface. We follow a cross-validation approach of testing {\tt subject-9} with \y{ssm} learned using only the data from {\tt subject-35} and vice-versa.

\begin{figure}[h]
\centering
\begin{subfigure}{0.25\textwidth}
  \centering
  \includegraphics[width=0.9\textwidth]{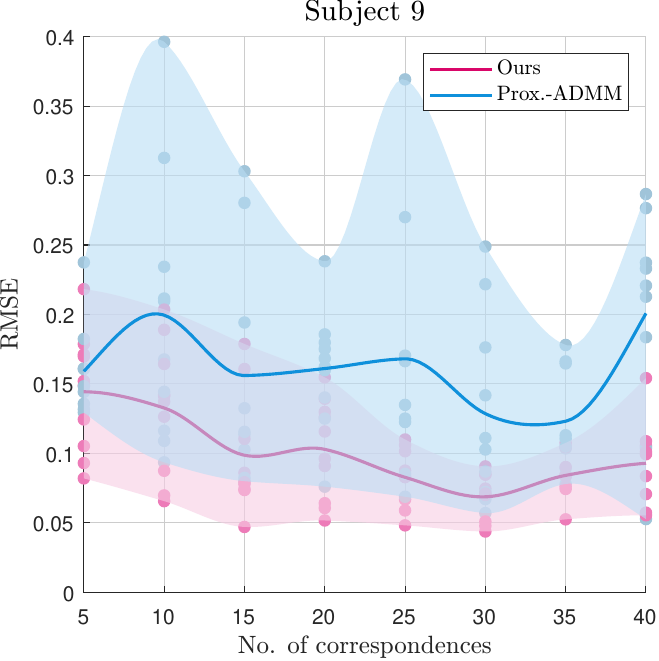}
  \caption{}
  \label{fig:sub_9}
\end{subfigure}%
\begin{subfigure}{0.25\textwidth}
  \centering
  \includegraphics[width=0.9\textwidth]{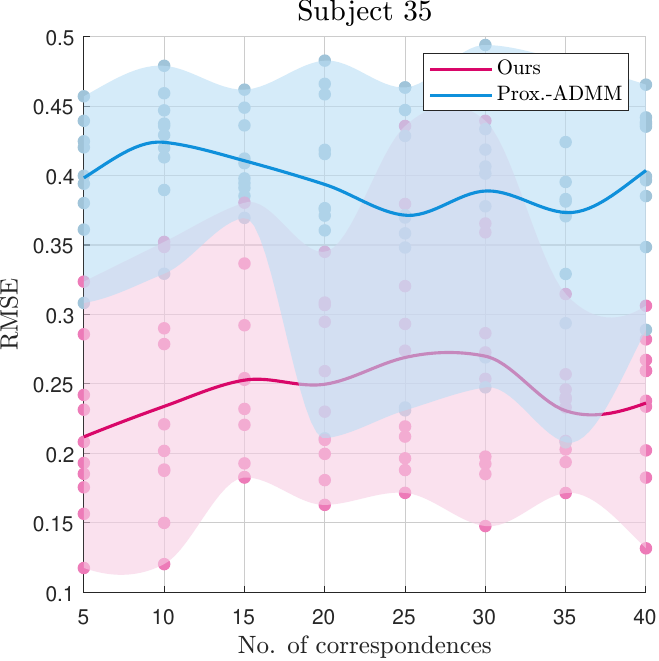}
  \caption{}
  \label{fig:sub_35}
\end{subfigure}%
\begin{subfigure}{0.25\textwidth}
  \centering
  \includegraphics[width=0.9\textwidth]{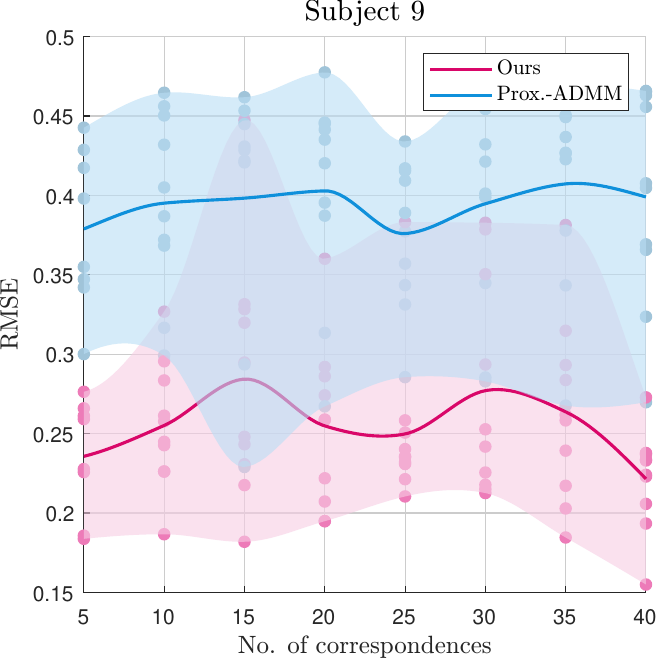}
  \caption{}
  \label{fig:sub_9_noisy}
\end{subfigure}%
\begin{subfigure}{0.25\textwidth}
  \centering
  \includegraphics[width=0.9\textwidth]{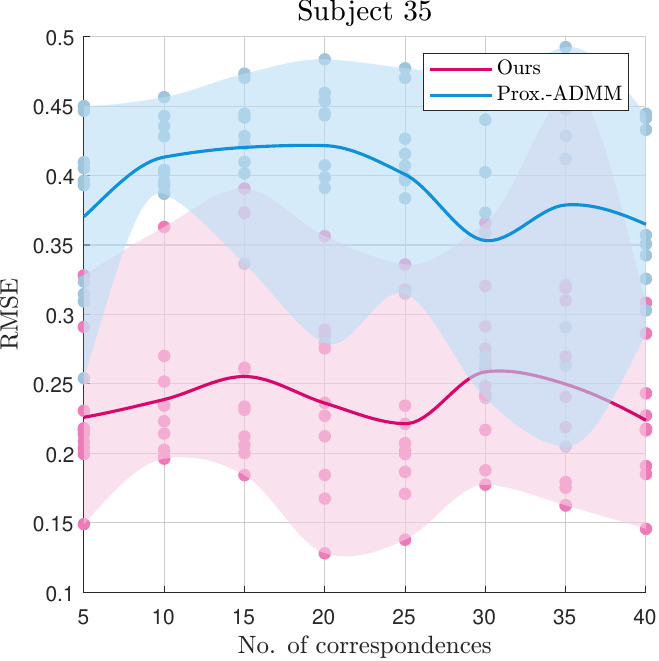}
  \caption{}
  \label{fig:sub_35_noisy}
\end{subfigure}%
\caption{Results from experiments on real, articulated data from the {\tt CMU-\y{mcp}} dataset. (a) and (b) shows the results with increasing correspondences while (c) and (d) shows the same results with added noise on the data}
\label{fig:art_real_noisy}
\end{figure}

\Cref{fig:sub_9} and \cref{fig:sub_35} shows the results obtained across varying correspondences, each correspondence configuration repeated 10 times. \Cref{fig:sub_9_noisy} and \cref{fig:sub_35_noisy} shows the same experiments, but with a noise of \y{sd} 0.1 added to the 3D data before projection. 

For both \y{ns} and {\tt Prox.-ADMM}, the variance with increasing correspondences are not too significant. However, in all cases, \y{ns} remains more accurate than {\tt Prox.-ADMM}. Some qualitative results are shown in \cref{art_qual_simpl}.

\subparagraph{Real data, RBO.}~As a validation on realistic, household objects, we offer the results from experiments on fourteen objects of the RBO dataset. Each of these fourteen objects in this dataset has between twenty-five and thirty-two sequences, each sequence equipped with the 3D \y{gt} position of some markers placed on the objects (refer \cite{1806.06465} for details) while they are being articulated. We use the \y{gt} of the first two sequence of each object to create the \y{ssm} for that particular object. We then test our reconstruction method on one-hundred frames sampled arbitrarily from the third sequence. This test of reconstruction is done by randomly roto-translating the \y{gt} and projecting to create the input correspondences; as with the previous experiments, the projection is perspective for our method and orthographic for the compared method of {\tt Prox.-ADMM} and all experiments are with $P=1$. As with {\tt CMU-MoCap}, none of the surface based methods are applicable to this dataset due to articulation and the absence of a single, continuous surface.

The results from these experiments are summarised in \cref{tab_rbo_all}. In this dataset, our proposed method is very close in accuracy to {\tt Prox.-ADMM}, however, the mean accuracy across all frames for all sequences is 1.31 $mm$ for our method and 1.40 $mm$ for {\tt Prox.-ADMM}, therefore, an improvement exists.

%and some qualitative results are shown in \cref{fig:rbo_qual_1} and \cref{fig:rbo_qual_2}. 

\begin{table}[b] \begin{adjustbox}{width=\textwidth,center} \begin{tabular}{rcccccccccccccc}  \toprule \hline & \cellcolor[RGB]{200,200,200} \textbf{Book} & \cellcolor[RGB]{200,200,200} \textbf{Cabinet} & \cellcolor[RGB]{200,200,200} \textbf{Cardboardbox} & \cellcolor[RGB]{200,200,200} \textbf{Clamp} & \cellcolor[RGB]{200,200,200} \textbf{Folding-rule} & \cellcolor[RGB]{200,200,200} \textbf{Globe} & \cellcolor[RGB]{200,200,200} \textbf{Ikea} & \cellcolor[RGB]{200,200,200} \textbf{Ikea-small} & \cellcolor[RGB]{200,200,200} \textbf{Laptop} & \cellcolor[RGB]{200,200,200} \textbf{Microwave} & \cellcolor[RGB]{200,200,200} \textbf{Pliers} & \cellcolor[RGB]{200,200,200} \textbf{Rubik's-cube} & \cellcolor[RGB]{200,200,200} \textbf{Treasure-box} & \cellcolor[RGB]{200,200,200} \textbf{Tripod}\\ \cline{2-15} 
\textbf{Ours}  & \cellcolor[RGB]{229.5, 240.8985, 248.3955}\underline{0.026} &\cellcolor[RGB]{0, 113.985, 188.955}\textbf{\textcolor{white}{\underline{0.001}}} &\cellcolor[RGB]{88.2692, 162.7979, 211.8167}\underline{0.011} &\cellcolor[RGB]{70.6154, 153.0353, 207.2444}\underline{0.008} &\cellcolor[RGB]{194.1923, 221.3733, 239.2508}\underline{0.023} &\cellcolor[RGB]{141.2308, 192.0856, 225.5338}0.015 &\cellcolor[RGB]{17.6538, 123.7476, 193.5273}\textbf{\underline{0.003}} &\cellcolor[RGB]{35.3077, 133.5102, 198.0997}\underline{0.006} &\cellcolor[RGB]{105.9231, 172.5605, 216.3891}\underline{0.011} &\cellcolor[RGB]{211.8462, 231.1359, 243.8232}0.025 &\cellcolor[RGB]{123.5769, 182.323, 220.9614}0.015 &\cellcolor[RGB]{158.8846, 201.8482, 230.1061}0.015 &\cellcolor[RGB]{176.5385, 211.6108, 234.6785}0.017 &\cellcolor[RGB]{52.9615, 143.2727, 202.672}\underline{0.007}\\
\textbf{Prox.-ADMM}  & \cellcolor[RGB]{229.5, 240.8985, 248.3955}0.027 &\cellcolor[RGB]{0, 113.985, 188.955}\textbf{\textcolor{white}{0.002}} &\cellcolor[RGB]{158.8846, 201.8482, 230.1061}0.019 &\cellcolor[RGB]{141.2308, 192.0856, 225.5338}0.018 &\cellcolor[RGB]{211.8462, 231.1359, 243.8232}0.024 &\cellcolor[RGB]{35.3077, 133.5102, 198.0997}\underline{0.006} &\cellcolor[RGB]{17.6538, 123.7476, 193.5273}\textbf{0.004} &\cellcolor[RGB]{176.5385, 211.6108, 234.6785}0.02 &\cellcolor[RGB]{123.5769, 182.323, 220.9614}0.014 &\cellcolor[RGB]{194.1923, 221.3733, 239.2508}\underline{0.023} &\cellcolor[RGB]{105.9231, 172.5605, 216.3891}\underline{0.013} &\cellcolor[RGB]{88.2692, 162.7979, 211.8167}\underline{0.012} &\cellcolor[RGB]{70.6154, 153.0353, 207.2444}\underline{0.009} &\cellcolor[RGB]{52.9615, 143.2727, 202.672}0.007\\\hline\bottomrule \end{tabular} \end{adjustbox} \caption{Results from experiments on the fourteen objects of the RBO dataset \cite{1806.06465}. We compare our method against {\tt Prox.-ADMM} and show the median \y{rms} across one hundred trials. The columns are colour-coded by accuracy, darker indicates higher accuracy, the most accurate method is highlighted in bold+white, the second-most accurate method is in bold; across rows, the most accurate method is underlined. All values in metres.}   \label{tab_rbo_all} \end{table}

\subsubsection{\y{nsc}}\label{sec_exp_gsftp}

The second set of our experimental results are from the \y{nsc} case, where the pose of the multiple viewpoint centres are unknown.

\paragraph{\y{mcp} with \y{nsc}} \label{par_mocap_gsftp}
Using the {\tt CMU-MoCap} dataset, we create the \y{nsc} setup in five configurations. Each configuration has between one and twenty viewpoints, and in each configuration, each viewpoint has the same number of non-stereo correspondences which always sums (across viewpoints) to one-hundred. The number of viewpoints and per-viewpoint correspondences for each configuration are detailed below:
\begin{itemize}
    \item \textbf{Configuration 1}. Two viewpoints, fifty correspondences each
    \item \textbf{Configuration 2}. Three viewpoints, thirty-three correspondences on the first two and thirty-four on the third viewpoint
    \item \textbf{Configuration 3}. Four viewpoints, twenty-five correspondences each
    \item \textbf{Configuration 4}. Ten viewpoints, ten correspondences each
    \item \textbf{Configuration 5}. Twenty viewpoints, five correspondences each
\end{itemize}
\noindent Each set of correspondences in each experiment are randomly drawn, the viewpoints are assigned with random roto-translation around the object. 

Due to lack of similar methods to compare \y{nsc} against, we compare the result of \y{nsc} against the solution of \y{ns} repeated for a single viewpoint across all the views, e.g.: for {\tt configuration-1}, \y{ns} is done twice with fifty correspondences each, for {\tt configuration-4}, \y{ns} is done ten times with 10 correspondences each, etc. We term this approach of repeated \y{ns} across viewpoints without inter-frame constraints, which is indeed a trivial approach to the \y{nsc} problem, as \y{ns}$^{\ast}$. The experiments are repeated ten times for each configuration and the resulting accuracies of \y{nsc} against \y{ns}$^{\ast}$ are reported in \cref{fig:art_real_mult}. Qualitative results for experiments on {\tt subject-9} are shown in \cref{fig_subj9_mc_qual} and for experiments on {\tt subject-35} are shown in \cref{fig_subj35_mc_qual}.

\begin{figure}[b]
\centering
\begin{subfigure}{0.5\textwidth}
  \centering
  \includegraphics[width=\textwidth]{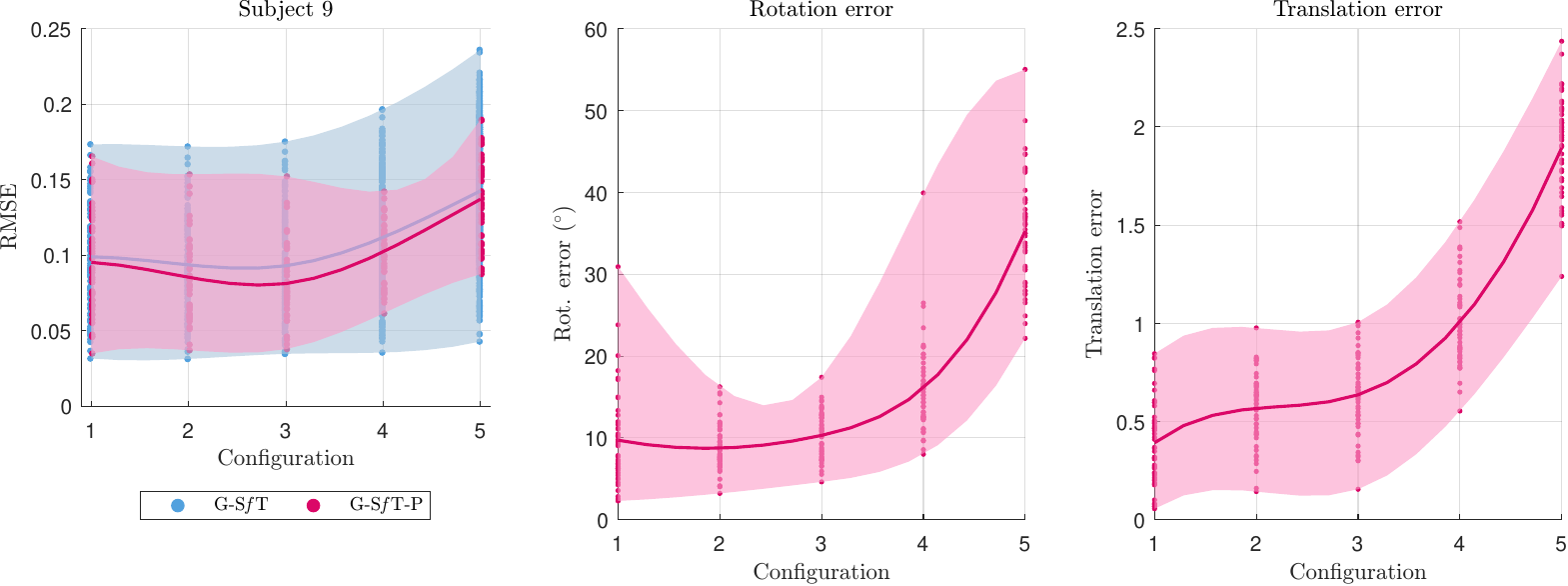}
  \caption{}
  \label{fig:sub_9_mult}
\end{subfigure}%
\begin{subfigure}{0.5\textwidth}
  \centering
  \includegraphics[width=\textwidth]{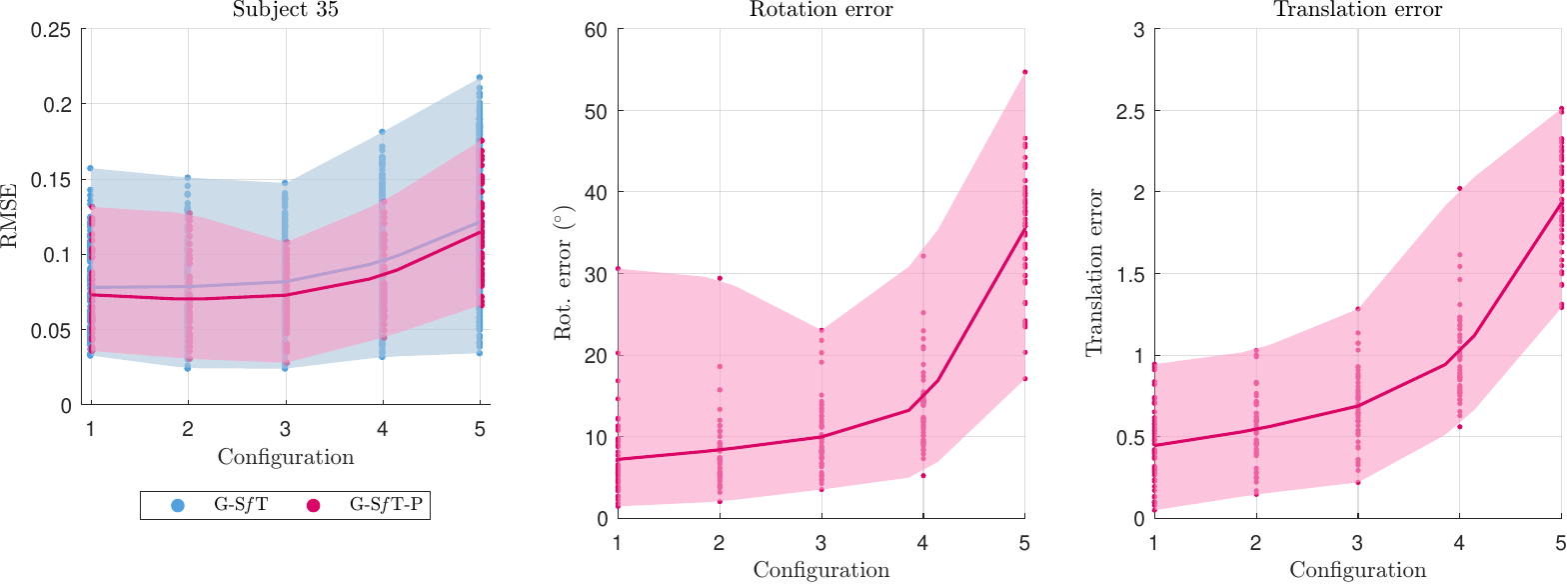}
  \caption{}
  \label{fig:sub_35_mult}
\end{subfigure}%
\caption{Results from the {\tt CMU-\y{mcp}} data with the \y{nsc} setup. (a) shows results obtained from {\tt subject-9} while (b) shows results obtained from {\tt subject-35}. The \textcolor{blueshade}{\textbf{blue}} line (and shaded area) represents the results from repeated \y{ns} across viewpoints (i.e., \y{ns}$^{\ast}$), the \textcolor{redshade}{\textbf{red}} line (and shaded area) represents the results from \y{nsc}}
\label{fig:art_real_mult}
\end{figure}

The important conclusion from \cref{fig:art_real_mult} is the fact that the accuracy of \y{nsc} is somewhat affected by splitting correspondences across cameras while with \y{ns}$^{\ast}$, the effect is pronounced, especially in {\tt configuration-2}, {\tt configuration-3} and {\tt configuration-4}. In {\tt configuration-1}, both methods are almost equally accurate, in {\tt configuration-5}, both methods are almost equally bad. However, \y{nsc} always maintains a significantly lower \y{sd} of \y{rms} than \y{ns}$^{\ast}$ and the mean \y{rms} of \y{nsc} always remains better than \y{ns}$^{\ast}$. Moreover, the difference between \y{ns}$^{\ast}$ and \y{nsc} is somewhat diminished due to random sampling of many keypoints from across the shape (\textit{c.f.}: out next experiment in \cref{sec_def_gsftp}, which shows a significant difference). 

Apart from reconstructing the shape of imaged structure, \y{nsc} also computes the pose of the viewpoints from which the \y{sl}s were images. In \cref{fig:art_real_mult}, we also report the accuracy of such pose estimation in terms of: a) \textit{rotational error} expressed as the difference in orientation of the estimated viewpoint pose with its \y{gt} represented by the mean Euler angle (in degrees), and b) \textit{translational error} expressed as the L$_2$-norm of the difference of the estimated viewpoint position in $\mathbb{R}^3$ with its \y{gt}. For both {\tt subject-9} and {\tt subject-35}, rotational accuracy remains quite stable at $\sim$ 10$^{\circ}$ between {\tt configuration-1} through {\tt configuration-4} but the accuracy drops sharply at {\tt configuration-5}. The translational accuracy is less resilient and shows a nearly linear decrease in accuracy with the configurations. 

The difference between the rotational and translational accuracy is explainable as follows. With lesser correspondences, the solution of \cref{eqn_final_sdp_unknown} often produces Gram matrices $\{\Delta_x^{\ast}\}$ such that $\mathrm{rank}(\Delta_x^{\ast}) > 1$, i.e., a rank-relaxed solution. Our empirical analysis of the results in \cref{fig:art_real_mult} suggests the most common artefact causing this rank relaxation is a scale-drift, i.e., say $\mathbf{R}_x^{\dagger}$ is the desired rotation matrix to be obtained from \cref{eqn_final_sdp_unknown}, for fewer correspondences (higher configurations), we obtain an $\mathbf{R}_x^{\ast} = s_x^{\ast} \mathbf{R}_x^{\dagger}$ from $\Delta_x^{\ast}$, where $s_x^{\ast}$ is a spurious scale artefact introduced due to the rank-relaxation. When $s_x^{\ast}$ is transmitted to the pose computation, the orientation is merely scaled but conveys accurate direction, however, the translation is wrongly affected resulting in spurious forward translation of the viewpoint centre if $s_x^{\ast} < 1$ and spurious backward translation if $s_x^{\ast} > 1$. Therefore, an observation is that our proposed solution method for \y{nsc} results in a method that is more accurate as an orientation estimator than an absolute pose estimator, but the cause is a side-effect of the solution methodology and not inherent to the \y{nsc} problem setup. Nonetheless, both pose and orientation gets estimated accurately in our solution for the cases with slightly higher correspondences.

\paragraph{\y{nsc} on deformed surfaces}\label{sec_def_gsftp}
We demonstrate an example of using \y{nsc} on deformed surfaces in \cref{fig_teaser} with the {\tt deformed-paper} data. The $\sum_{x=1}^3 n_x = 40$ keypoints in the template of {\tt deformed-paper} are observed across three images in batches of $n_1 = 6$, $n_2 = 19$ and $n_3 = 15$ correspondences with no overlap. These three images are captured from three different points in space with a wide pose-baseline. As with all \y{nsc} problem, there is a strong deformation between the template and the object, but in between the three images, the object is rigid, i.e., in a single deformed configuration. 

As comparison, we run all the seven surface-based \y{sft} methods on all the three patches. For a given patch, the reconstructed shape of all 40 keypoints is recovered via a \y{tps} \cite{bartoli2010generalized} warp from the $n_x < 40$ keypoints, for all $x \in [1,3]$. Naturally, we are left with three separate reconstructions for each method, which are then combined in the Euclidean sense to derive a mean shape \cite{bai2022procrustes}. \Cref{fig_teaser} shows the outcome of this reconstruction, with the result from our proposed \y{nsc} method in \cref{fig_teaser}h, which has an accuracy of 3.99 \textit{mm}, significantly ahead of the next best accuracy from {\tt CPB14IR} of 9.03 \textit{mm} and the other methods are even worse, demonstrating a clear success of our proposed \y{nsc} approach. The relative camera pose recovered from \y{nsc} is shown in \cref{fig_teaser}g.

{\color{revCol1}
\paragraph{\y{nsc} on $\varphi$-\y{sft} data}\label{sec_phi_sft}

\noindent While a direct comparison between our generalised camera setup for \y{sft} and existing approaches is not feasible owing to the distinctions outlined in \cref{tab_sft_baseline} and the overarching lack of prior methods tailored to generalised camera-based \y{sft} - we provide a comparative evaluation in the special case of $P=1$ using the $\varphi$-\y{sft} dataset. This enables a limited yet meaningful point of reference against recent \y{sft} techniques such as \cite{kairanda2022f, stotko2024physics}. However, given the fundamental differences in input modality - our method relies on point correspondences, whereas \cite{kairanda2022f, stotko2024physics} operate directly on image intensities - we recommend interpretation of the comparison with appropriate caution.

\vspace{2mm}
\noindent \textbf{Data pre-processing.} The \y{nsc} framework requires two key components that are not directly available in the $\varphi$-\y{sft} dataset: (a) feature correspondences between the template and the imaged deformed object, and (b) a representative population of deformed 3D structures suitable for learning the \y{ssm}.

To address the challenge of \textit{feature matching}, we leverage the pre-trained transformer-based architecture proposed by \cite{edstedt2024roma}, which yields approximately 10,000 highly accurate, semi-dense correspondences between the first frame and all subsequent frames in each scene of the $\varphi$-\y{sft} dataset. Subsequently, we enhance the provided ground truth (\y{gt}) meshes by applying the mesh subdivision technique of \cite{boye2010least}, resulting in refined meshes containing 1226, 1279, 1227, 1223, 730, 1010, 1028, 995, and 1069 vertices for the nine respective sequences.

These densified meshes are projected onto the image plane, and vertices are retained if they fall within a 1.5-pixel radius of any of the aforementioned $\sim$10,000 correspondences (as identified in the first image of each sequence using \cite{edstedt2024roma}). For instance, in a scene comprising 50 frames, this procedure yields 49 unique correspondence sets (i.e., between the first frame and each of the remaining frames). The average number of retained correspondences per scene, after projection and filtering, are 275.06, 337.21, 238.69, 171.58, 233.35, 142.21, 142.18, 148.86, and 155.17, respectively.

To \textit{learn the \y{ssm}}, we align the densified meshes to the provided ground truth point clouds via the classical non-rigid registration method of \cite{myronenko2010point}. For each experiment, a random subset of 25 non-rigidly registered meshes is sampled, and a shape model is learned using standard \y{pca}-based technique.

\vspace{2mm}
\noindent \textbf{Accuracy of \y{nsc}}. The predominant accuracy metric adopted in prior work \cite{kairanda2022f, stotko2024physics} is the \y{cd} between the reconstructed surface and the \y{gt} pointcloud. However, due to the pre-processing procedure described above, we gain access to an additional and informative error measure: the \y{rms} distance between the reconstructed mesh generated via \y{nsc} and the non-rigidly registered mesh obtained using the method of \cite{myronenko2010point}.

\begin{table}[t] \begin{adjustbox}{width=0.5\textwidth,center} \begin{tabular}{crccccccccc}  \toprule \hline & & \cellcolor[RGB]{200,200,200} \textbf{S1} & \cellcolor[RGB]{200,200,200} \textbf{S2} & \cellcolor[RGB]{200,200,200} \textbf{S3} & \cellcolor[RGB]{200,200,200} \textbf{S4} & \cellcolor[RGB]{200,200,200} \textbf{S5} & \cellcolor[RGB]{200,200,200} \textbf{S6} & \cellcolor[RGB]{200,200,200} \textbf{S7} & \cellcolor[RGB]{200,200,200} \textbf{S8} & \cellcolor[RGB]{200,200,200} \textbf{S9}\\ \cline{3-11} \multirow{2}{*}{\textbf{CD}} &\textbf{Mean}  $\downarrow$  & \cellcolor[RGB]{28.6875, 129.8492, 196.3851}\textbf{3.09} &\cellcolor[RGB]{0, 113.985, 188.955}\textbf{\textcolor{white}{2.49}} &\cellcolor[RGB]{86.0625, 161.5776, 211.2452}4.79 &\cellcolor[RGB]{229.5, 240.8985, 248.3955}58.63 &\cellcolor[RGB]{200.8125, 225.0343, 240.9654}16.46 &\cellcolor[RGB]{57.375, 145.7134, 203.8151}4.1 &\cellcolor[RGB]{172.125, 209.1701, 233.5354}10.84 &\cellcolor[RGB]{114.75, 177.4417, 218.6753}6.23 &\cellcolor[RGB]{143.4375, 193.3059, 226.1053}7.11\\ & \textbf{SD}  $\downarrow$  & \cellcolor[RGB]{0, 113.985, 188.955}\textbf{\textcolor{white}{1.23}} &\cellcolor[RGB]{28.6875, 129.8492, 196.3851}\textbf{1.36} &\cellcolor[RGB]{114.75, 177.4417, 218.6753}3.94 &\cellcolor[RGB]{229.5, 240.8985, 248.3955}34.98 &\cellcolor[RGB]{200.8125, 225.0343, 240.9654}9.26 &\cellcolor[RGB]{57.375, 145.7134, 203.8151}1.97 &\cellcolor[RGB]{172.125, 209.1701, 233.5354}5.64 &\cellcolor[RGB]{143.4375, 193.3059, 226.1053}4.39 &\cellcolor[RGB]{86.0625, 161.5776, 211.2452}2.7\\\multirow{2}{*}{\textbf{RMSE}} &\textbf{Mean}  $\downarrow$  & \cellcolor[RGB]{28.6875, 129.8492, 196.3851}\textbf{0.02} &\cellcolor[RGB]{0, 113.985, 188.955}\textbf{\textcolor{white}{0.02}} &\cellcolor[RGB]{86.0625, 161.5776, 211.2452}0.03 &\cellcolor[RGB]{229.5, 240.8985, 248.3955}0.14 &\cellcolor[RGB]{172.125, 209.1701, 233.5354}0.05 &\cellcolor[RGB]{57.375, 145.7134, 203.8151}0.03 &\cellcolor[RGB]{200.8125, 225.0343, 240.9654}0.06 &\cellcolor[RGB]{114.75, 177.4417, 218.6753}0.04 &\cellcolor[RGB]{143.4375, 193.3059, 226.1053}0.04\\ & \textbf{SD}  $\downarrow$  & \cellcolor[RGB]{0, 113.985, 188.955}\textbf{\textcolor{white}{0.01}} &\cellcolor[RGB]{28.6875, 129.8492, 196.3851}\textbf{0.01} &\cellcolor[RGB]{86.0625, 161.5776, 211.2452}0.02 &\cellcolor[RGB]{229.5, 240.8985, 248.3955}0.06 &\cellcolor[RGB]{143.4375, 193.3059, 226.1053}0.02 &\cellcolor[RGB]{57.375, 145.7134, 203.8151}0.01 &\cellcolor[RGB]{200.8125, 225.0343, 240.9654}0.03 &\cellcolor[RGB]{172.125, 209.1701, 233.5354}0.03 &\cellcolor[RGB]{114.75, 177.4417, 218.6753}0.02\\\hline\bottomrule \end{tabular} \end{adjustbox} \caption{{\color{revCol1}\y{cd} and \y{rms} of \y{nsc} on the nine scenes of $\varphi$-\y{sft} dataset, the \y{cd} is obtained in metres and multiplied by $10^4$ following convention from \cite{stotko2024physics}, \y{rms} is in metres. The colour-coding of the cells follow the same scheme as \cref{tab_rbo_all}}}   \label{tab_phi_sft_res} \end{table} 

Notably, during these experiments, we observed again, the rank-relaxation artifact in the solution to \cref{eqn_final_sdp_unknown}, manifested as a solution Gram matrix $\Delta_1^{\ast}$ with $\mathrm{rank}(\Delta_1^{\ast}) > 1$, somewhat similar to what had been reported in \cref{par_mocap_gsftp}. This behaviour is a well-documented limitation in \y{sdp}-based rank minimization problems \cite{dattoroConvex}. In practical terms, this means that although a feasible solution $\Delta_1^{\ast}$ to \cref{eqn_final_sdp_unknown} has been obtained, the rotation matrix $\mathbf{R_1}^{\ast}$ recovered from $\Delta_1^{\ast}$ no longer satisfies the orthonormality constraint, i.e., $\mathbf{R_1}^{\ast}(\mathbf{R_1}^{\ast})^{\top} \neq \mathbbm{1}_3$. As a result, the absolute pose is lost, though the underlying structure is still recovered with reasonable fidelity.

To address this, and following standard practice in monocular 3D reconstruction, we estimate the pose post hoc by solving the orthogonal Procrustes problem \cite{horn1988closed}. For quantitative evaluation of the reconstruction quality, we report both the mean and \y{sd} of the \y{cd} (between the reconstructed surface and the \y{gt} point cloud) and the \y{rms} (between the reconstruction and the non-rigidly aligned mesh) in \cref{tab_phi_sft_res} while one randomly selected frame from each scene is displayed qualitatively in \cref{fig:phi_sft}.

\begin{figure}[t]
\centering
\begin{subfigure}{0.7\textwidth}
  \centering
  \includegraphics[width=\textwidth]{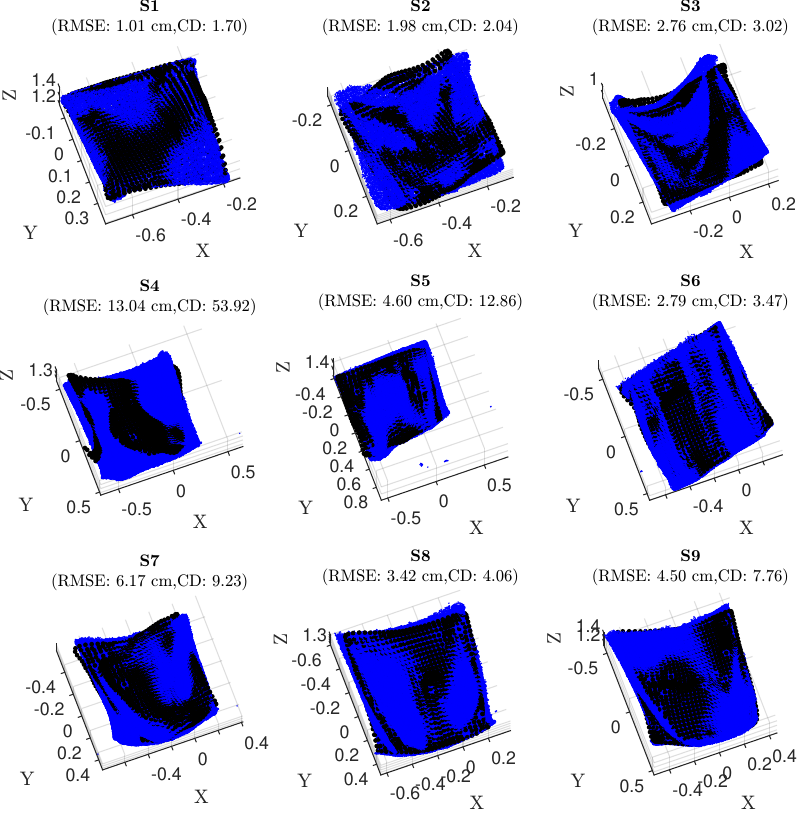}
  \caption{}
  \label{fig:phi_sft_qual}
\end{subfigure}%
\caption{{\color{revCol1}One randomly selected frame and our reconstruction result from each scene of the $\varphi$-\y{sft} dataset, the blue pointcloud is the \y{gt}, the black points are the vertices of the mesh reconstructed from \y{nsc}. \y{cd} is obtained in metres and multiplied by $10^4$ following convention from \cite{stotko2024physics}, \y{rms} is in centimetres}}
\label{fig:phi_sft}
\end{figure}

% \begin{wrapfigure}{r}{0.45\textwidth}
% \begin{overpic}[width=0.43\textwidth]{results/PhiSfT_2_c.pdf}
% \end{overpic} 
% \caption{xxx}
% \label{fig:phi_sft}
% \end{wrapfigure}

\vspace{2mm}
\noindent\textbf{Insights gained from applying \y{nsc} to the $\varphi$-\y{sft} dataset}~The \y{nsc} method demonstrates strong reconstruction accuracy on sequences S1, S2, S3, and S6, with reduced performance observed in S5, S7, S8, and S9, and notably lower accuracy in S4. This trend aligns with results in \cite{kairanda2022f, stotko2024physics}, where S4 and S5 appeared among the less accurately reconstructed scenes, suggesting they posed particular challenges. When viewed alongside the results of \cite{stotko2024physics} - a more recent approach - the performance of our proposed \y{nsc} appears improved in S3, marginally lower across several other sequences, and substantially lower in S4. Importantly, our method was able to reconstruct all sequences successfully. We emphasize yet again that these comparisons should be interpreted with care, as our method operates on a fundamentally different form of input (point correspondences) compared to the direct image-based inputs used in \cite{kairanda2022f, stotko2024physics}. As such, these results are intended to offer context rather than direct performance equivalence.

}

\subsubsection{Silhouette-boosted \y{ns}}\label{sec_exp_gsft_sb}
We now describe results obtained by boosting the accuracy of \y{ns} with silhouettes.

\vspace{2mm}
\noindent \textbf{Synthetic data}. With the {\tt Stanford bunny}, {\tt Nefertiti}, and {\tt Spot}, we feed the silhouettes of the object and use the solution from \cref{sec_synth_sol} for silhouette-boosted \y{sft}. We try with five configurations, with point correspondences from 3 through 7. 

We show the evolution of accuracy, in terms of \y{rms}, across iterations in \cref{fig:iter_silhouette} and some sample, qualitative results in \cref{fig:sil_synth_qual_bunny}, \cref{fig:sil_synth_qual_nefertiti}, and \cref{fig:sil_synth_qual_spot} for the three models. The accuracy from silhouette-boosting, for three and four correspondences are somehow of unclear utility and there remains a few instances where the accuracy decreases. This is unsurprising since with three or four correspondences, the initial estimate of deformation is inaccurate and empirically insufficient for reconstructing the complex shapes of {\tt Stanford bunny}, {\tt Nefertiti}, or {\tt Spot}. But with higher correspondences, the benefit of silhouette boosting becomes obvious. The accuracy improves in all cases, sometimes significantly. Importantly, the solution converges to a stationary point in all cases (even with lesser correspondences, although it takes larger number of iterations).

\begin{figure}[h]
\begin{subfigure}{\textwidth}
  \centering
  \includegraphics[width=0.9\textwidth]{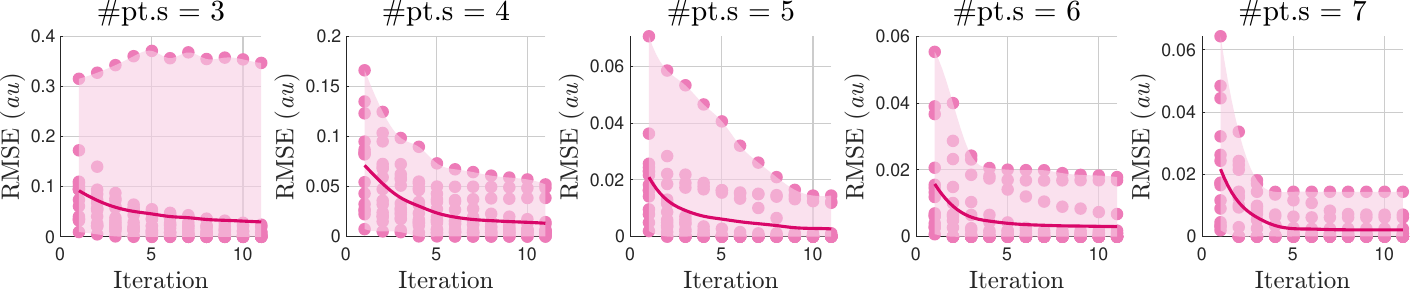}
  \caption{Stanford bunny}
  \label{fig:bunny_silhouette}
\end{subfigure}%

\begin{subfigure}{\textwidth}
  \centering
  \includegraphics[width=0.9\textwidth]{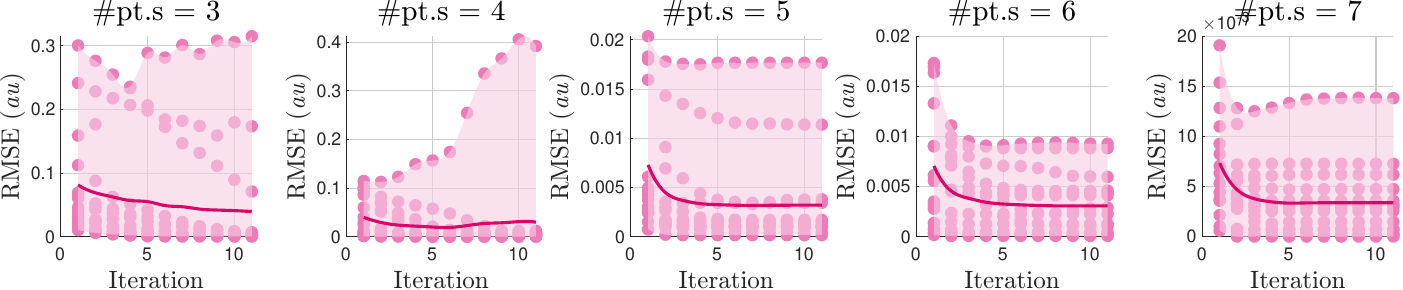}
  \caption{Nefertiti}
  \label{fig:nefertiti_silhouette}
\end{subfigure}%

\begin{subfigure}{\textwidth}
  \centering
  \includegraphics[width=0.9\textwidth]{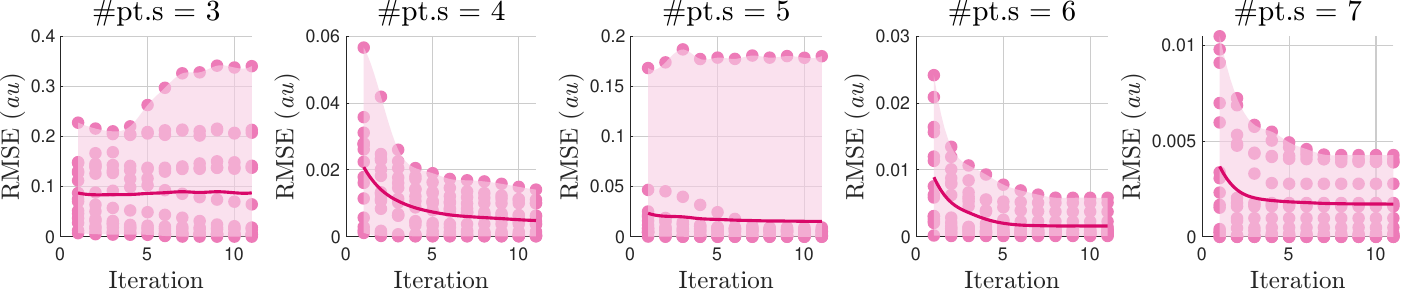}
  \caption{Spot}
  \label{fig:spot_noiseless}
\end{subfigure}%
\caption{Results obtained by silhouette-boosting \y{ns} for the synthetic 3D models (a) {\tt Stanford bunny}, (b) {\tt Nefertiti}, and (c) {\tt Spot}. The five columns represent the accuracy over iterations with correspondences numbering 3, 4, 5, 6, and 7 respectively}
\label{fig:iter_silhouette}
\end{figure}

\vspace{2mm}
{\color{revCol1}\noindent\textbf{Effect of noise on silhouettes}. To evaluate the robustness of silhouette-augmented \y{ns} to noise in silhouette detection, we selected the {\tt Spot} 3D model due to its relatively higher shape complexity. The experimental setup follows that described previously and illustrated in \cref{fig:sil_synth_qual_spot}. In this setting, silhouettes are generated by introducing Gaussian noise with a standard deviation of 0.005 \y{au} to the internal points. While this may appear to be a minor perturbation in 3D space, when projected onto the image plane using a camera model with intrinsics similar to the {\tt deformed-paper} sequence, it corresponds to an average reprojection error of approximately 34.2 pixels - an appreciable level of noise.

We report the reconstruction accuracy of both standard \y{ns} and silhouette-boosted \y{ns} under three correspondence configurations. In each case, the first camera provides two correspondences, while the second camera contributes 3, 4, and 5 correspondences, respectively - yielding a total of 5, 6, and 7 correspondences across the configurations. Results are summarized in \cref{fig:noisy_sillh}.

In the absence of noise, silhouette-boosting yields substantial improvements in reconstruction accuracy, with gains approaching $\sim$99\% across all configurations. Under noisy conditions, silhouette-boosted \y{ns} generally converges to more accurate solutions than standard \y{ns}. However, a subset of cases exhibited reduced accuracy following silhouette boosting. These instances correspond to scenarios where the initial \y{ns} reconstruction error exceeded $\sim$0.05 \y{au} \y{rms}. This outcome is not unexpected, as silhouette-boosted \y{ns} is a non-convex optimization, and poor initializations can lead to suboptimal convergence. Thus, these results underscore the well-known importance of initialization quality in non-convex refinement and demonstrate both the potential and limitations of silhouette-based \y{ns} under noisy conditions.}

\begin{figure}[h]
\centering
\begin{subfigure}{0.5\textwidth}
  \centering
  \includegraphics[width=0.9\textwidth]{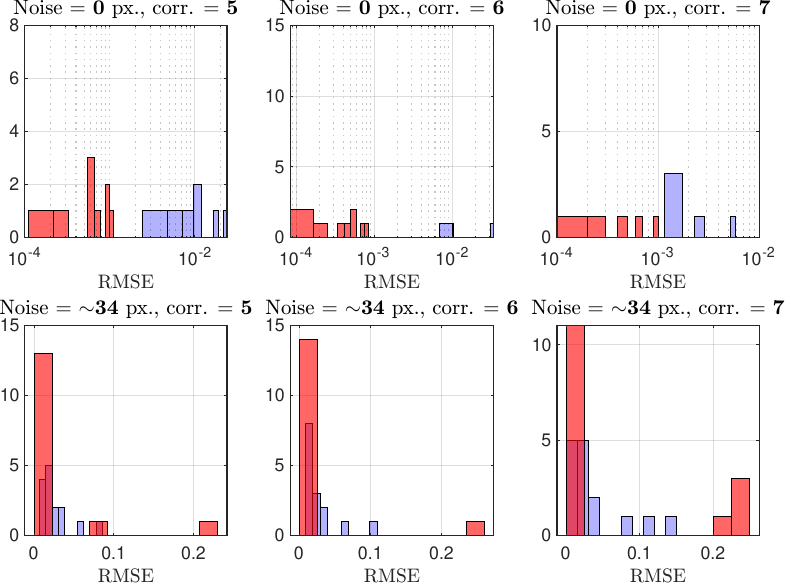}
  \caption{}
  \label{fig:sub_noisy_sil1}
\end{subfigure}%
\begin{subfigure}{0.5\textwidth}
  \centering
  \includegraphics[width=0.9\textwidth]{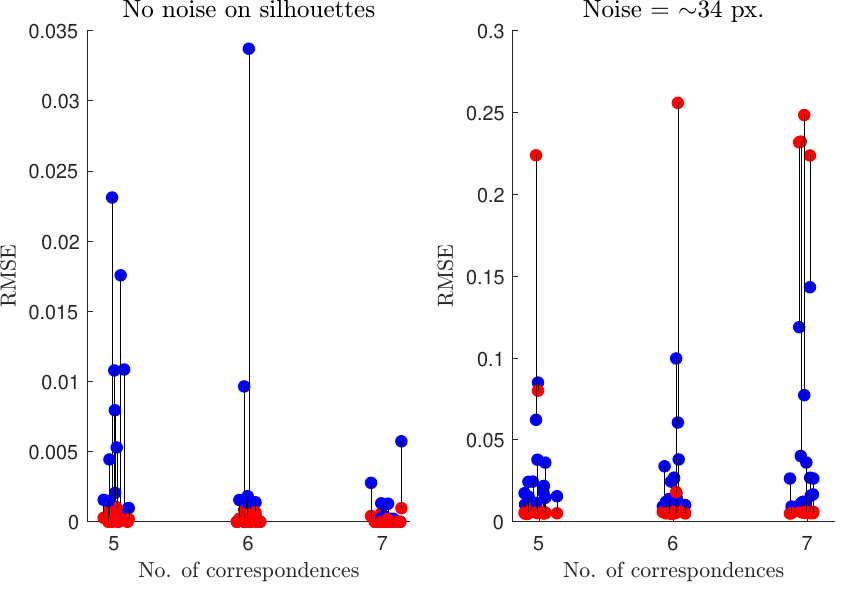}
  \caption{}
  \label{fig:sub_noisy_sil2}
\end{subfigure}%
\caption{{\color{revCol1} Effect of noise on silhouette-boosted \y{ns} for the {\tt Spot} 3D model. (a) \y{rms} error histograms: top row shows results without noise; bottom row shows results with added Gaussian noise (0.005 \y{au}, $\approx$34 pixels reprojected \y{rms} with typical intrinsics). Each row corresponds to 5, 6, and 7 correspondences, respectively. Blue bars indicate \y{ns}; red bars indicate results after silhouette boosting. (b) Plots linking each \y{ns} result (blue) to its corresponding silhouette-boosted result (red) by black lines, random jitter added to X-axis (showing correspondences) to improve readability.}}
\label{fig:noisy_sillh}
\end{figure}

\vspace{2mm}
\noindent \textbf{Real data - {\tt deformed-paper}}. We use the {\tt deformed-paper} proposed by us to validate the silhouette-boosted \y{ns} method with real data. Using the \y{ssm} generated out of eight deformed shape of the paper, we test the reconstruction accuracy of \y{ns} and silhouette-boosted \y{ns} on five different images of the paper, deformed in a configuration which was not learned by the \y{ssm}. The results are shown in \cref{silh_real}. Given that the \y{ssm} was learned from just eight poses, the \y{ns} problem was challenging and, without the silhouette boost, the reconstruction \y{rms} was 1.1 \textit{cm}, 1 \textit{cm}, 2 \textit{cm}, 1.8 \textit{cm}, and 1.6 \textit{mm} respectively, for the five images; therefore the accuracy was low. With silhouette boosting, the accuracies improved to 2.4 \textit{mm}, 2.2 \textit{mm}, 1.5 \textit{mm}, 1.5 \textit{mm}, and 1.9 \textit{mm} for the same five images; a significant magnitude of improvement. 

{\color{revCol1}
\vspace{2mm}
\noindent \textbf{Real data - biplanar X-ray}. As a comprehensive and realistic validation of our silhouette-boosted \y{ns} method, we employed the biplanar X-ray dataset from \cite{welte2022biplanar} to reconstruct the 3D morphology of the calcaneus bone. The experimental setup and corresponding results are presented in \cref{fig_biplanar}. The dataset provides an accurate 3D template of the calcaneus, serving as \y{gt}, derived from \y{ct} reconstruction, as illustrated in \cref{fig_biplanar}a. For one randomly selected X-ray image pair, three feature correspondences per image and silhouettes were semi-automatically annotated, with manual filtering and corrections. These annotations are visualized in the first two columns (from left) of \cref{fig_biplanar}a.

Due to the absence of a population dataset for calcaneus bones, we generated a synthetic population comprising nine deformed calcaneus models by applying harmonic splines \cite{joshi2007harmonic}. These deformations approximate scaling along the X, Y, and Z axes, shearing, torsion, and the manual protrusion or depression of bone regions via aligned control handles. This synthetic dataset is depicted in \cref{fig_biplanar}b.

Upon applying our silhouette-boosted \y{ns} method to this experimental framework, the reconstructed calcaneus shape achieved a root mean square (\y{rms}) error of 2.08 \textit{mm} relative to the ground truth (\y{gt}). The reconstructed morphology and corresponding error statistics are detailed in \cref{fig_biplanar}c and \cref{fig_biplanar}d. These findings underscore the potential efficacy and robustness of the silhouette-boosted \y{ns} methodology.

\begin{figure}[t]
\centering
\begin{subfigure}{\textwidth}
  \centering
  \includegraphics[width=0.9\textwidth]{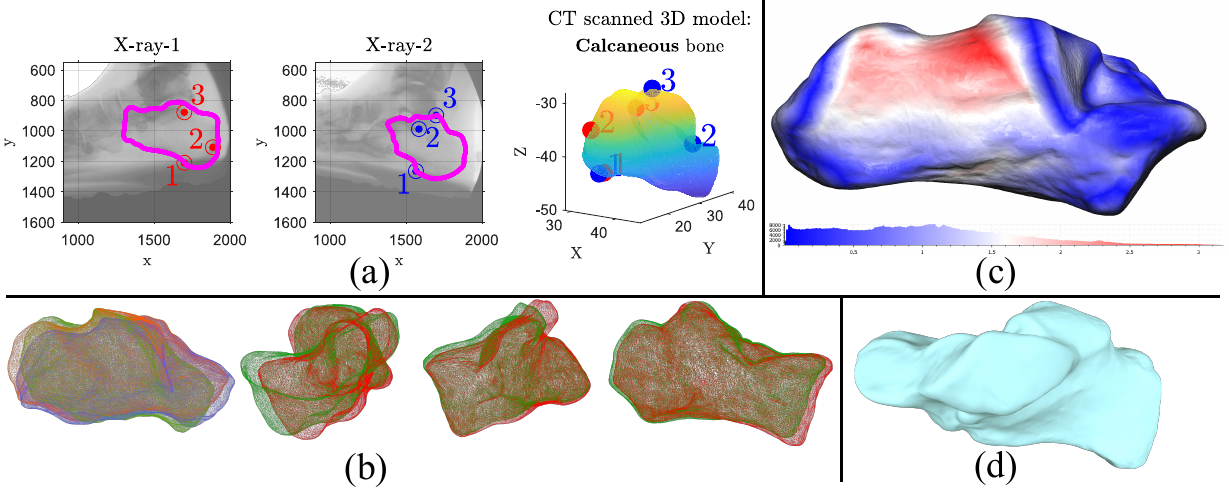}
\end{subfigure}%
\caption{{\color{revCol1}Results on the biplanar X-ray dataset from \cite{welte2022biplanar}. (a) Depicts input point correspondences and silhouettes: the first two images are X-rays; the third shows the \y{ct}-derived 3D \y{gt} model with annotated correspondences and silhouettes. (b) Illustrates synthetic calcaneus shape variations: scaling along principal axes (first), shearing and torsion (second), localised depressions (third), and protrusions (fourth); all shape variations are represented by meshes of contrasting colours. (c) Presents the reconstructed shape, colour-coded by Euclidean distance from the aligned \y{gt}; error histogram shown in \textit{mm}. (d) Displays an alternative view of the reconstructed shape.}}
\label{fig_biplanar}
\end{figure}
}

\begin{figure}[h]
\centering
\begin{overpic}[width=0.568\linewidth]{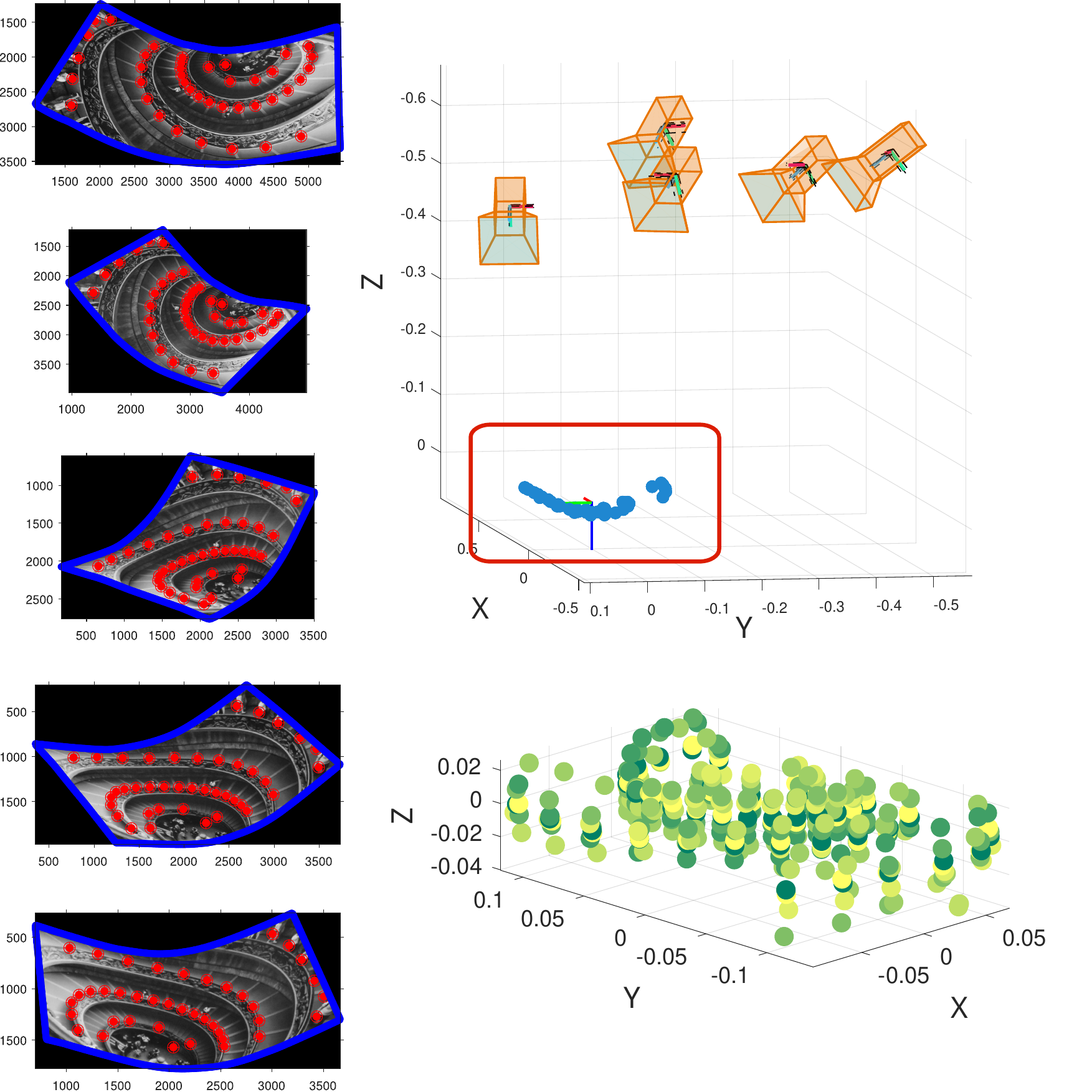}
    \put(20,-3){(a)}
    \put(70,40){(b)}
    \put(70,5){(c)}
    \put(49,62){3D \y{gt}}
\end{overpic}%
\begin{overpic}[width=0.232\linewidth]{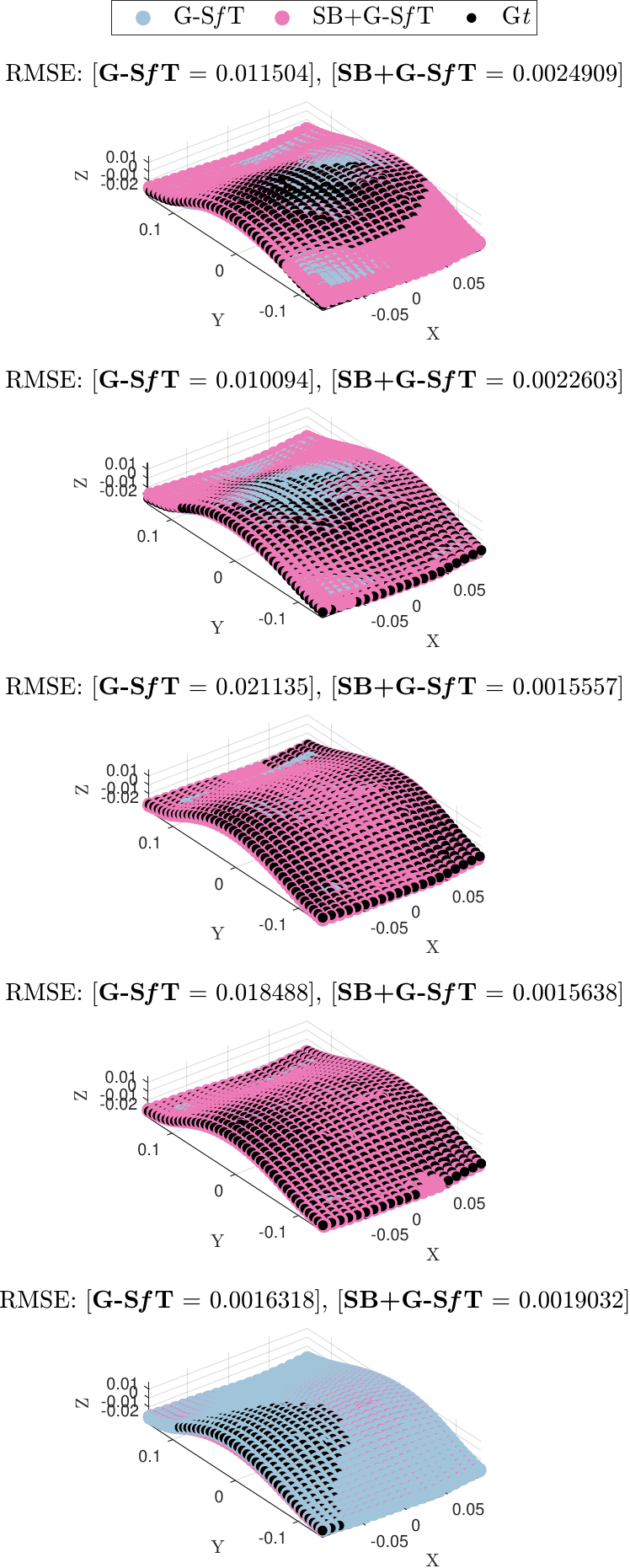}
    \put(20,-3){(d)}
\end{overpic}%
\vspace*{5mm}
\caption{Results from experiments with silhouette-boosted \y{ns}. (a) shows the input data, the red dots are the forty keypoints and the blue outline are the supplied silhouettes; (b) shows the \y{gt} 3D pointcloud obtained with \y{sfm} and the recovered \y{gt} camera poses (although the camera poses are not relevant for this particular experiment with \y{ns}, it has been visualised to clarify the experimental setup); (c) shows the 3D \y{gt} keypoints of eight other deformed poses of the paper \big(not the deformation shown in (a)\big) used for generating the \y{ssm}; (d) shows the results for the five images, the light-blue points denote the results from \y{ns}, the orchid points denote the results of silhouette-boosted \y{ns} (SB+\y{ns}) and the black points are \y{gt}, shown here are all the 844 keypoints used to generate the silhouette}
\label{silh_real}
\end{figure}

{\color{revCol1}
\subsubsection{Runtime analysis}
We evaluate the computational performance of our proposed methods, with particular emphasis on \y{nsc} - the generalised formulation of \y{ns} - as well as the silhouette-augmented variant of \y{ns}. As previously described in \cref{sol_ns}, all implementations rely on standard semidefinite programming (SDP) solvers via the CVX framework \citep{cvx, gb08} with MOSEK \citep{mosek}, executed within MATLAB. The reported timings correspond exclusively to solver runtimes; modelling overhead is excluded, as it is highly dependent on implementation details and largely optimisable.Three runtime evaluations are presented below.

\vspace{2mm}
\noindent \textcolor{revCol2}{\textbf{Effect of increasing number of views}}.~For \y{nsc}, we assess the impact of varying viewpoint count using the five configurations outlined in \cref{par_mocap_gsftp}, while fixing the total number of point correspondences to 100. As shown in \cref{fig_timing}a, solver time remains largely constant for 1, 2, 4, and 10 viewpoints. However, a pronounced increase is observed at 100 viewpoints - consistent with the degenerate and unsolvable configuration described in \cref{par_mocap_gsftp} - which substantially burdens the interior-point solver.

\vspace{2mm}
\noindent \textcolor{revCol2}{\textbf{Effect of increasing correspondences}}.~We fix the number of viewpoints to one and vary the number of correspondences from 20 to 100 in increments of 20. As depicted in \cref{fig_timing}b, runtime increases approximately linearly with the number of correspondences - a predictable outcome.

\vspace{2mm}
\noindent \textcolor{revCol2}{\textbf{Effect of increasing density of silhouette}}.~We analyse silhouette-boosted \y{ns}, where runtime variability is primarily driven by the number of silhouette points, which can be large. To ensure consistency, each instance was forced to run for exactly seven iterations (independent of convergence criteria), and the solver time was recorded. \Cref{fig_timing}c shows a clear linear correlation between solver time and the number of silhouette points, with each configuration repeated thrice. The trend is visually evident and requires no further statistical treatment.}

\vspace{2mm}
\color{revCol2}{
\noindent\textbf{Discussion on scalability}. The proposed methods exhibit favourable scalability across the key problem parameters of viewpoints, correspondences and silhouette density. Solver runtime remains largely stable with increasing numbers of camera viewpoints, except under degenerate configurations. Runtime grows approximately linearly with the number of 2D point correspondences and silhouette points, as expected from the underlying \y{sdp} formulation. These trends suggest the method is computationally efficient for moderate-scale problems and scales predictably with input size.}

\begin{figure}[t]
\centering
\begin{subfigure}{\textwidth}
  \centering
  \includegraphics[width=0.5\textwidth]{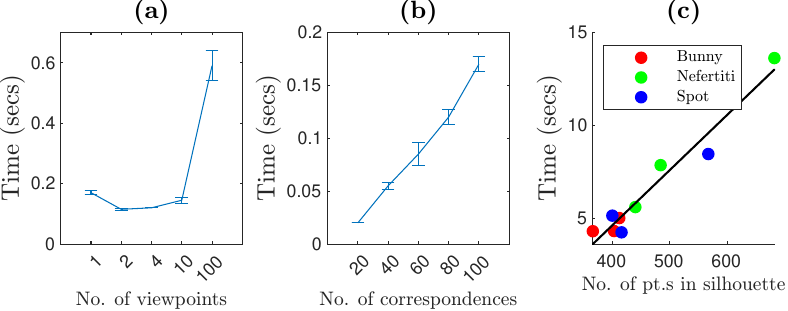}
\end{subfigure}%
\caption{{\color{revCol1}Runtime analysis of proposed methods.
(a) Solver time for \y{nsc} with increasing number of viewpoints (1, 2, 4, 10, 100) while keeping the number of correspondences fixed at 100. 
(b) Solver time for \y{nsc} with increasing number of correspondences (20 to 100 in steps of 20), with the number of viewpoints fixed at 1. 
(c) Solver time for silhouette-boosted \y{ns} as a function of the number of silhouette points. Each shape instance is run for exactly seven iterations and repeated three times. }}
\label{fig_timing}
\end{figure}

\color{revCol1}{\subsection{Discussion}\label{sec_disc}
The presented results demonstrate that the proposed methods effectively solve the \y{ns}, \y{nsc}, and silhouette-boosted \y{ns} formulations with high accuracy. While the approaches are robust across a wide range of scenarios, two intriguing challenges remain that present exciting opportunities for future investigation. The \textit{first} relates to the theoretical limitation observed in \cref{par_mocap_gsftp} and \cref{sec_phi_sft}: the inherent ambiguity in jointly recovering absolute camera pose and deformed shape, especially with $P = 1$, in \y{nsc}. The camera pose and deformation being coupled, the solution is non-unique. Thankfully, we successfully recover the deformed shape up to an arbitrary pose in all such cases, demonstrating the practical utility of our method even within this constraint. The \textit{second} challenge arises from the reliance on a learned \y{ssm}, as highlighted in \cref{fig:iso_synth}. Reconstruction accuracy can be affected when the target deformation lies outside the span of the learned shape bases. While this is a well-known limitation of \y{ssm}, it highlights the importance of expressive shape priors. Furthermore, as noted in \cref{app_challenge_phy}, extending the methodology to more general or non-statistical shape models while preserving convexity is a non-trivial endeavour. Developing such extensions represents a promising and impactful direction for future research.}

\vspace{2mm}
\color{revCol2}{
\noindent \textbf{Future research directions.}~Although the problem statements defined in \cref{subsec_nsc_first} and \cref{sec_unknown} have been solved convexly, several closely related problems remain open and merit further investigation. A particularly important direction is the lack of convex formulations for the silhouette-boosted variant (\cref{pron_silh_nsc}) and the absence of correspondence-free, silhouette-based \y{sft} methods. Both problems are technically challenging but potentially impactful. Beyond these, other critical challenges remain outside the scope of this article. These include the development of robust formulations capable of handling correspondence blunders, solutions for \y{sft} in the presence of uncalibrated or inaccurately calibrated cameras, and extensions toward correspondence-free \y{sft} formulations.}

\color{black}{
\section{Conclusion}
In this article, we described, discussed and experimentally validated the first set of method for \y{sft} with the generalised camera. In the presence of known poses of the origin of the \y{sl}s of the generalised camera, our method remains agnostic of the underlying projection model and successfully reconstructs the deformed object shape combining information from all the \y{sl}s. In the multi-view setup with unknown origin of the \y{sl}s of the generalised camera, our proposed approach effectively and accurately handles the reconstruction as well as relative camera pose estimation, providing the first method to do so. Moreover, in the presence of silhouette, the proposed silhouette-boosted \y{sft} method produces highly accurate results, even with very few keypoint correspondences.}

\begin{figure}
\centering
\begin{subfigure}{0.25\textwidth}
  \centering
  \includegraphics[width=0.95\textwidth]{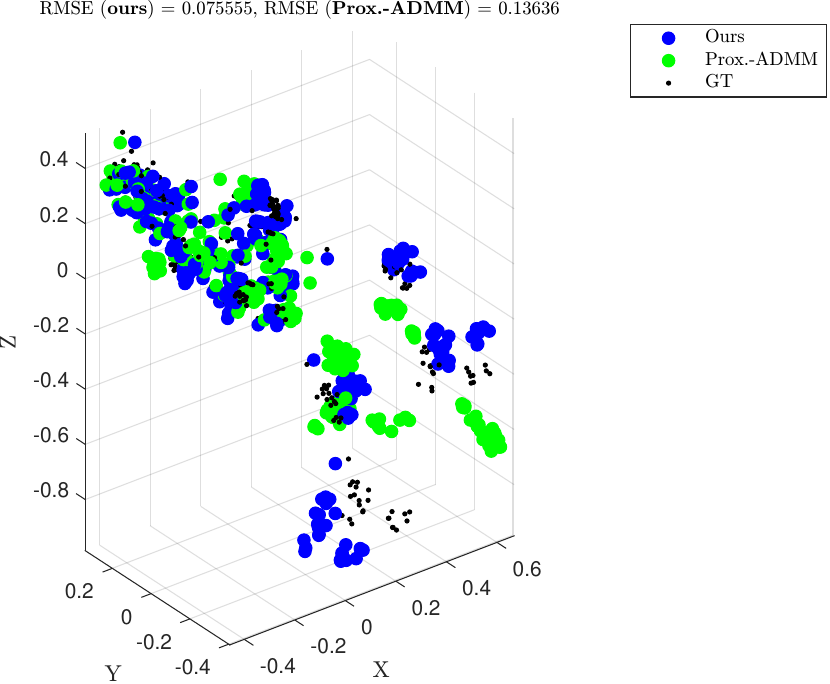}
  \label{fig:q9_1}
\end{subfigure}%
\begin{subfigure}{0.25\textwidth}
  \centering
  \includegraphics[width=0.95\textwidth]{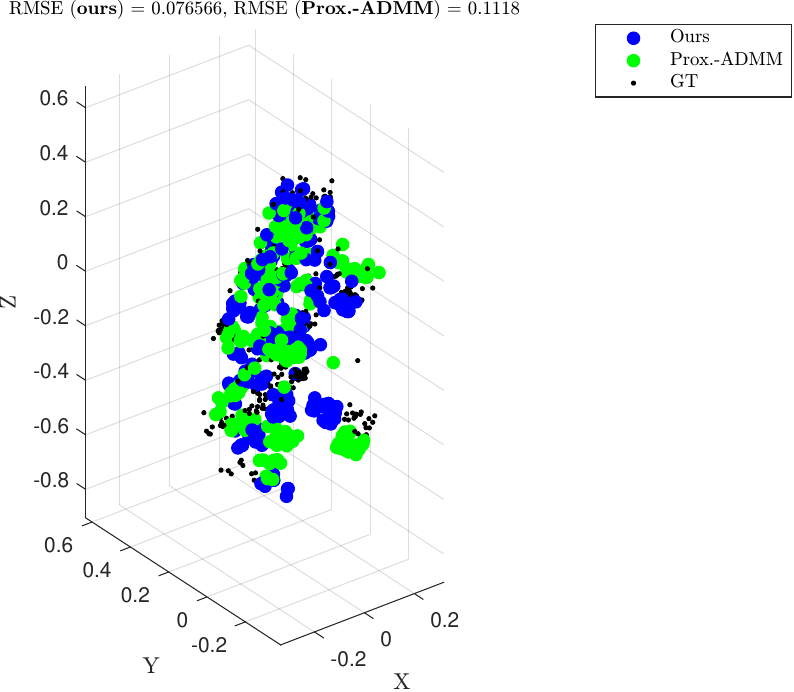}
  \label{fig:q9_3}
\end{subfigure}%
\begin{subfigure}{0.25\textwidth}
  \centering
  \includegraphics[width=0.95\textwidth]{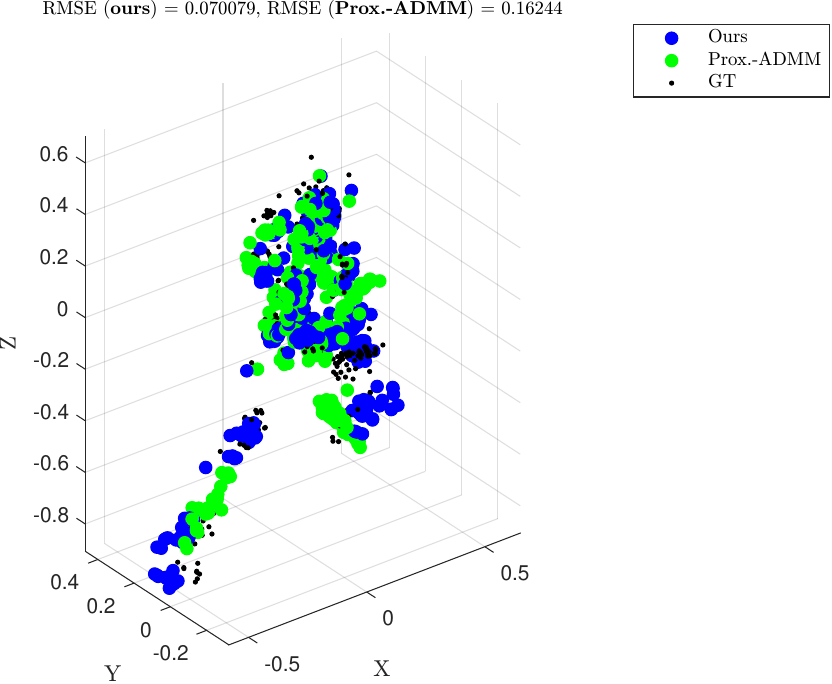}
  \label{fig:q9_5}
\end{subfigure}%
\begin{subfigure}{0.25\textwidth}
  \centering
  \includegraphics[width=0.95\textwidth]{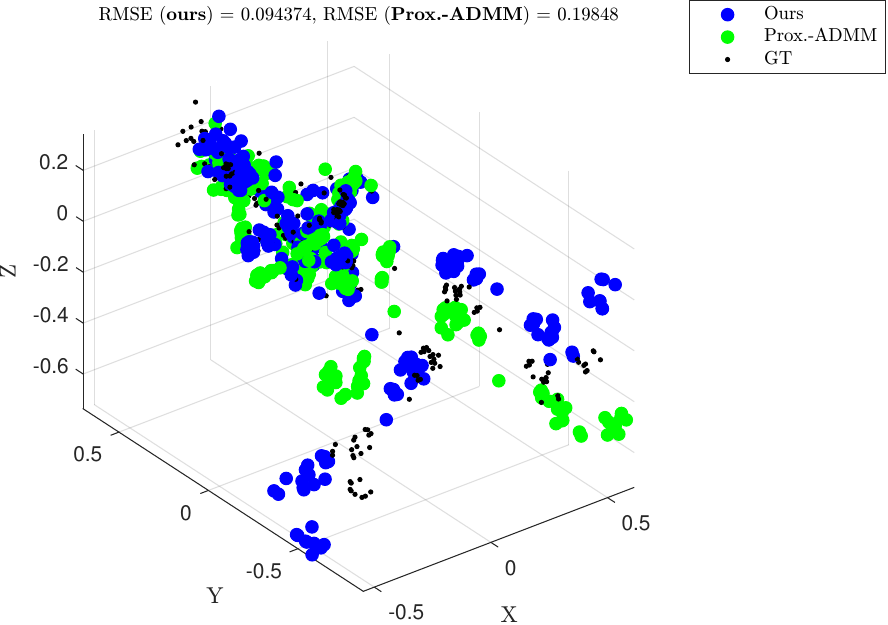}
  \label{fig:q9_7}
\end{subfigure}%

\begin{subfigure}{0.25\textwidth}
  \centering
  \includegraphics[width=0.95\textwidth]{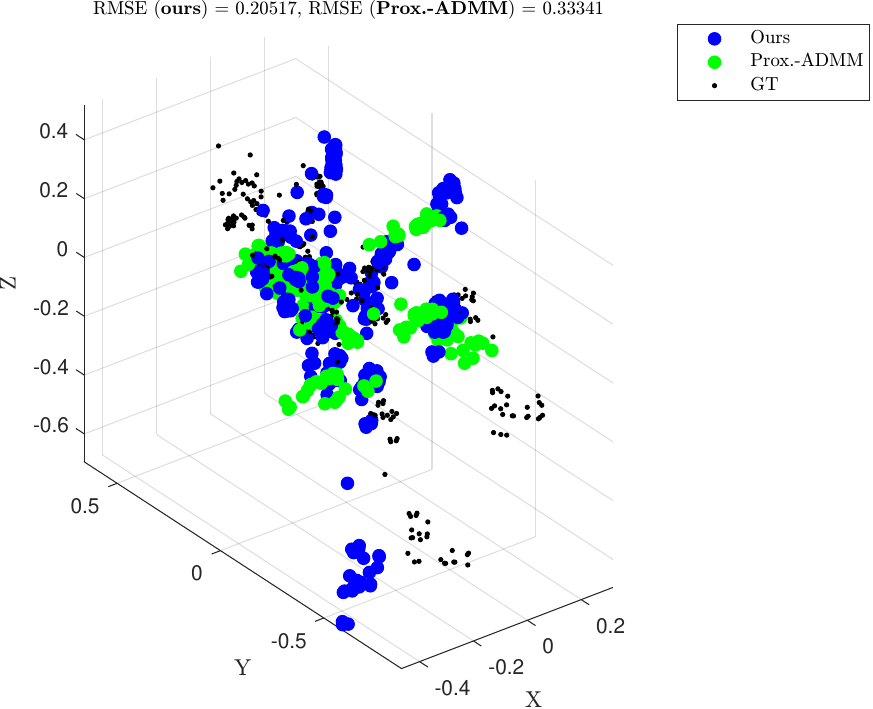}
  \label{fig:q9_2}
\end{subfigure}%
\begin{subfigure}{0.25\textwidth}
  \centering
  \includegraphics[width=0.95\textwidth]{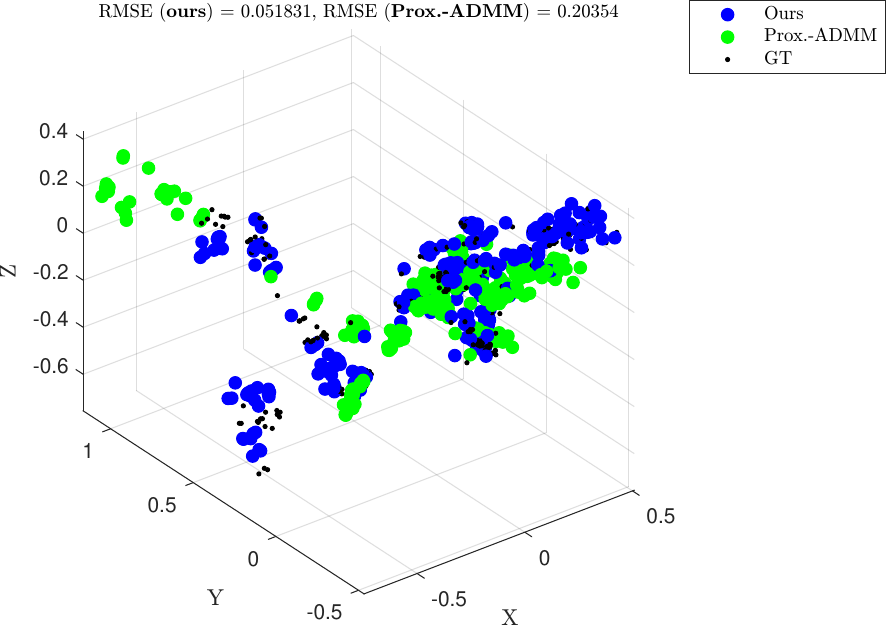}
  \label{fig:q9_4}
\end{subfigure}%
\begin{subfigure}{0.25\textwidth}
  \centering
  \includegraphics[width=0.95\textwidth]{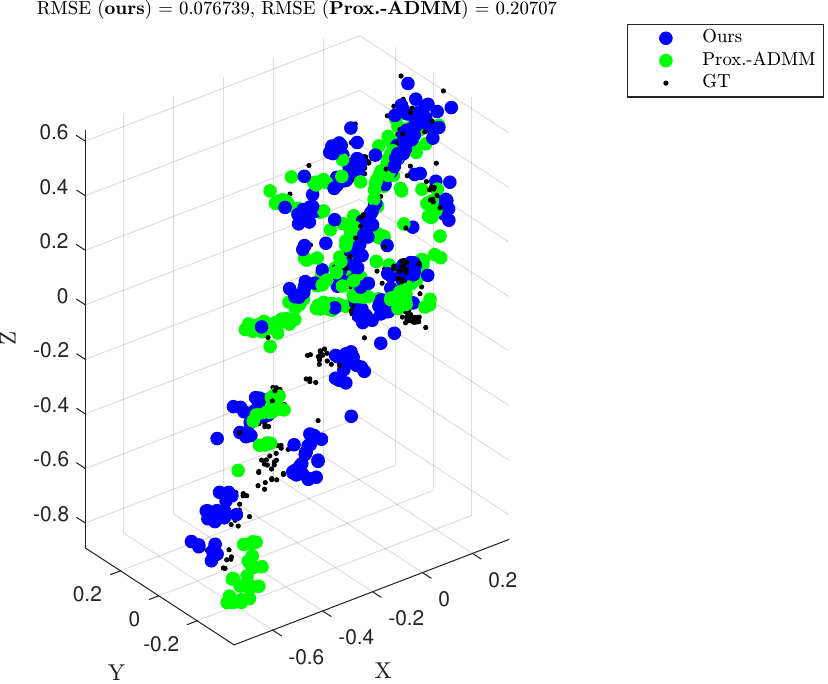}
  \label{fig:q9_6}
\end{subfigure}%
\begin{subfigure}{0.25\textwidth}
  \centering
  \includegraphics[width=0.95\textwidth]{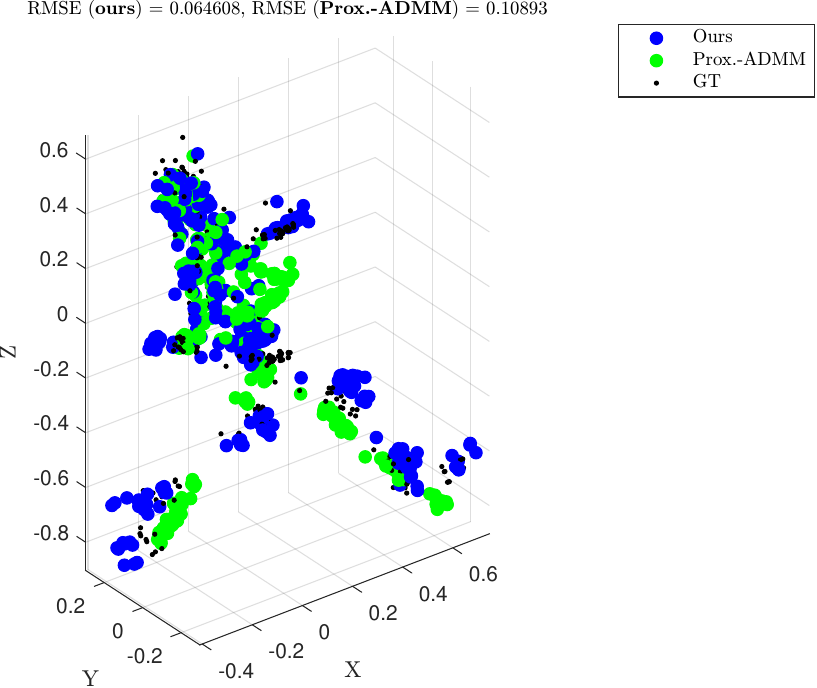}
  \label{fig:q9_8}
\end{subfigure}%

\begin{subfigure}{0.25\textwidth}
  \centering
  \includegraphics[width=0.95\textwidth]{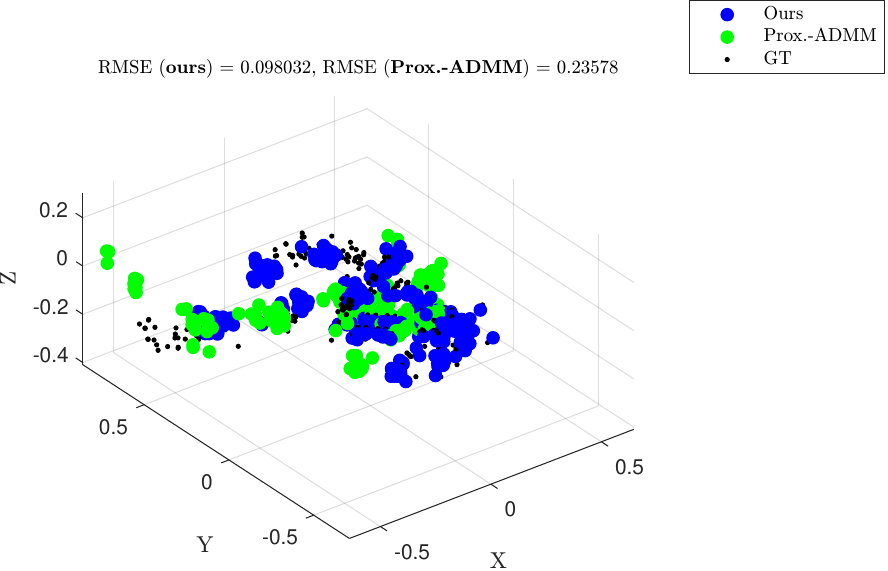}
  \label{fig:q35_1}
\end{subfigure}%
\begin{subfigure}{0.25\textwidth}
  \centering
  \includegraphics[width=0.95\textwidth]{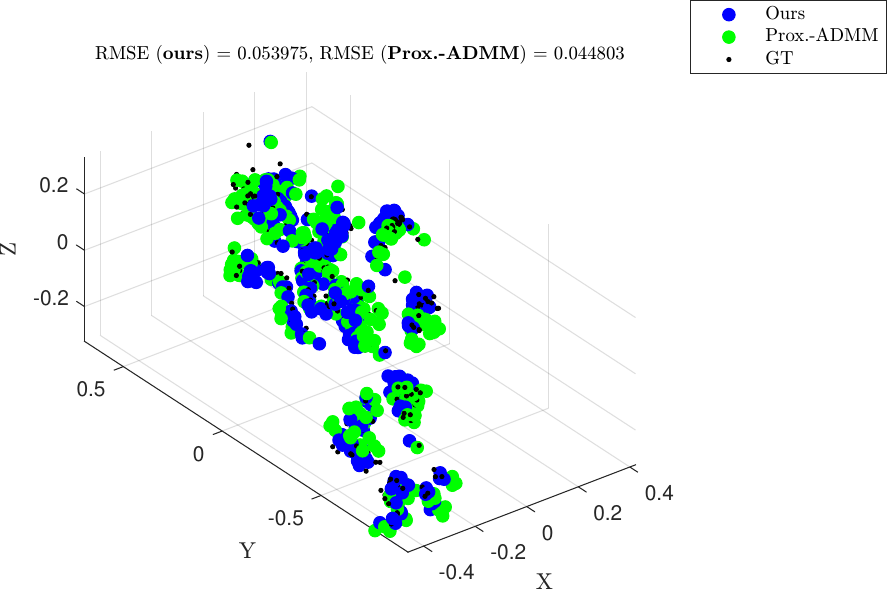}
  \label{fig:q35_3}
\end{subfigure}%
\begin{subfigure}{0.25\textwidth}
  \centering
  \includegraphics[width=0.95\textwidth]{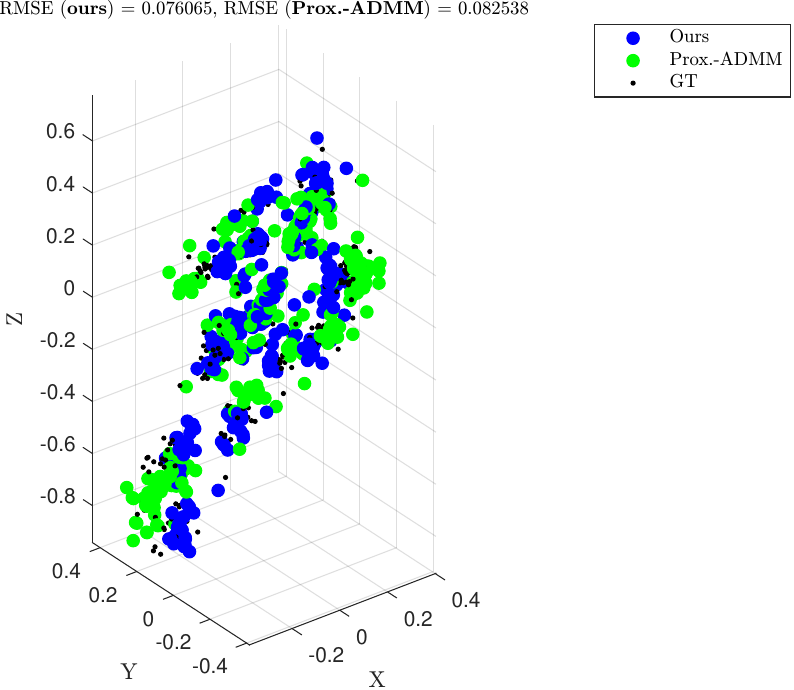}
  \label{fig:q35_5}
\end{subfigure}%
\begin{subfigure}{0.25\textwidth}
  \centering
  \includegraphics[width=0.95\textwidth]{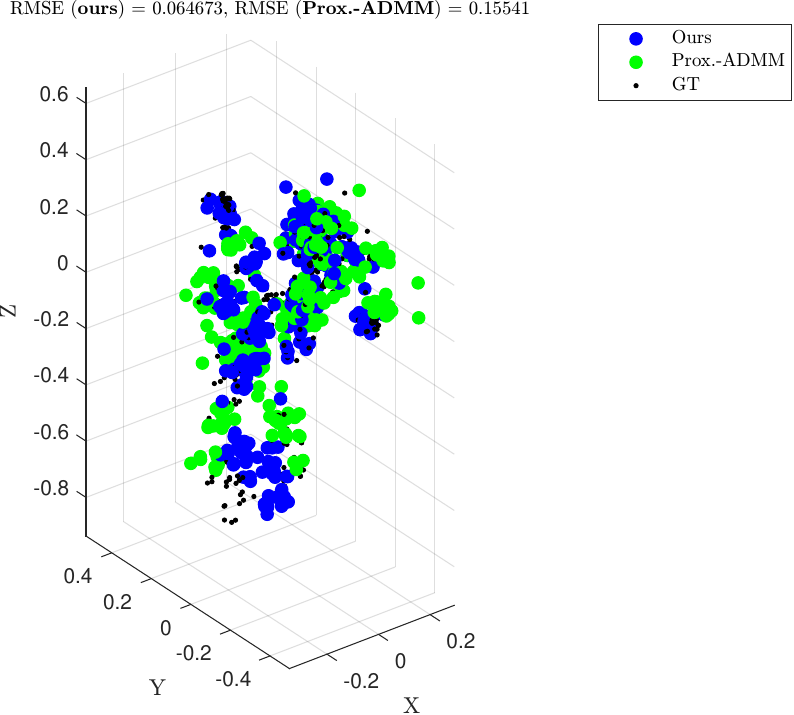}
  \label{fig:q35_7}
\end{subfigure}%

\begin{subfigure}{0.25\textwidth}
  \centering
  \includegraphics[width=0.95\textwidth]{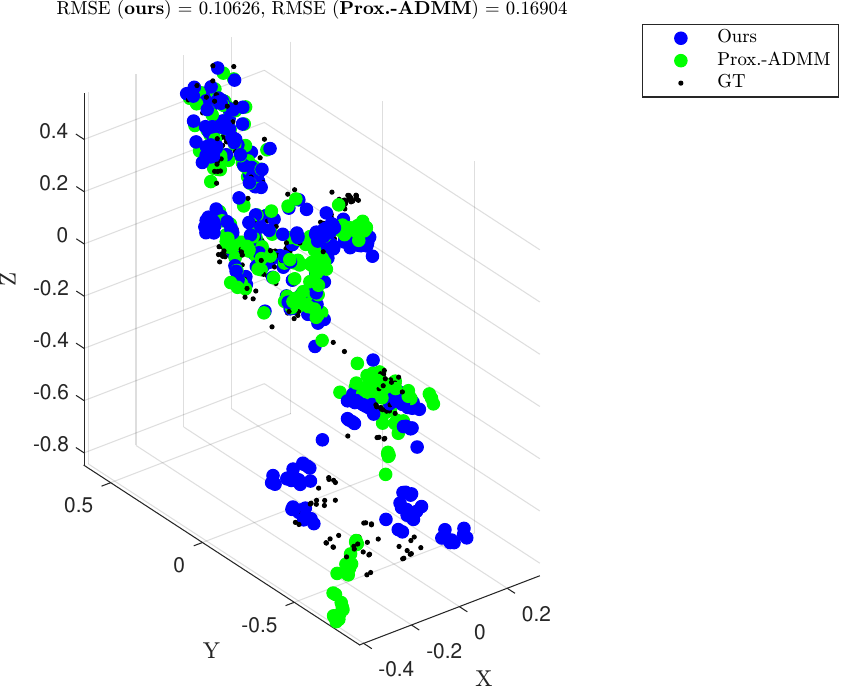}
  \label{fig:q35_2}
\end{subfigure}%
\begin{subfigure}{0.25\textwidth}
  \centering
  \includegraphics[width=0.95\textwidth]{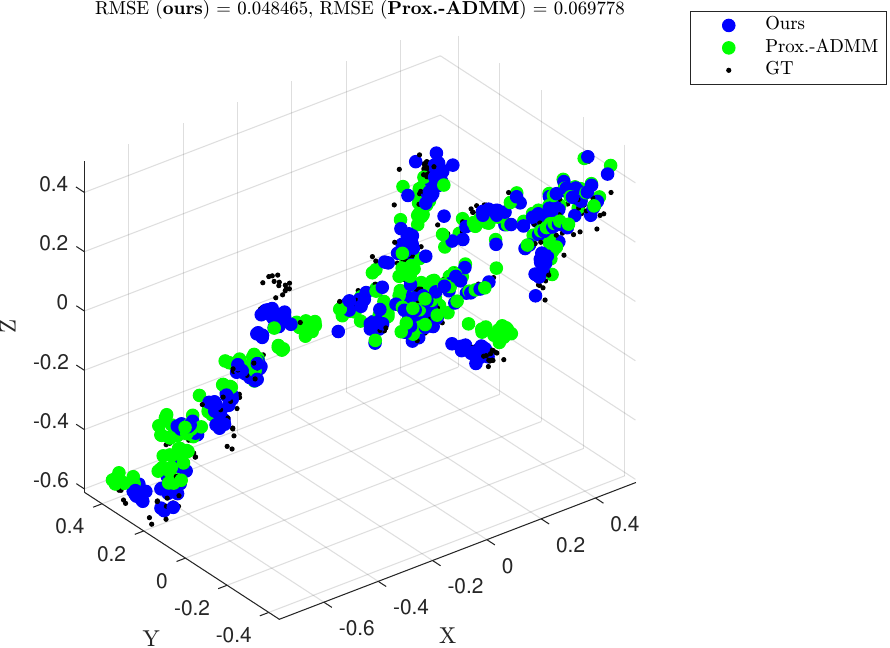}
  \label{fig:q35_4}
\end{subfigure}%
\begin{subfigure}{0.25\textwidth}
  \centering
  \includegraphics[width=0.95\textwidth]{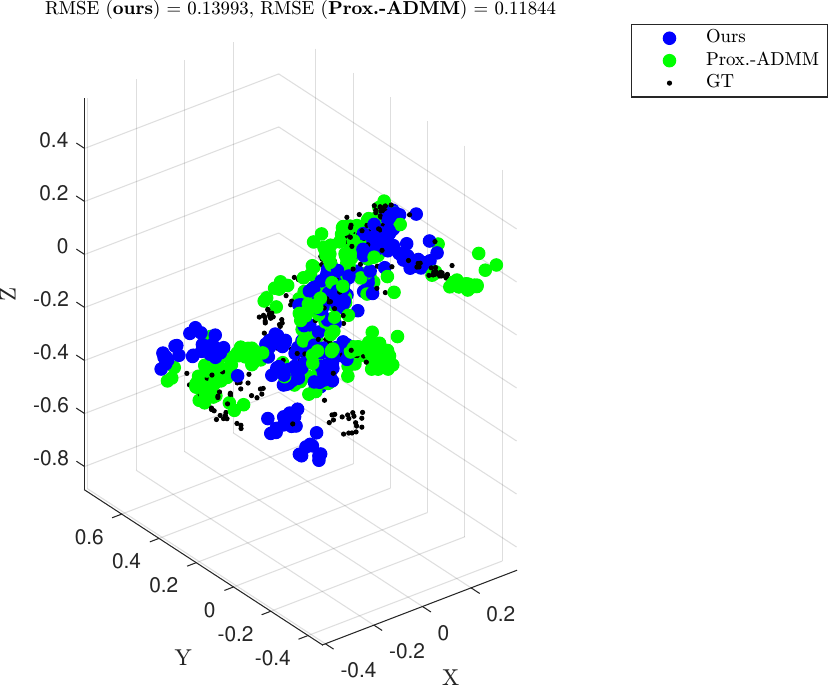}
  \label{fig:q35_6}
\end{subfigure}%
\begin{subfigure}{0.25\textwidth}
  \centering
  \includegraphics[width=0.95\textwidth]{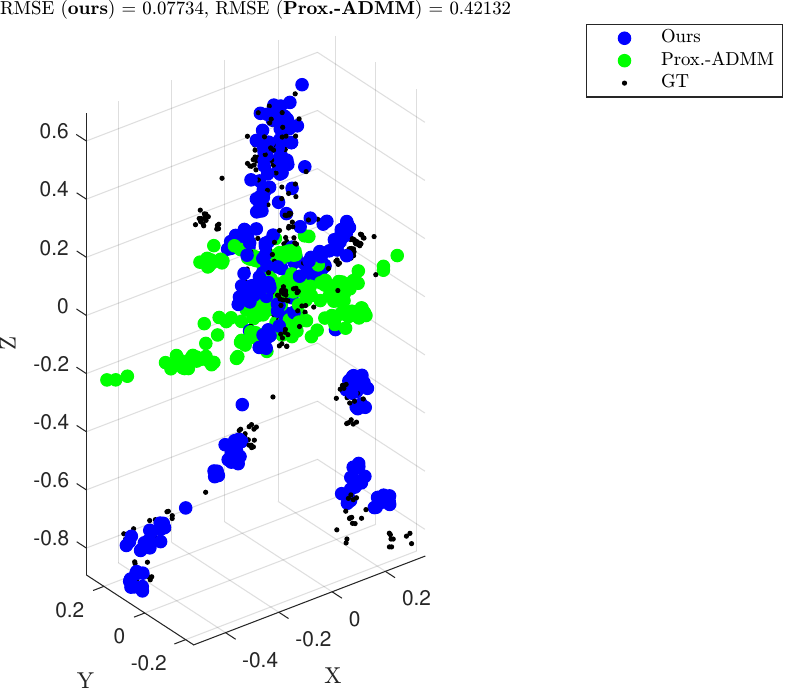}
  \label{fig:q35_8}
\end{subfigure}%
\caption{First two rows are {\tt subject-9}, the next two rows are {\tt subject-35}. The first column is with 10 correspondences, the second column is with 20 correspondences, the third column is with 30 correspondences, and the fourth column is with 40 correspondences.}
\label{art_qual_simpl}
\end{figure}

\begin{figure}
\centering
\begin{subfigure}{0.2\textwidth}
  \centering
  \includegraphics[width=0.95\textwidth]{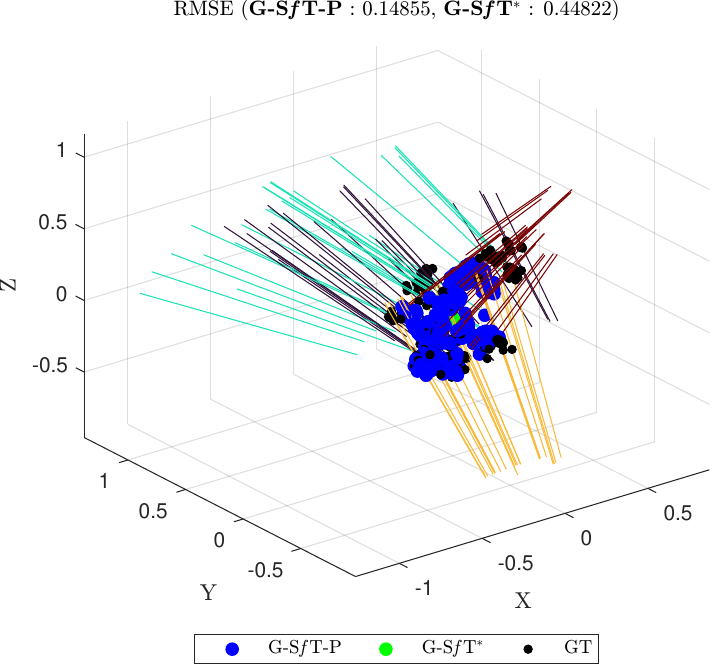}
  \label{fig:q9_mc_1_stp3}
\end{subfigure}%
\begin{subfigure}{0.2\textwidth}
  \centering
  \includegraphics[width=0.95\textwidth]{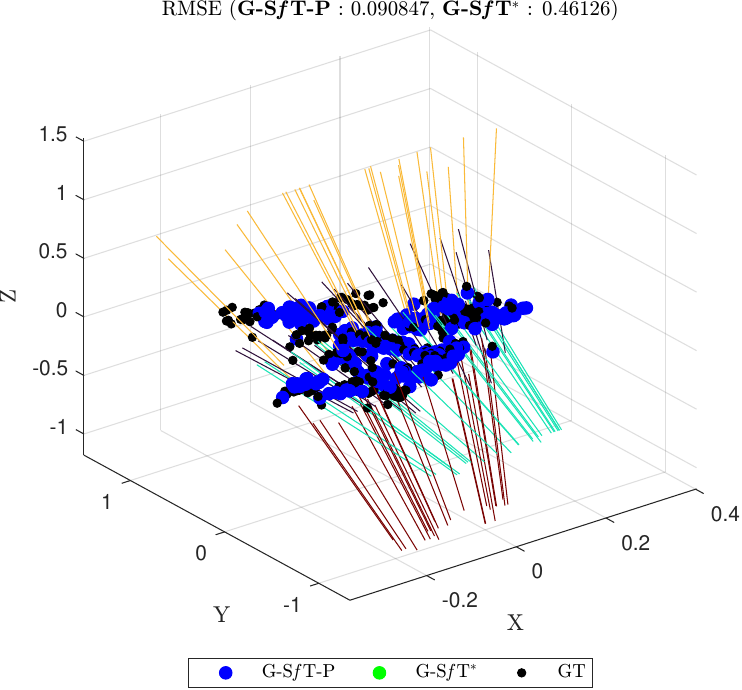}
  \label{fig:q9_mc_2_stp3}
\end{subfigure}%
\begin{subfigure}{0.2\textwidth}
  \centering
  \includegraphics[width=0.95\textwidth]{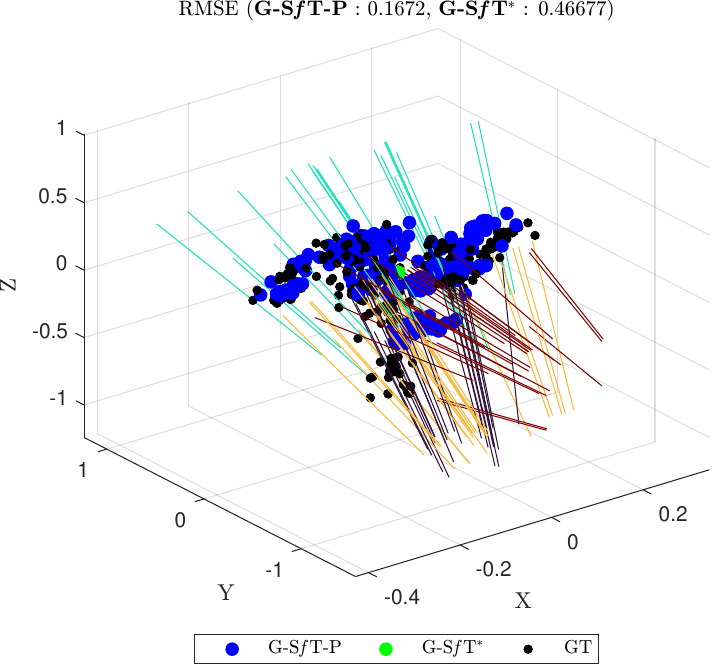}
  \label{fig:q9_mc_3_stp3}
\end{subfigure}%
\begin{subfigure}{0.2\textwidth}
  \centering
  \includegraphics[width=0.95\textwidth]{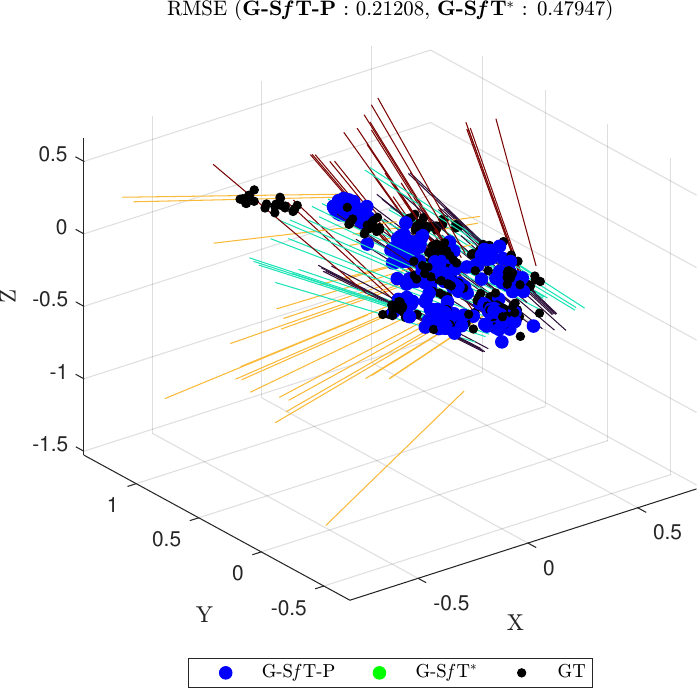}
  \label{fig:q9_mc_4_stp3}
\end{subfigure}%
\begin{subfigure}{0.2\textwidth}
  \centering
  \includegraphics[width=0.95\textwidth]{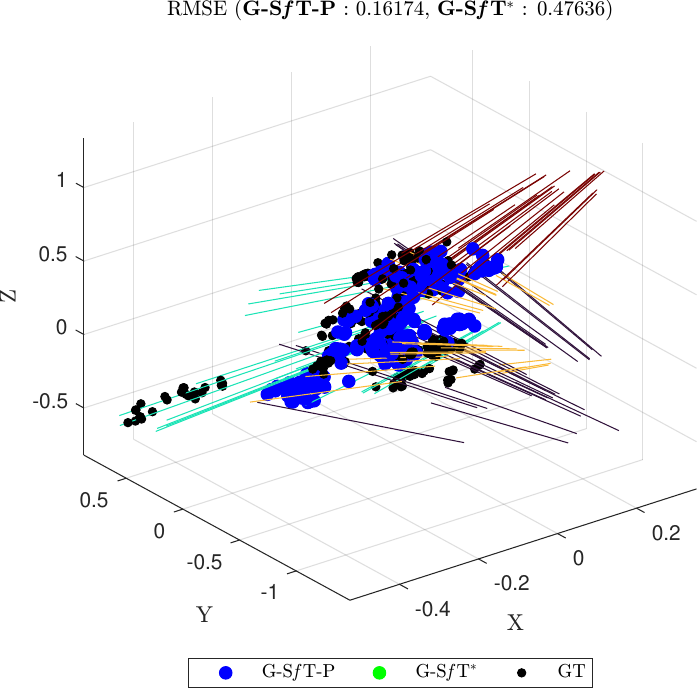}
  \label{fig:q9_mc_5_stp3}
\end{subfigure}%

\begin{subfigure}{0.2\textwidth}
  \centering
  \includegraphics[width=0.95\textwidth]{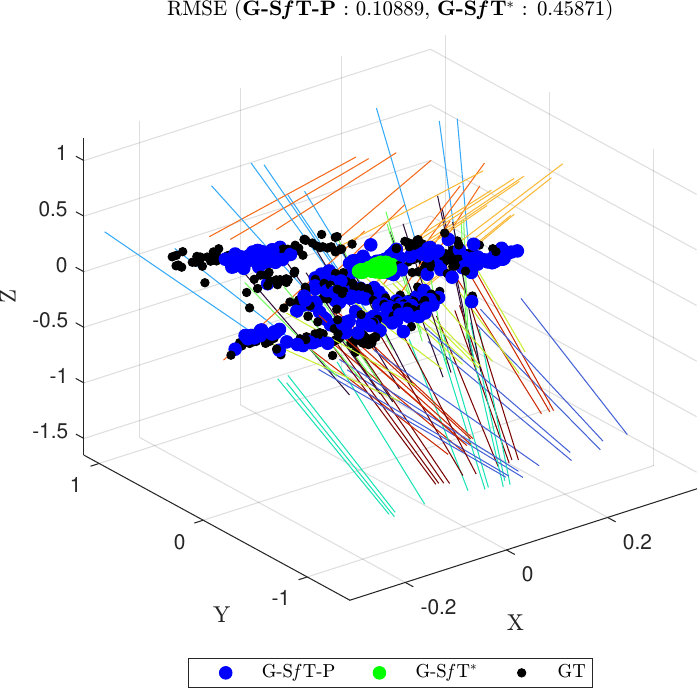}
  \label{fig:q9_mc_1_stp4}
\end{subfigure}%
\begin{subfigure}{0.2\textwidth}
  \centering
  \includegraphics[width=0.95\textwidth]{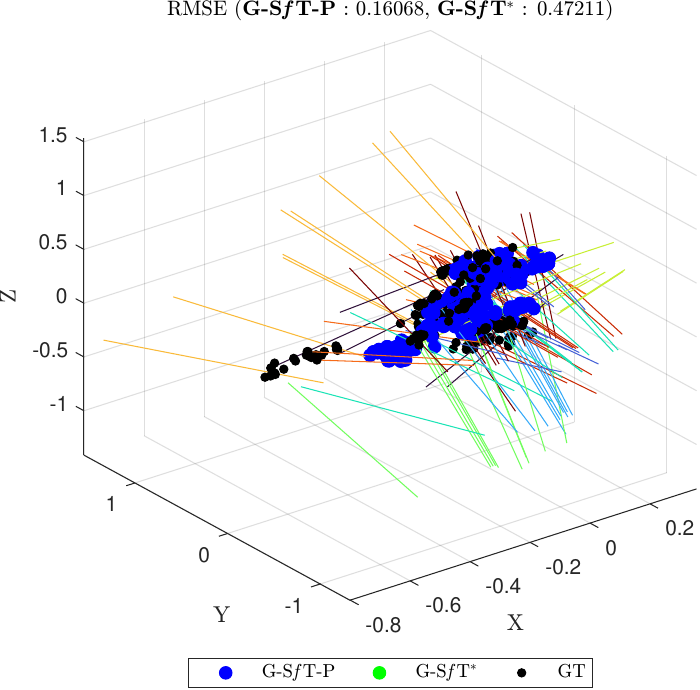}
  \label{fig:q9_mc_2_stp4}
\end{subfigure}%
\begin{subfigure}{0.2\textwidth}
  \centering
  \includegraphics[width=0.95\textwidth]{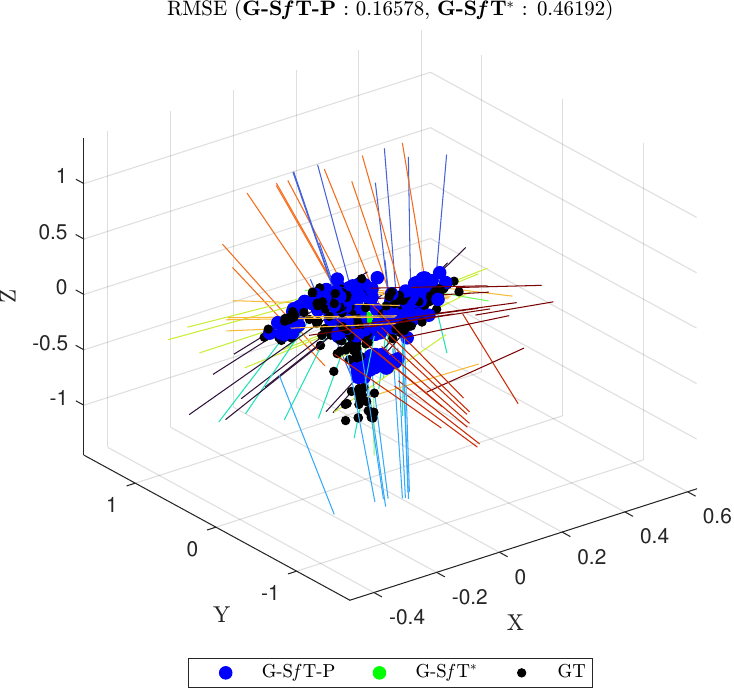}
  \label{fig:q9_mc_3_stp4}
\end{subfigure}%
\begin{subfigure}{0.2\textwidth}
  \centering
  \includegraphics[width=0.95\textwidth]{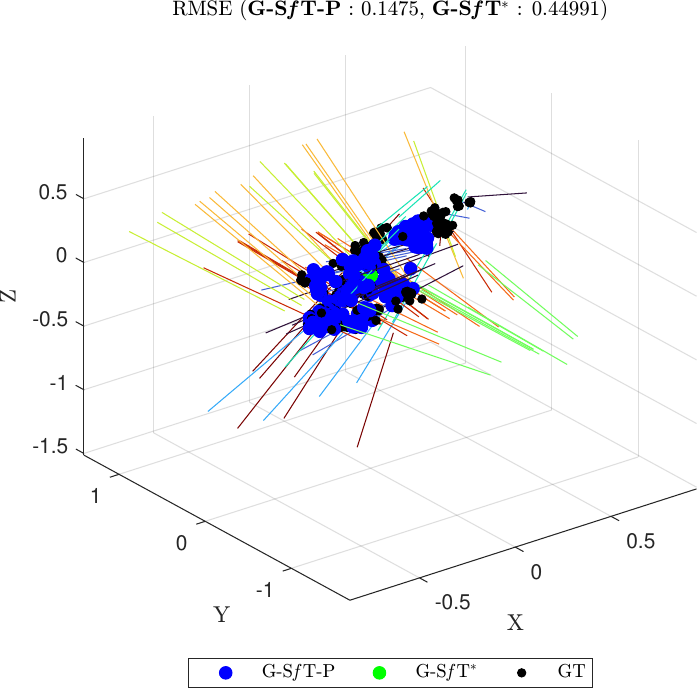}
  \label{fig:q9_mc_4_stp4}
\end{subfigure}%
\begin{subfigure}{0.2\textwidth}
  \centering
  \includegraphics[width=0.95\textwidth]{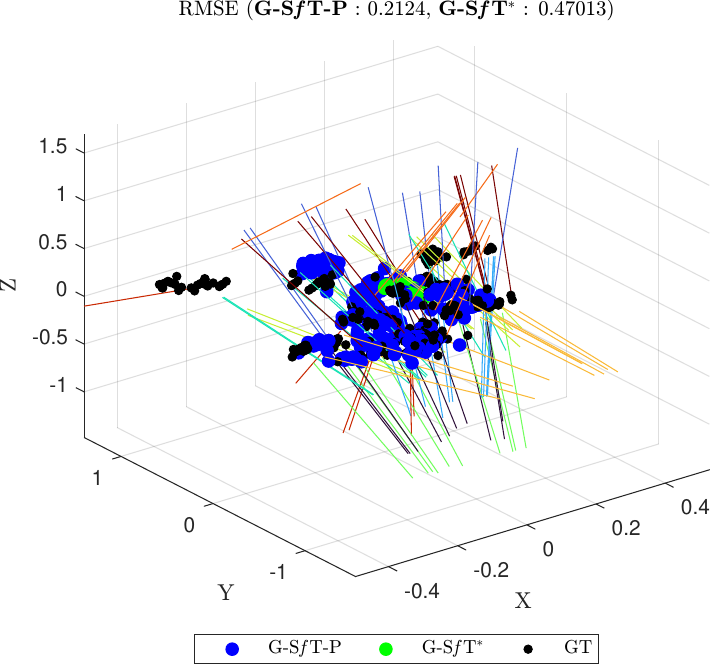}
  \label{fig:q9_mc_5_stp4}
\end{subfigure}%

\begin{subfigure}{0.2\textwidth}
  \centering
  \includegraphics[width=0.95\textwidth]{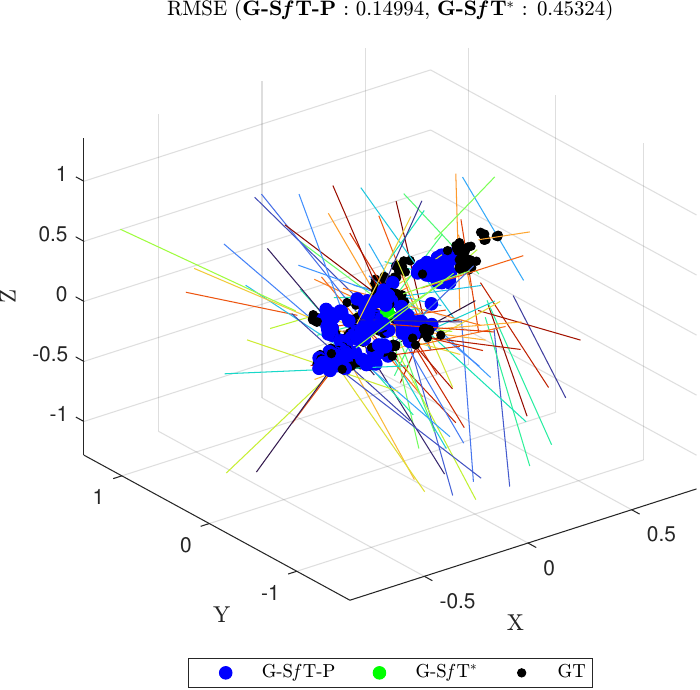}
  \label{fig:q9_mc_1_stp5}
\end{subfigure}%
\begin{subfigure}{0.2\textwidth}
  \centering
  \includegraphics[width=0.95\textwidth]{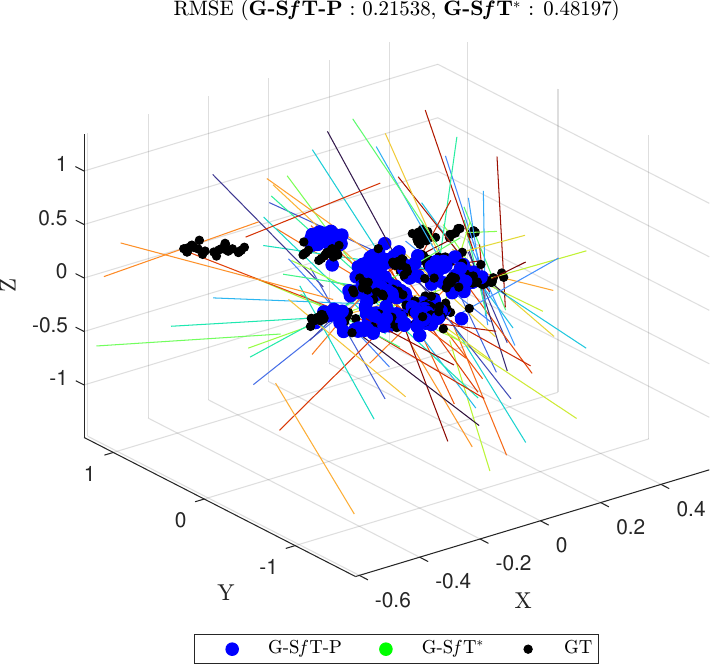}
  \label{fig:q9_mc_2_stp5}
\end{subfigure}%
\begin{subfigure}{0.2\textwidth}
  \centering
  \includegraphics[width=0.95\textwidth]{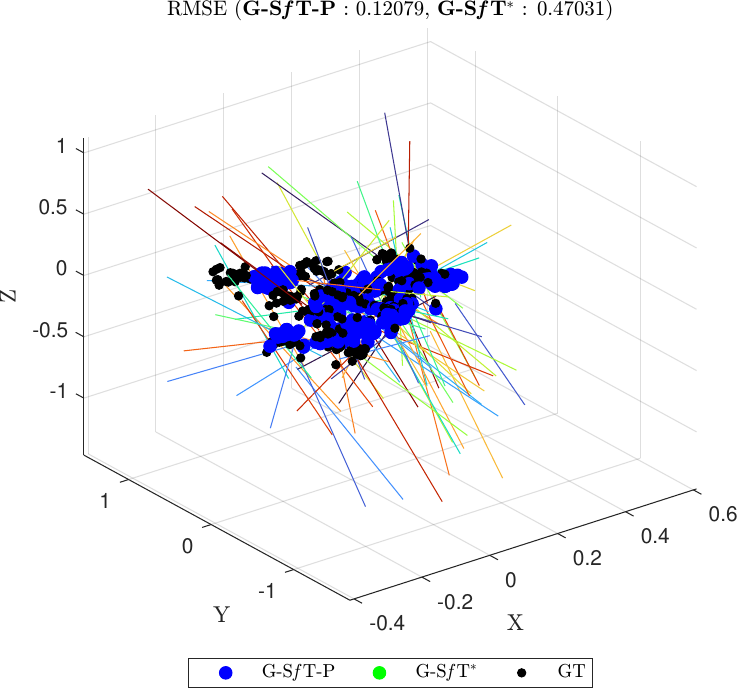}
  \label{fig:q9_mc_3_stp5}
\end{subfigure}%
\begin{subfigure}{0.2\textwidth}
  \centering
  \includegraphics[width=0.95\textwidth]{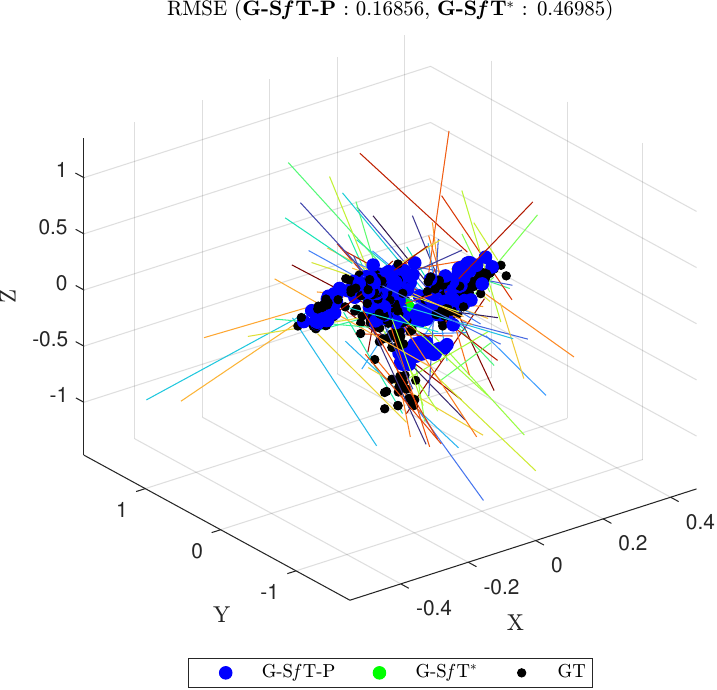}
  \label{fig:q9_mc_4_stp5}
\end{subfigure}%
\begin{subfigure}{0.2\textwidth}
  \centering
  \includegraphics[width=0.95\textwidth]{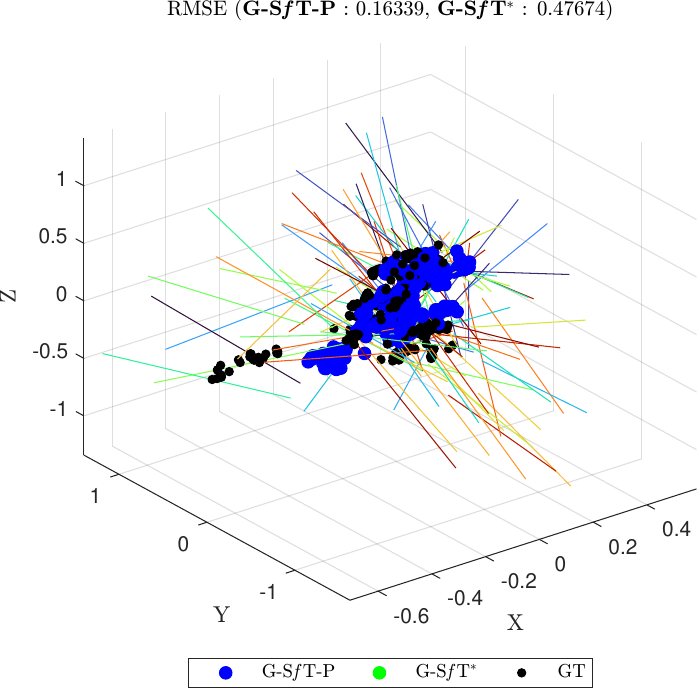}
  \label{fig:q9_mc_5_stp5}
\end{subfigure}%
\caption{Qualitative comparison of results from repeated \y{ns} and \y{nsc} on the generalised camera setup with {\tt subject-9}. The three rows show the configurations 3, 4, and 5, as described in \cref{sec_exp_gsftp}. The five columns are five different, randomly sampled experiments. The black keypoints are 3D \y{gt}, the blue keypoints are reconstruction using \y{nsc} and the green keypoints are one of the reconstruction of \y{ns}, this randomly sampled reconstruction is labelled {\y{ns}$^{\ast}$}. The lines incident on the reconstruction are sightlines, the lines sharing a common origin are of the same colour. The first-row/configuration-3 therefore has four groups of colours, the second-row/configuration-4 has ten groups of colours, the third-row/configuration-5 has one-hundred groups of colours on their respective sight-lines}
\label{fig_subj9_mc_qual}
\end{figure}

\begin{figure}
\centering
\begin{subfigure}{0.2\textwidth}
  \centering
  \includegraphics[width=0.95\textwidth]{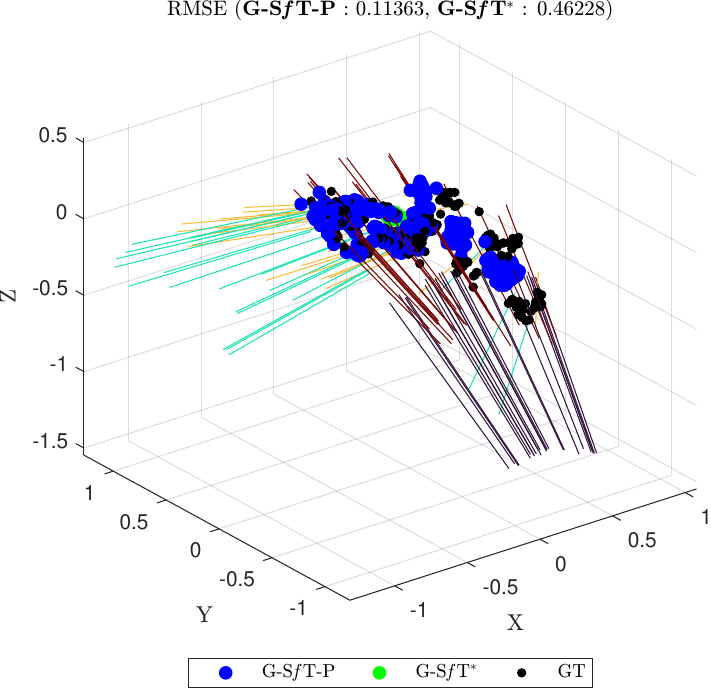}
  \label{fig:q35_mc_1_stp3}
\end{subfigure}%
\begin{subfigure}{0.2\textwidth}
  \centering
  \includegraphics[width=0.95\textwidth]{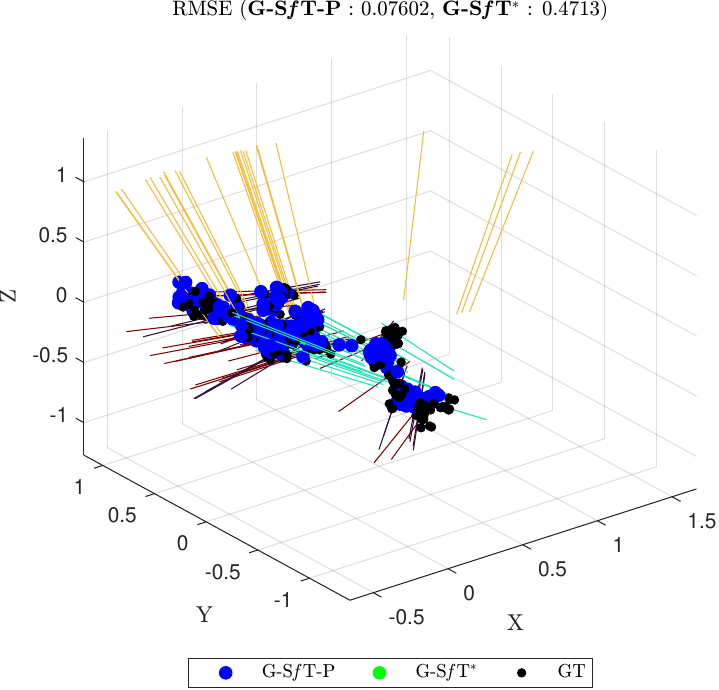}
  \label{fig:q35_mc_2_stp3}
\end{subfigure}%
\begin{subfigure}{0.2\textwidth}
  \centering
  \includegraphics[width=0.95\textwidth]{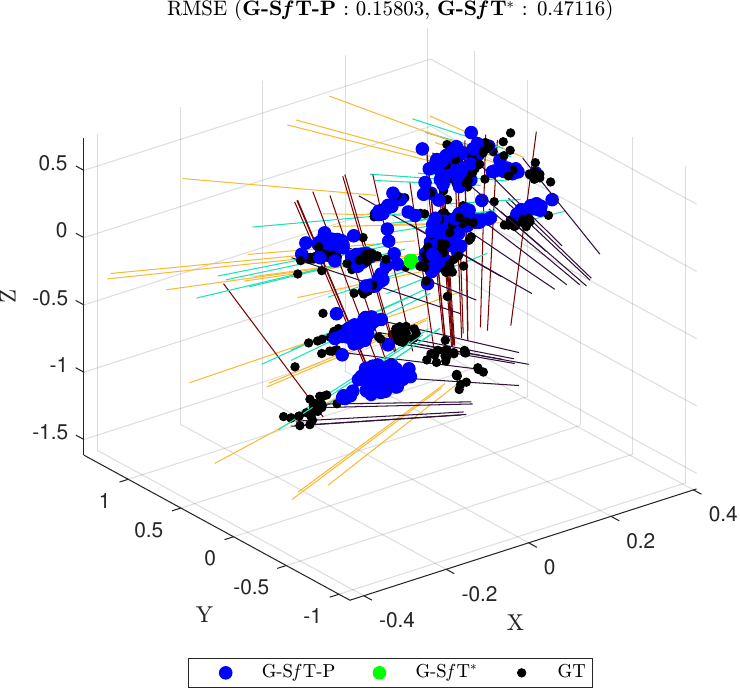}
  \label{fig:q35_mc_3_stp3}
\end{subfigure}%
\begin{subfigure}{0.2\textwidth}
  \centering
  \includegraphics[width=0.95\textwidth]{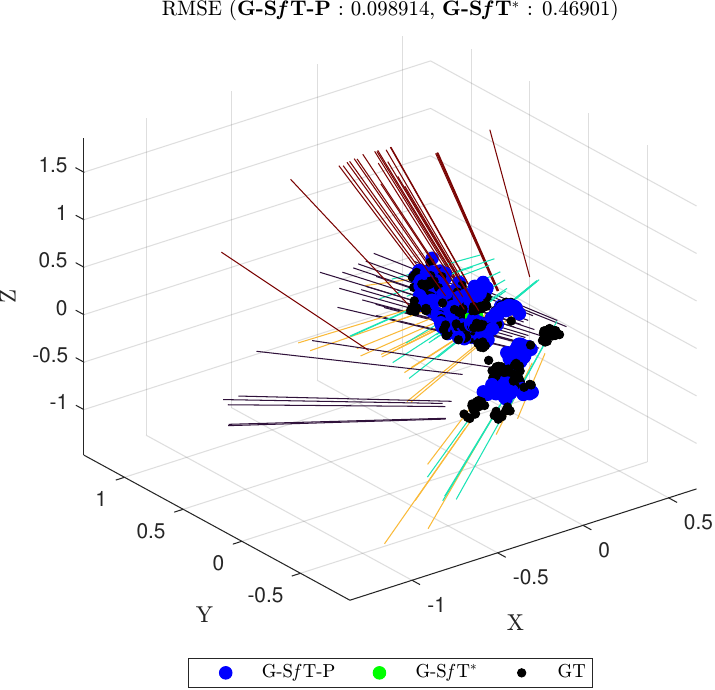}
  \label{fig:q35_mc_4_stp3}
\end{subfigure}%
\begin{subfigure}{0.2\textwidth}
  \centering
  \includegraphics[width=0.95\textwidth]{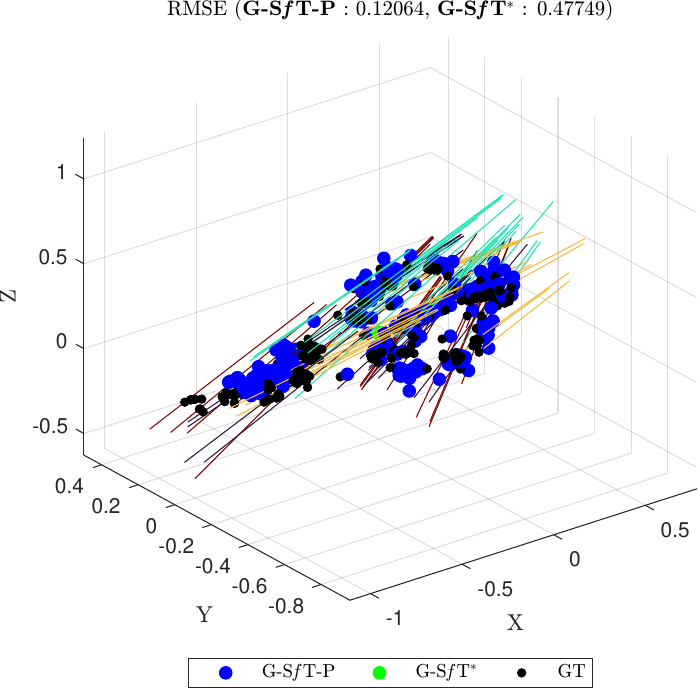}
  \label{fig:q35_mc_5_stp3}
\end{subfigure}%

\begin{subfigure}{0.2\textwidth}
  \centering
  \includegraphics[width=0.95\textwidth]{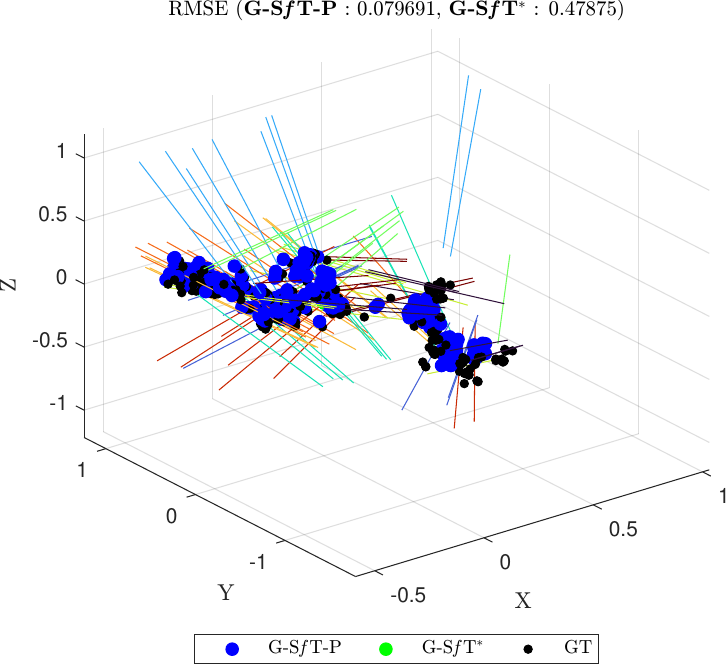}
  \label{fig:q35_mc_1_stp4}
\end{subfigure}%
\begin{subfigure}{0.2\textwidth}
  \centering
  \includegraphics[width=0.95\textwidth]{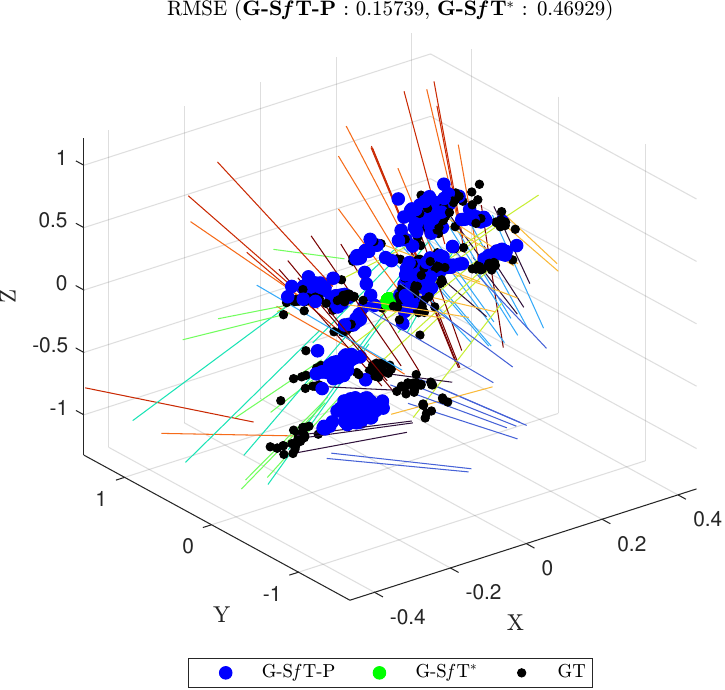}
  \label{fig:q35_mc_2_stp4}
\end{subfigure}%
\begin{subfigure}{0.2\textwidth}
  \centering
  \includegraphics[width=0.95\textwidth]{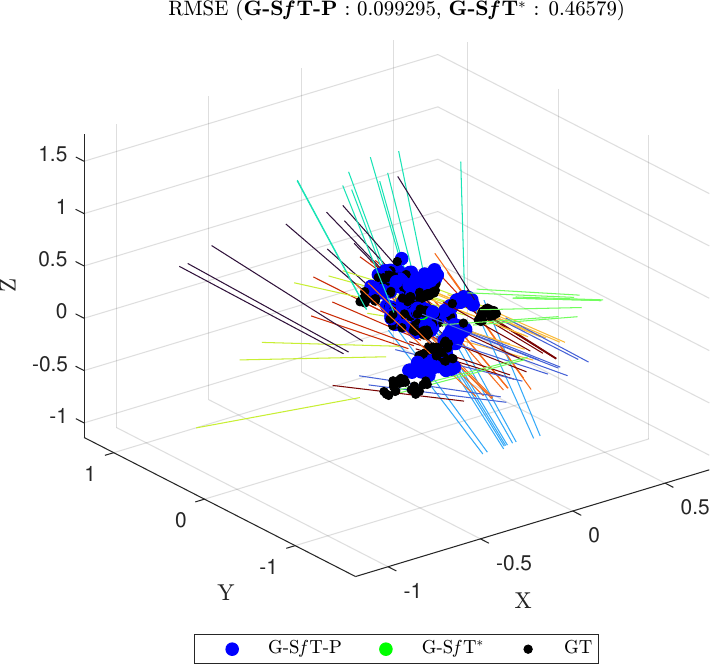}
  \label{fig:q35_mc_3_stp4}
\end{subfigure}%
\begin{subfigure}{0.2\textwidth}
  \centering
  \includegraphics[width=0.95\textwidth]{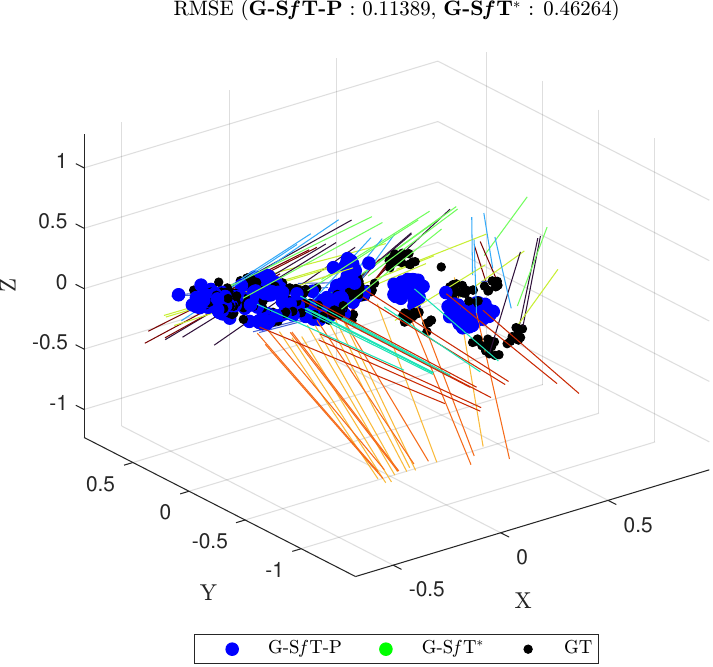}
  \label{fig:q35_mc_4_stp4}
\end{subfigure}%
\begin{subfigure}{0.2\textwidth}
  \centering
  \includegraphics[width=0.95\textwidth]{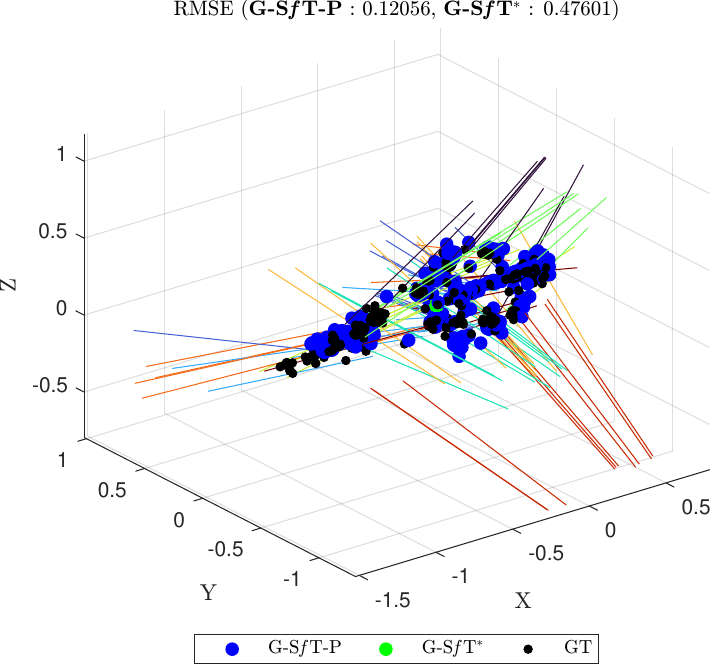}
  \label{fig:q35_mc_5_stp4}
\end{subfigure}%

\begin{subfigure}{0.2\textwidth}
  \centering
  \includegraphics[width=0.95\textwidth]{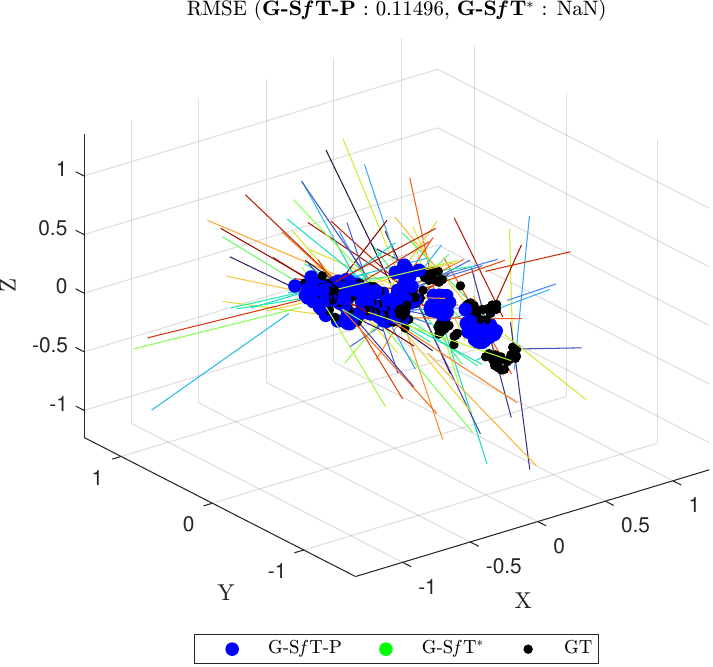}
  \label{fig:q35_mc_1_stp5}
\end{subfigure}%
\begin{subfigure}{0.2\textwidth}
  \centering
  \includegraphics[width=0.95\textwidth]{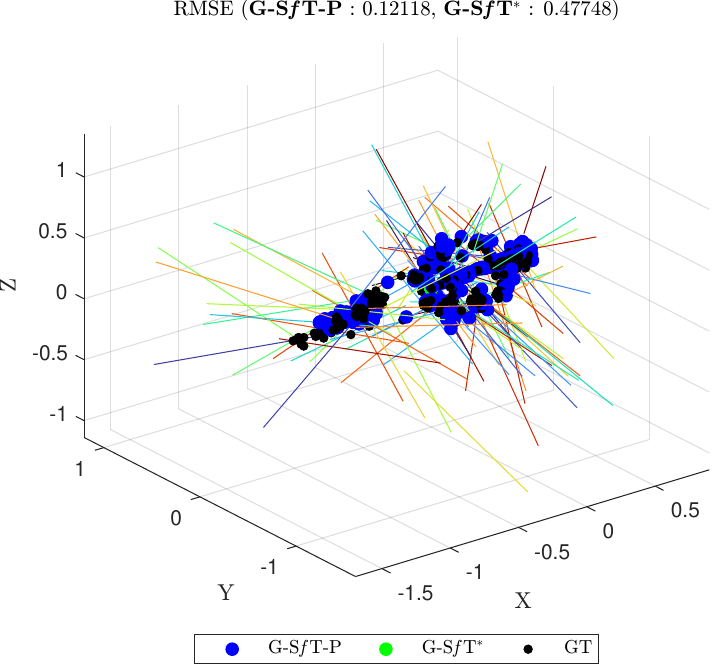}
  \label{fig:q35_mc_2_stp5}
\end{subfigure}%
\begin{subfigure}{0.2\textwidth}
  \centering
  \includegraphics[width=0.95\textwidth]{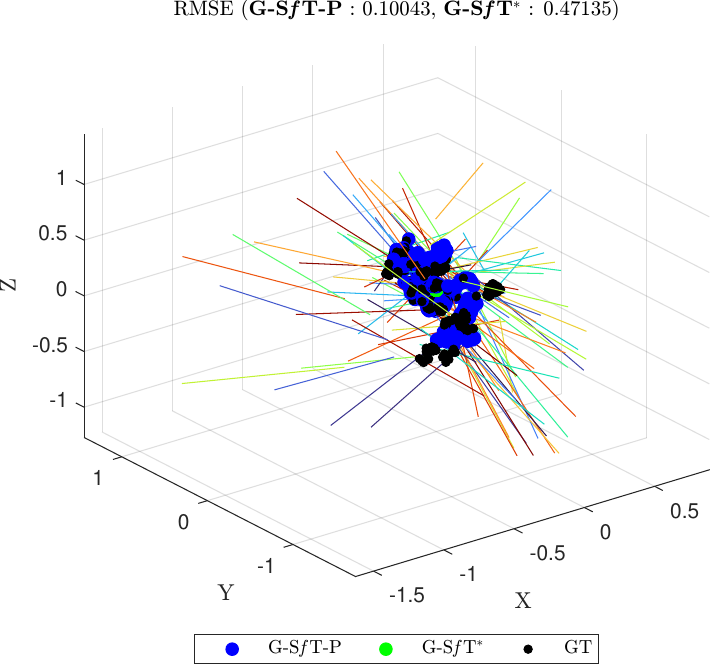}
  \label{fig:q35_mc_3_stp5}
\end{subfigure}%
\begin{subfigure}{0.2\textwidth}
  \centering
  \includegraphics[width=0.95\textwidth]{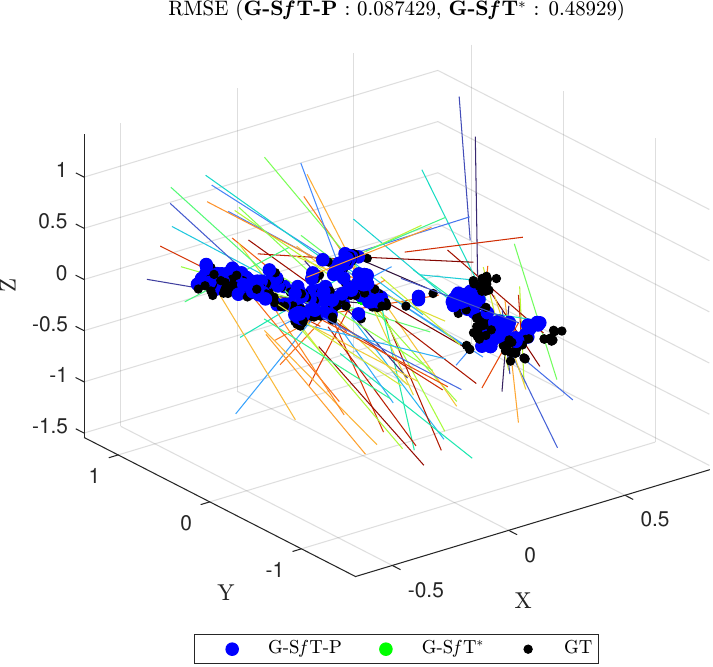}
  \label{fig:q35_mc_4_stp5}
\end{subfigure}%
\begin{subfigure}{0.2\textwidth}
  \centering
  \includegraphics[width=0.95\textwidth]{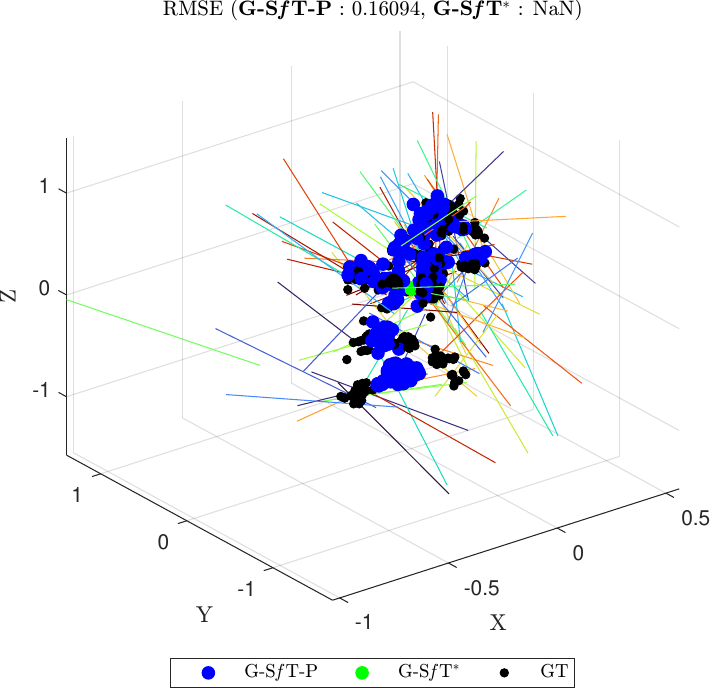}
  \label{fig:q35_mc_5_stp5}
\end{subfigure}%
\caption{Qualitative comparison of results from repeated \y{ns} and \y{nsc} on the generalised camera setup with {\tt subject-35}. The three rows show the configurations 3, 4, and 5, as described in \cref{sec_exp_gsftp}. The five columns are five different, randomly sampled experiments. The black keypoints are 3D \y{gt}, the blue keypoints are reconstruction using \y{nsc} and the green keypoints are one of the reconstruction of \y{ns}, this randomly sampled reconstruction is labelled {\y{ns}$^{\ast}$}. The lines incident on the reconstruction are sightlines, the lines sharing a common origin are of the same colour. The first-row/configuration-3 therefore has four groups of colours, the second-row/configuration-4 has ten groups of colours, the third-row/configuration-5 has one-hundred groups of colours on their respective sight-lines}
\label{fig_subj35_mc_qual}
\end{figure}

\begin{figure}
\centering
\begin{subfigure}{0.5\textwidth}
  \centering
  \includegraphics[width=0.95\textwidth]{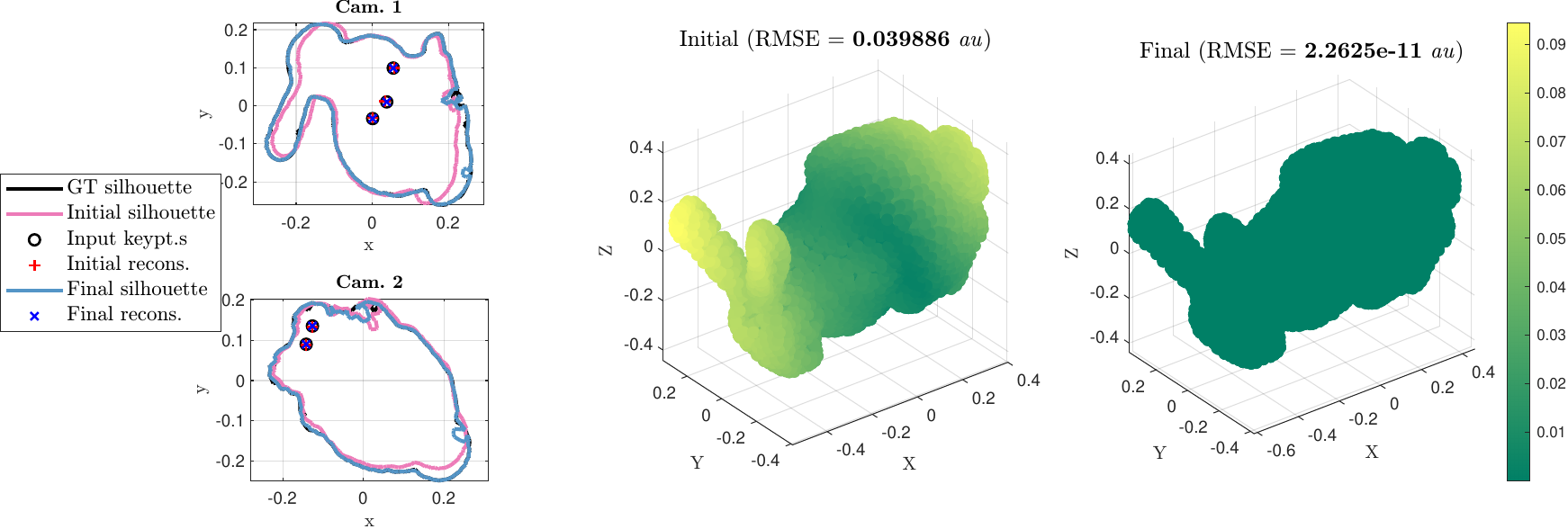}
  \label{fig:b_2_1}
\end{subfigure}%
\begin{subfigure}{0.5\textwidth}
  \centering
  \includegraphics[width=0.95\textwidth]{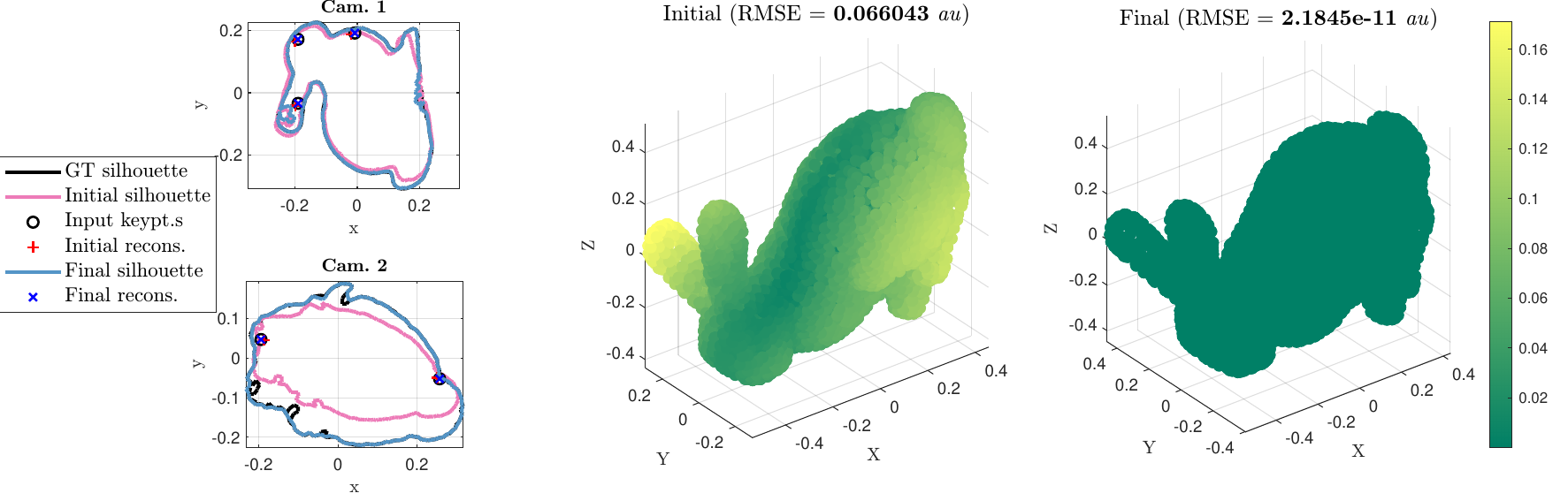}
  \label{fig:b_2_2}
\end{subfigure}%

\begin{subfigure}{0.5\textwidth}
  \centering
  \includegraphics[width=0.95\textwidth]{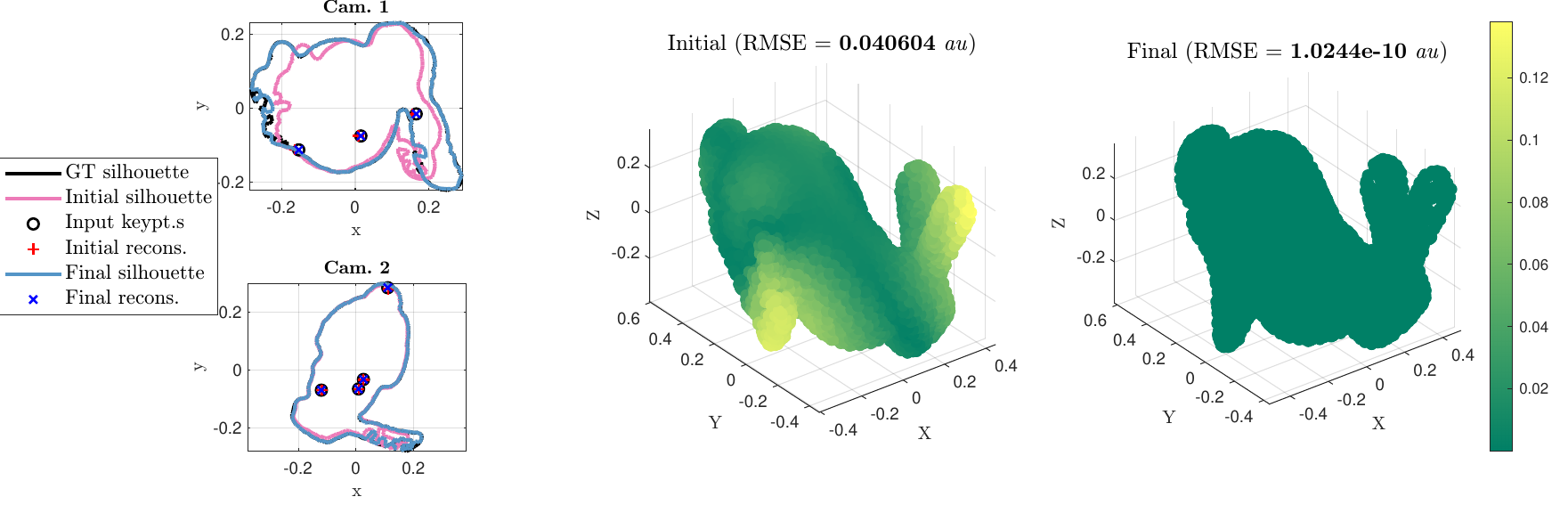}
  \label{fig:b_4_1}
\end{subfigure}%
\begin{subfigure}{0.5\textwidth}
  \centering
  \includegraphics[width=0.95\textwidth]{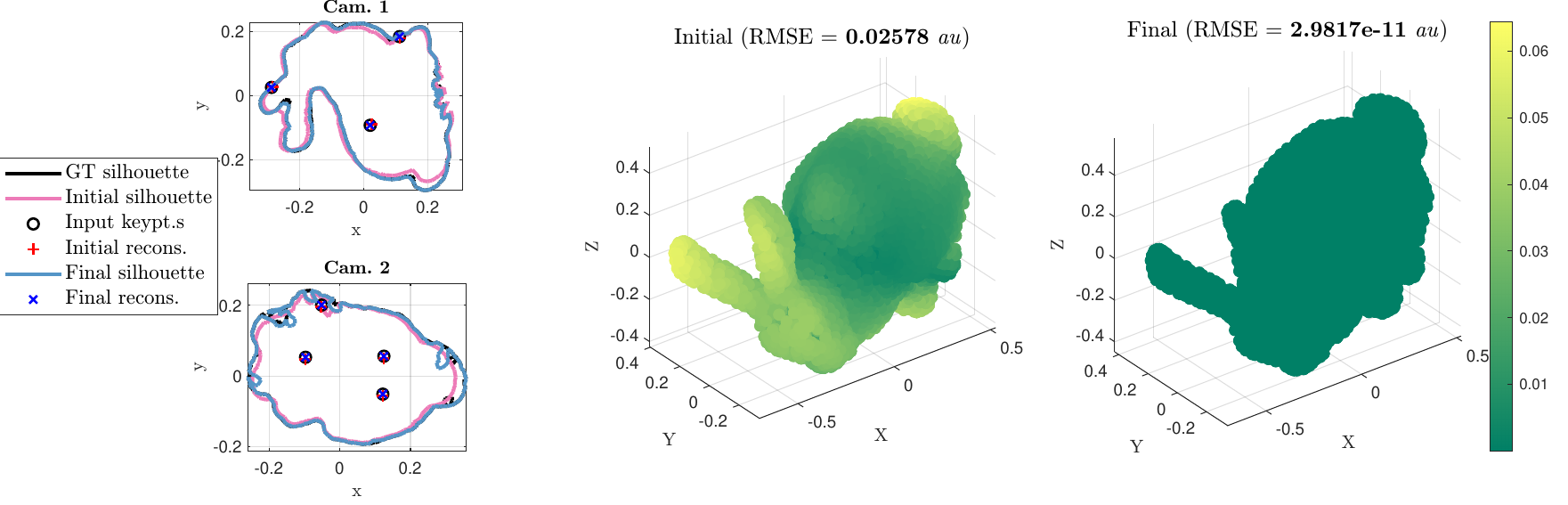}
  \label{fig:b_4_2}
\end{subfigure}%

\begin{subfigure}{0.5\textwidth}
  \centering
  \includegraphics[width=0.95\textwidth]{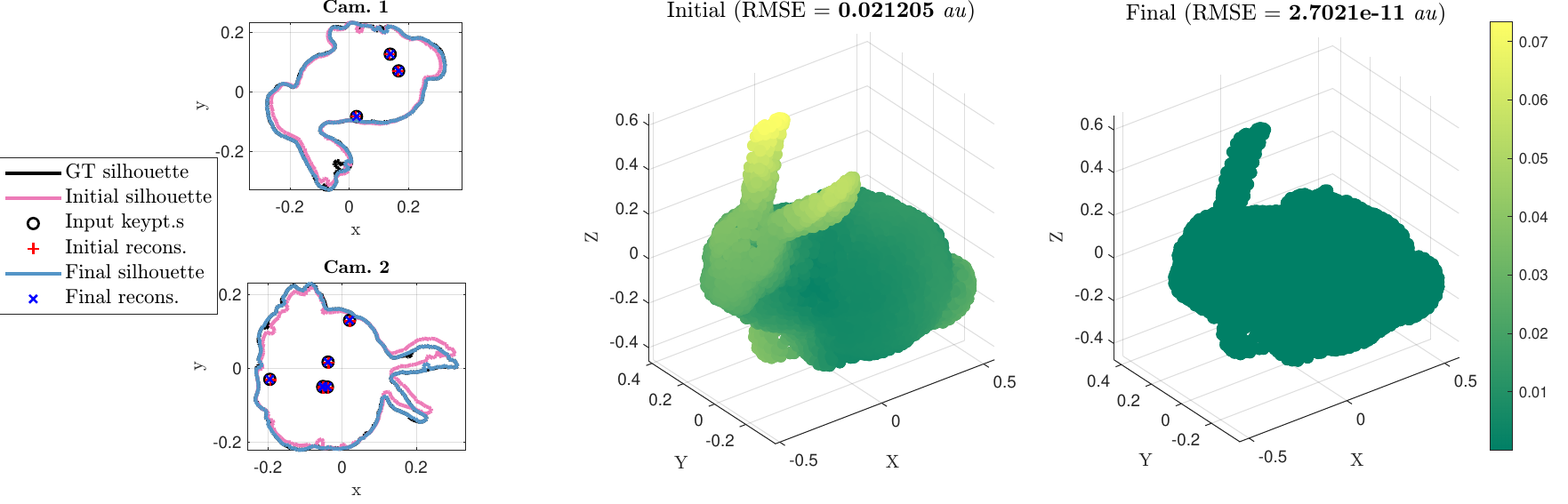}
  \label{fig:b_6_1}
\end{subfigure}%
\begin{subfigure}{0.5\textwidth}
  \centering
  \includegraphics[width=0.95\textwidth]{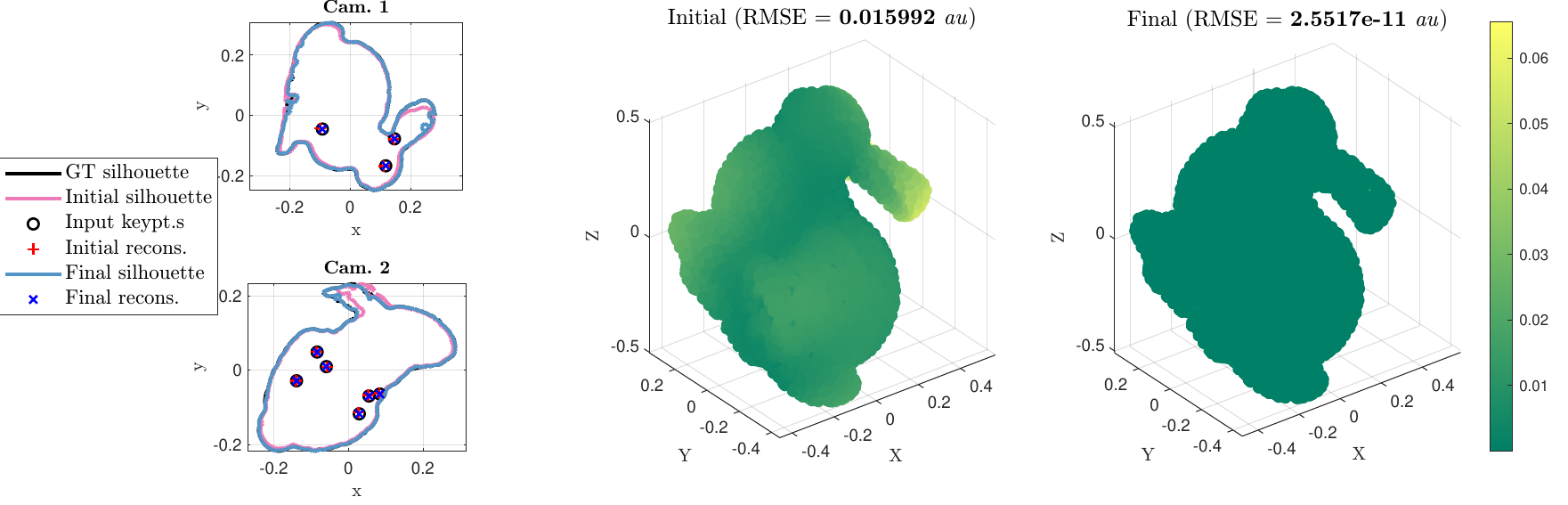}
  \label{fig:b_6_2}
\end{subfigure}%

\caption{Synthetic qualitative {\tt Stanford bunny}}
\label{fig:sil_synth_qual_bunny}
\end{figure}

\begin{figure}
    \centering
\begin{subfigure}{0.5\textwidth}
  \centering
  \includegraphics[width=0.95\textwidth]{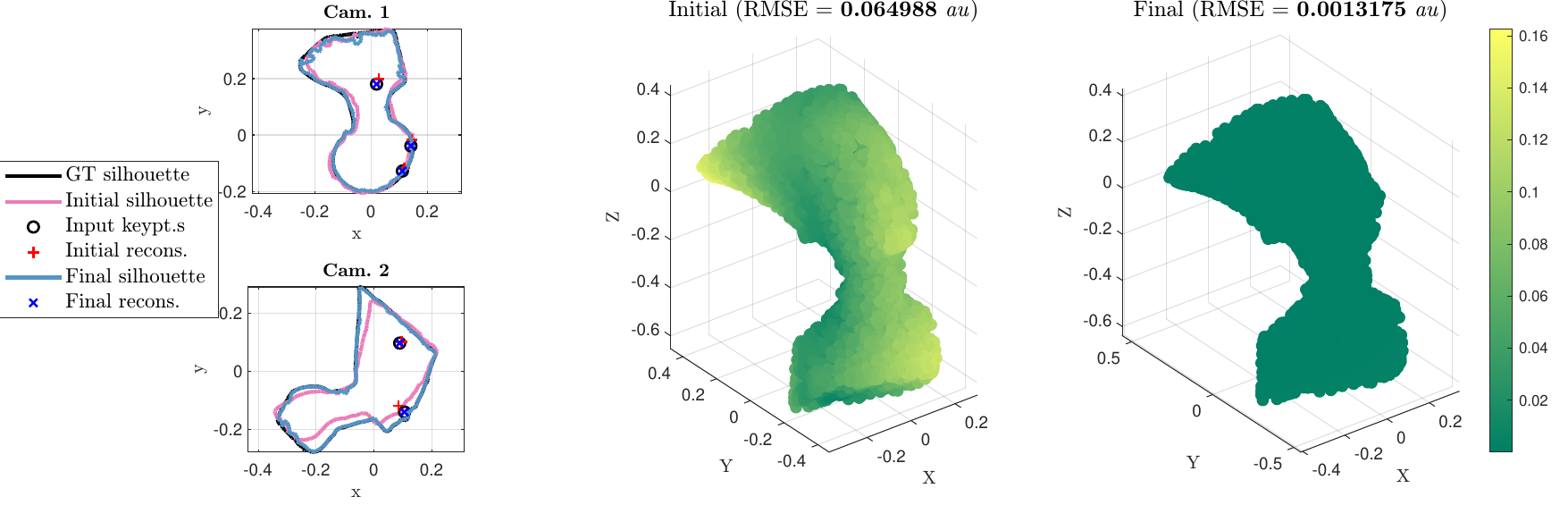}
  \label{fig:n_2_1}
\end{subfigure}%
\begin{subfigure}{0.5\textwidth}
  \centering
  \includegraphics[width=0.95\textwidth]{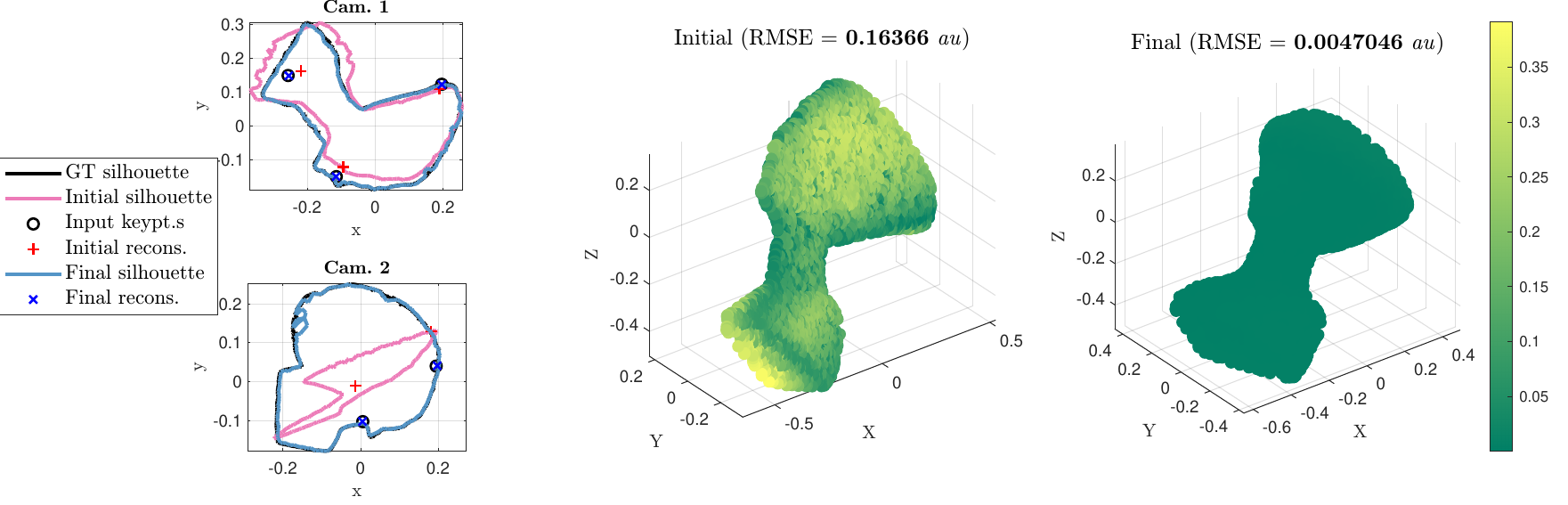}
  \label{fig:n_2_2}
\end{subfigure}%

\begin{subfigure}{0.5\textwidth}
  \centering
  \includegraphics[width=0.95\textwidth]{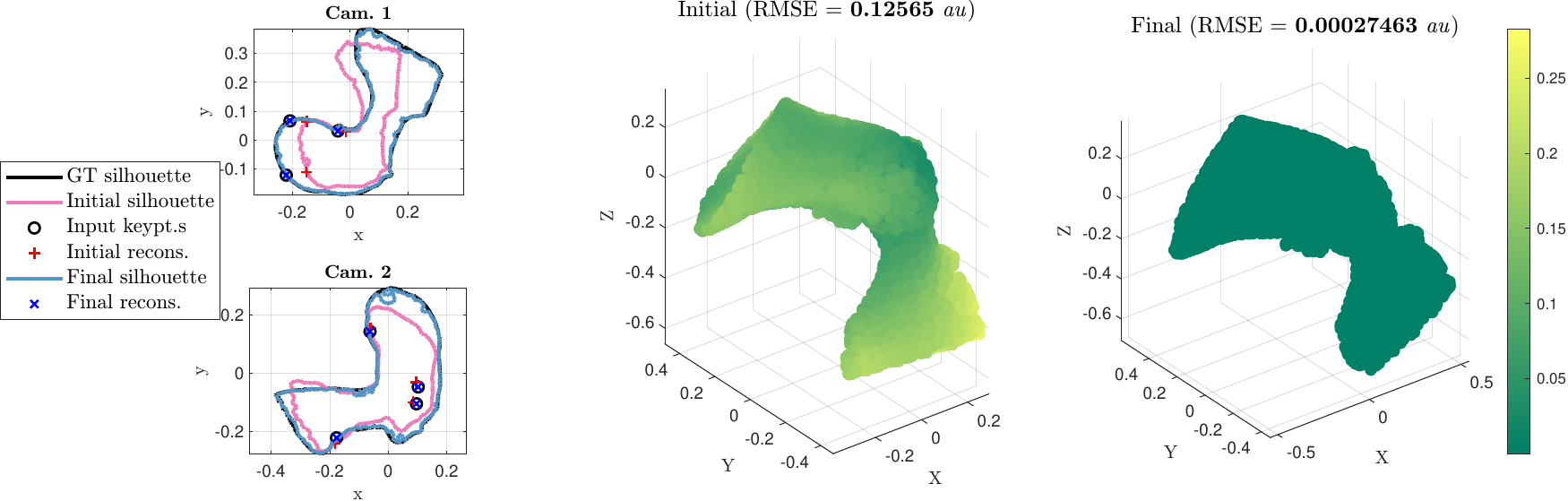}
  \label{fig:n_4_1}
\end{subfigure}%
\begin{subfigure}{0.5\textwidth}
  \centering
  \includegraphics[width=0.95\textwidth]{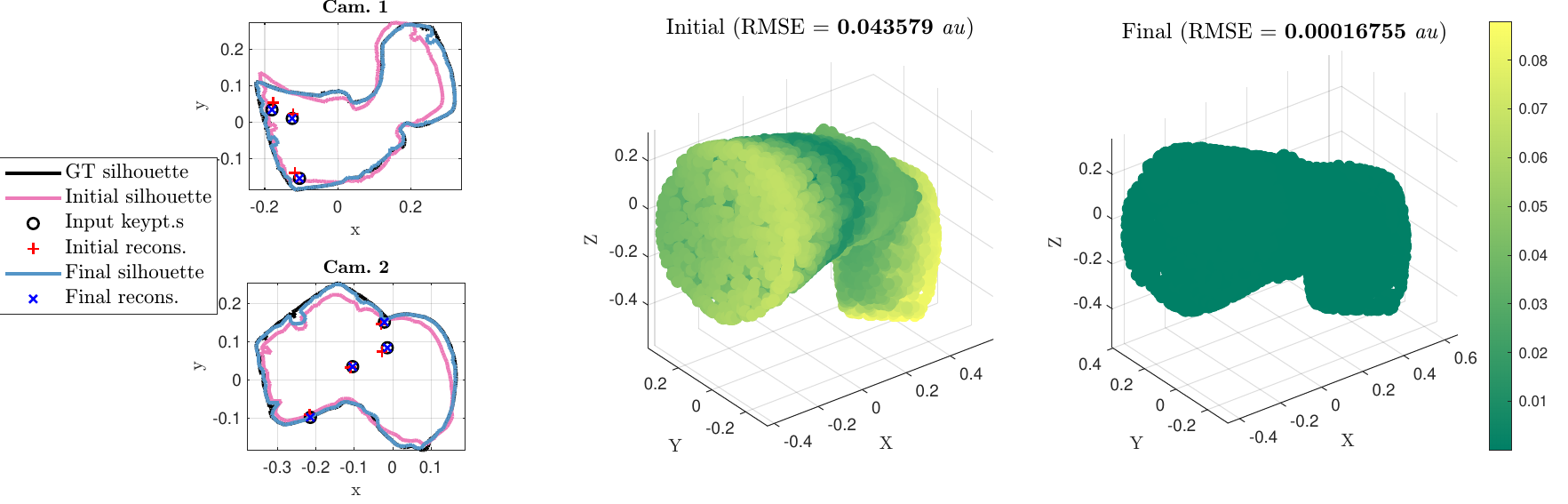}
  \label{fig:n_4_2}
\end{subfigure}%

\begin{subfigure}{0.5\textwidth}
  \centering
  \includegraphics[width=0.95\textwidth]{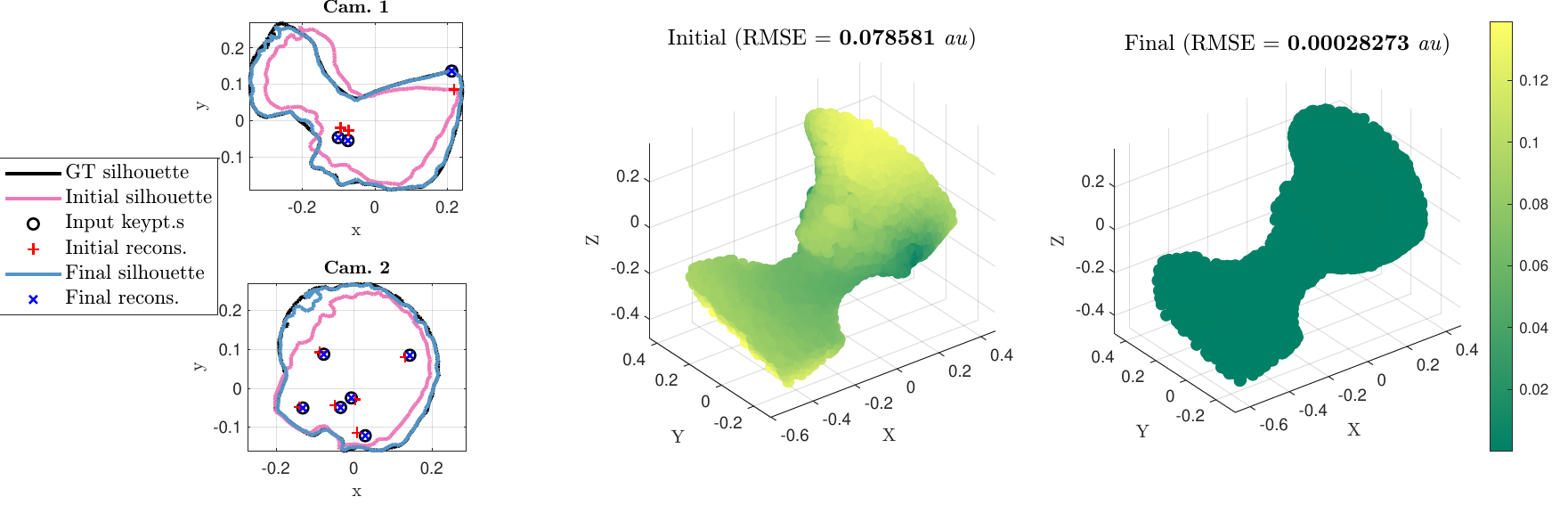}
  \label{fig:n_6_1}
\end{subfigure}%
\begin{subfigure}{0.5\textwidth}
  \centering
  \includegraphics[width=0.95\textwidth]{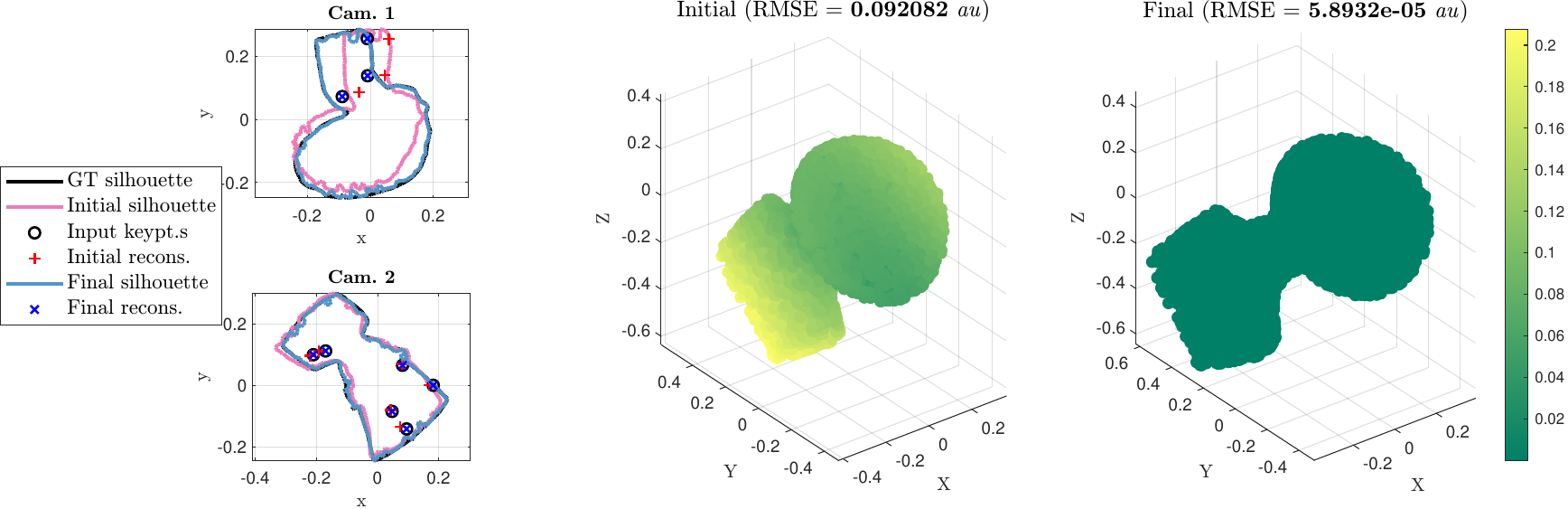}
  \label{fig:n_6_2}
\end{subfigure}%
\caption{Synthetic qualitative {\tt Nefertiti}}
\label{fig:sil_synth_qual_nefertiti}
\end{figure}

\begin{figure}
    \centering
\begin{subfigure}{0.5\textwidth}
  \centering
  \includegraphics[width=0.95\textwidth]{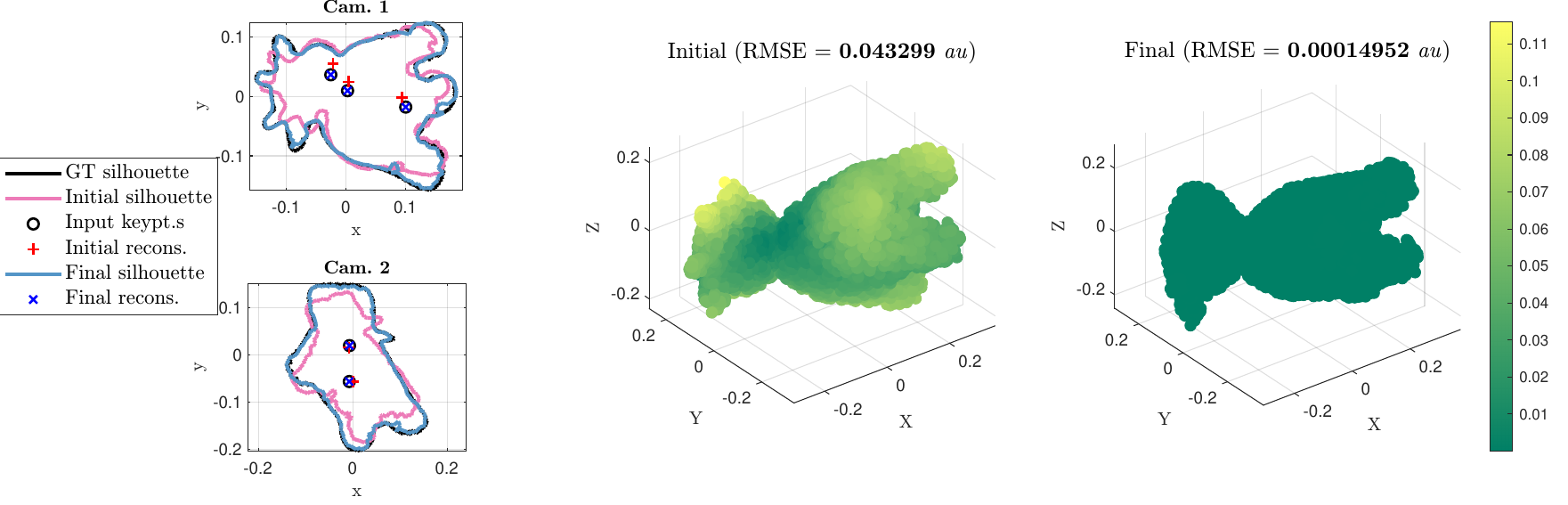}
  \label{fig:s_2_1}
\end{subfigure}%
\begin{subfigure}{0.5\textwidth}
  \centering
  \includegraphics[width=0.95\textwidth]{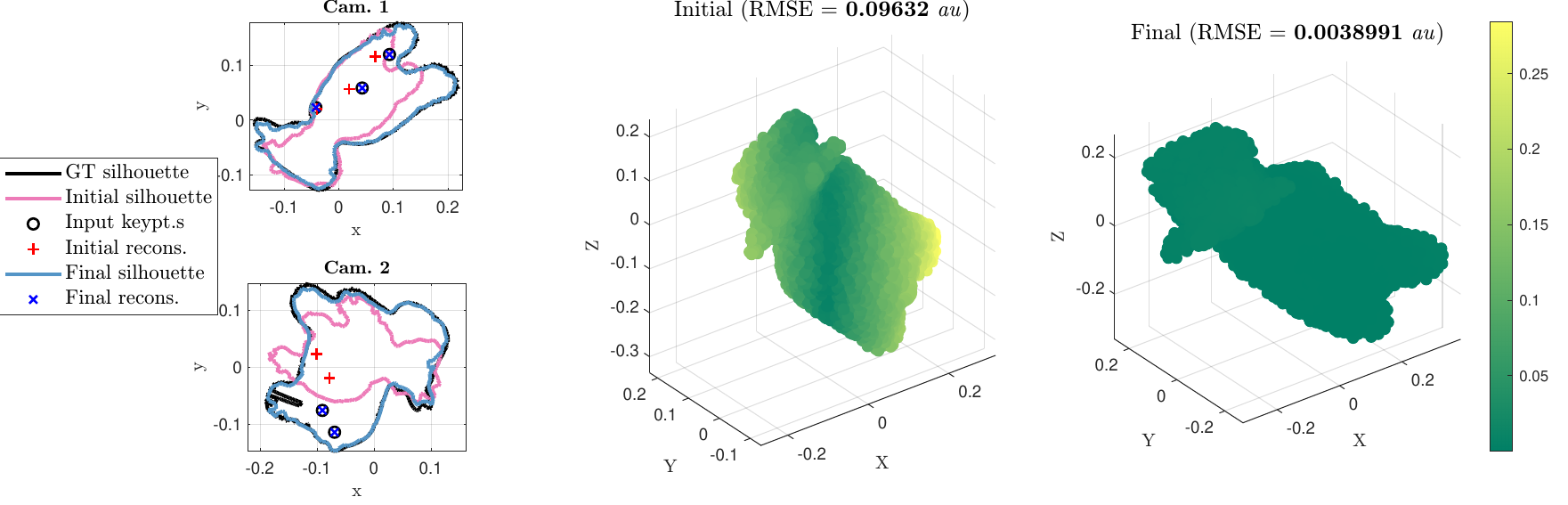}
  \label{fig:s_2_2}
\end{subfigure}%

\begin{subfigure}{0.5\textwidth}
  \centering
  \includegraphics[width=0.95\textwidth]{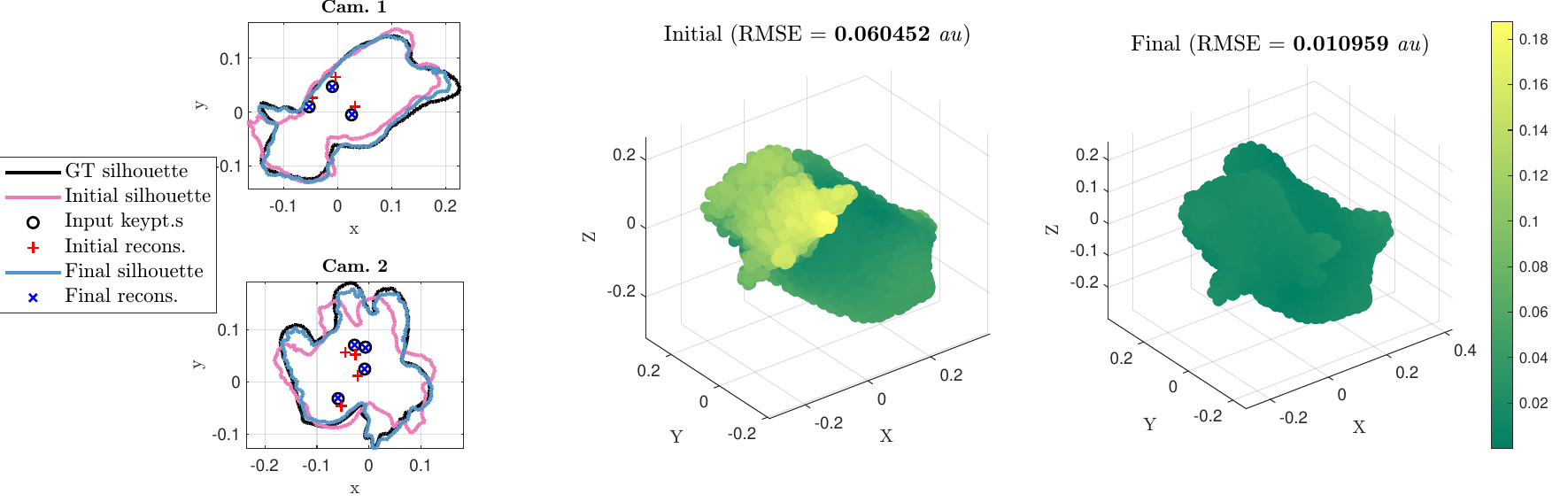}
  \label{fig:s_4_1}
\end{subfigure}%
\begin{subfigure}{0.5\textwidth}
  \centering
  \includegraphics[width=0.95\textwidth]{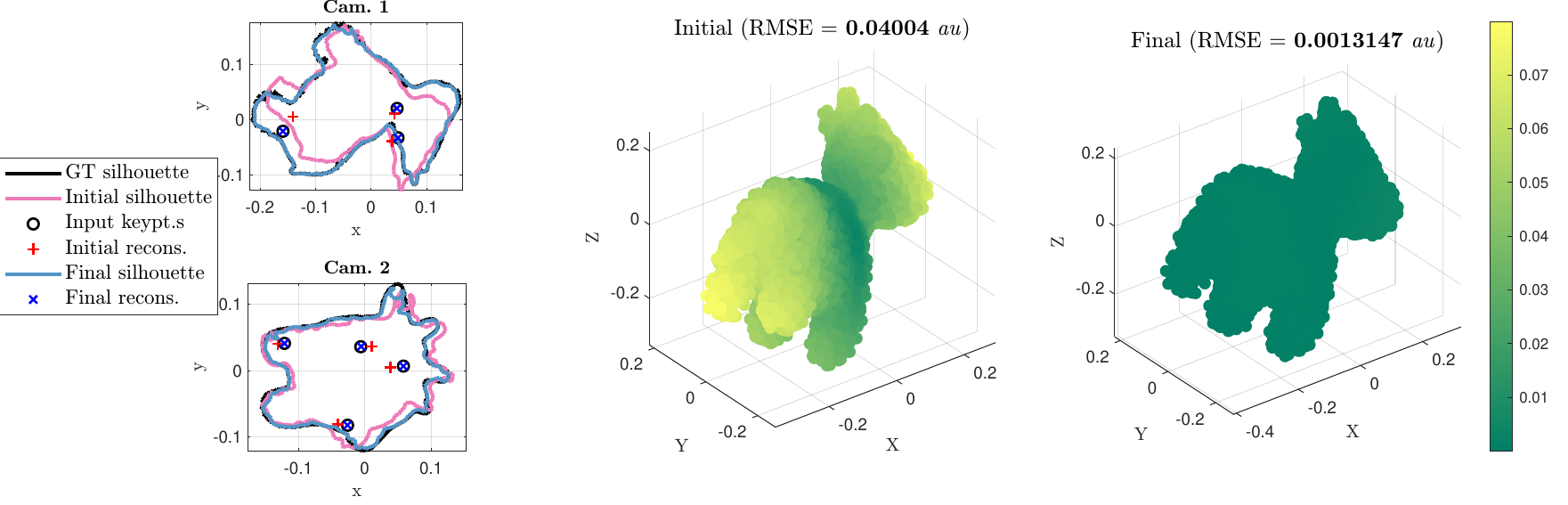}
  \label{fig:s_4_2}
\end{subfigure}%

\begin{subfigure}{0.5\textwidth}
  \centering
  \includegraphics[width=0.95\textwidth]{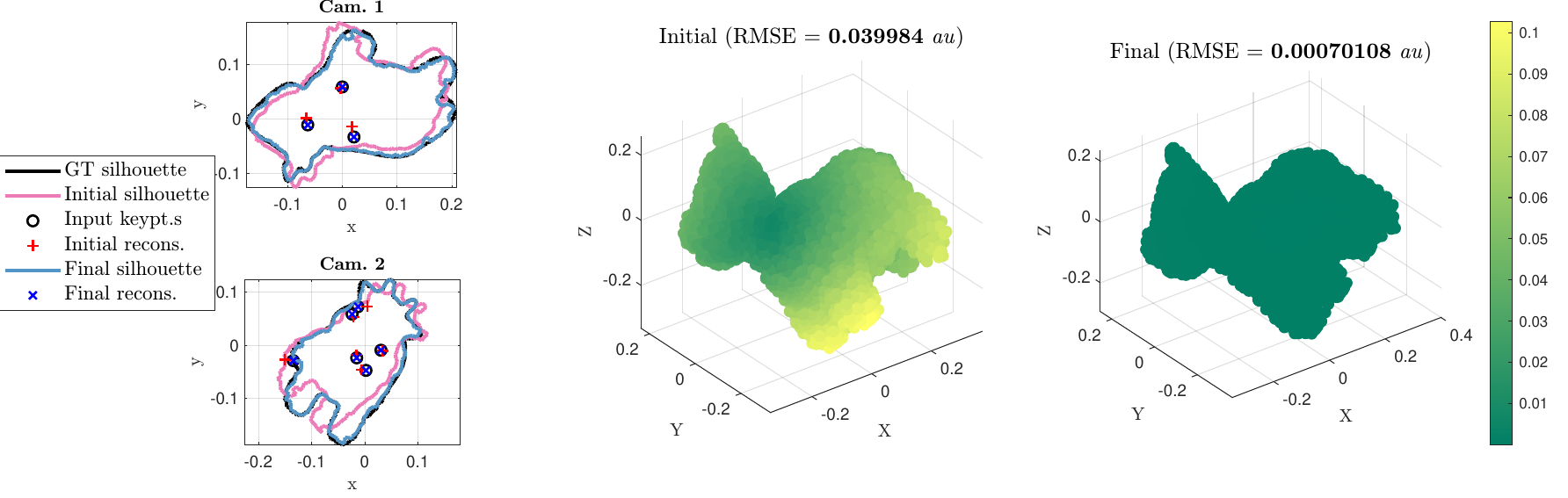}
  \label{fig:s_6_1}
\end{subfigure}%
\begin{subfigure}{0.5\textwidth}
  \centering
  \includegraphics[width=0.95\textwidth]{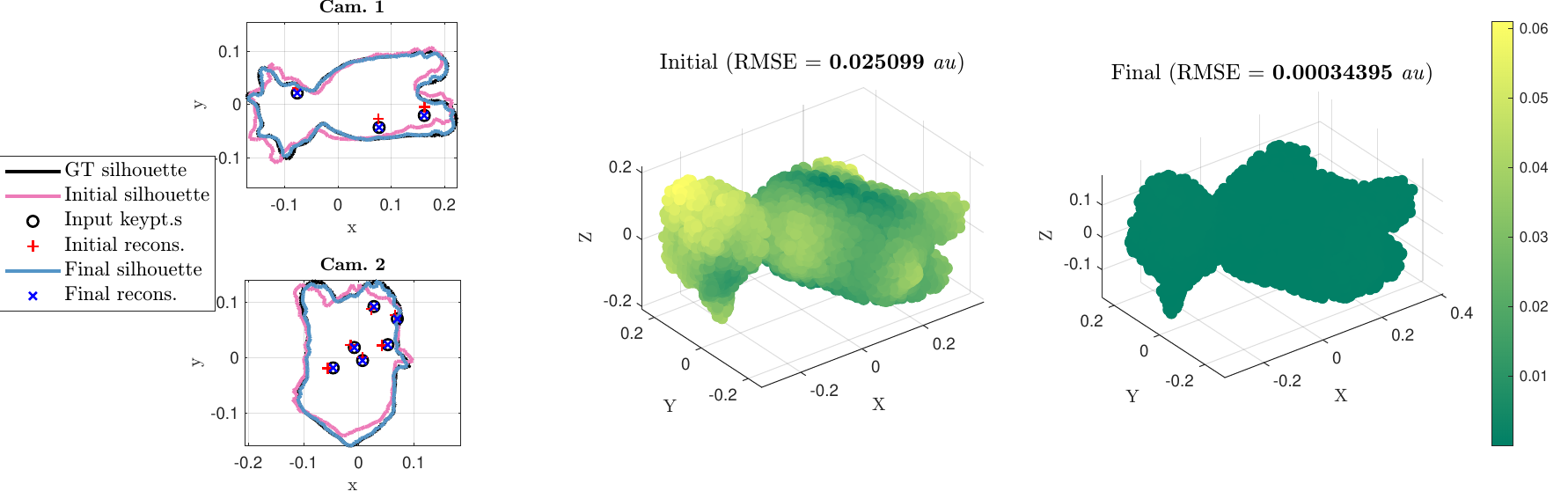}
  \label{fig:s_6_2}
\end{subfigure}%
\caption{Synthetic qualitative {\tt Spot}}
\label{fig:sil_synth_qual_spot}
\end{figure}

\section*{Appendices}
\appendix
We present the appendices below:

\section{Proof of \cref{prop_gauge_fix}}\label{prop_proof}
The reprojection cost, expanded, now writes as:
    \begin{equation}
        \begin{gathered}\label{reproj_cost_exp_2}
            \gamma\bigg(\Pi\Big( \mathbf{R}_x^{\top}\big(\mathbf{R}\mathbf{Q}_j + \mathbf{t} - \mathbf{t}_x)\big)\Big), \mathbf{p}_{x,j'}\bigg) = \| \mathbf{R}_x^{\top}\big(\mathbf{R}\mathbf{Q}_j + \mathbf{t} - \mathbf{t}_x\big) \times \mathfrak{d}_{x,j'}\|_1 \\
            = \| (\mathbf{R}_x^{\top} \mathbf{R}\mathbf{Q}_j \times \mathfrak{d}_{x,j'}) + (\mathbf{R}_x^{\top}\mathbf{t}\times \mathfrak{d}_{x,j'}) - (\mathbf{R}_x^{\top}\mathbf{t}_x\times \mathfrak{d}_{x,j'})\|_1.
        \end{gathered}
    \end{equation}
    \noindent It can be quickly verified that adding spurious rigid transformations like \cref{spurious_T} into \cref{reproj_cost_exp_2} results in:
    \begin{equation}
        \begin{gathered}
            \| (\mathbf{R}_x^{\top} \mathbf{R}' \mathbf{R}\mathbf{Q}_j \times \mathfrak{d}_{x,j'}) + (\mathbf{R}_x^{\top}\mathbf{R}'\mathbf{t}\times \mathfrak{d}_{x,j'}) - (\mathbf{R}_x^{\top}\mathbf{R}'\mathbf{t}_x\times \mathfrak{d}_{x,j'}) + (\mathbf{R}_x^{\top} \mathbf{t}'\times \mathfrak{d}_{x,j'}) \|_1 = \gamma\bigg(\Pi\Big( \mathbf{R}_x^{\top}\big(\mathbf{R}\mathbf{Q}_j + \mathbf{t})\big)\Big), \mathbf{p}_{x,j'}\bigg), \\~\text{iff:} ~ \mathbf{R}' = \mathbbm{1}_3, \mathbf{t}' = \mathbf{0}_{3 \times 1}.
        \end{gathered}
    \end{equation}

\section{Proof of \cref{threorem_1}}\label{theo_2}
 The vector product in the projection cost of \cref{eqn_pt_based_unext} can be distributed to express it as:
    \begin{equation}\label{proj_cost_expanded}
        \begin{gathered}
            \gamma_P\Big(\Pi\big(\mathbf{R}\mathbf{Q}_j + \mathbf{t})\big), \mathbf{p}_{x,j'}\Big) = \| \big(\mathbf{R}\mathbf{Q}_j + \mathbf{t} - \mathbf{C}_x\big) \times \mathbf{R}_x \mathfrak{d}_{x,j'}\|_1 \\ = \| \big(\mathbf{R}\mathbf{Q}_j \times \mathfrak{d}_{x,j'} \big) + \big( \mathbf{t} \times \mathfrak{d}_{x,j'} \big) - \big( \mathbf{C}_x \times \mathbf{R}_x\mathfrak{d}_{x,j'} \big)\|_1 , \quad \forall x \in [1,P], ~(j,j') \in \Omega_x
        \end{gathered}
    \end{equation}
\noindent We show that if we add a spurious rigid transformation $(\mathbf{R}', \mathbf{t}')$ to the object in world coordinates $\mathbf{R}\mathbf{Q}_j + \mathbf{t}$ and the viewpoint centres $\mathbf{C}_x$ in \cref{proj_cost_expanded}, i.e., 
\begin{equation}\label{spurious_T}
    \mathbf{R}\mathbf{Q}_j + \mathbf{t} \leftarrow \mathbf{R}' \mathbf{R} \mathbf{Q}_j + \mathbf{R}' \mathbf{t} + \mathbf{t}', \quad \text{and} \quad \mathbf{C}_x \leftarrow \mathbf{R}' \mathbf{C}_x + \mathbf{t}',
\end{equation}
\noindent the cost remains unchanged, which can be verified by expanding
\begin{equation}
    \begin{gathered}
            \gamma\Big(\Pi\big(\mathbf{R}\mathbf{Q}_j + \mathbf{t})\big), \mathbf{p}_{x,j'}\Big) = \| \big( ( \mathbf{R}' \mathbf{R}\mathbf{Q}_j) \times \mathfrak{d}_{x,j'}  \big) + \big( (\mathbf{R}' \mathbf{t} + \mathbf{t}') \times \mathfrak{d}_{x,j'}  \big) - \big( (\mathbf{R}' \mathbf{C}_x + \mathbf{t}') \times \mathbf{R}_x \mathfrak{d}_{x,j'} \big)  \|_1\\
             = \|  \mathbf{R}' \Big(\big(\mathbf{R}\mathbf{Q}_j + \mathbf{t} - \mathbf{C}_x\big) \times \mathbf{R}_x \mathfrak{d}_{x,j'} \Big)\|_1, \quad \forall \mathbf{R}' \in \mathbb{SO}(3),
    \end{gathered}
\end{equation}
\noindent since rotations are norm preserving. On the other hand, the regularizer $\rho(\bar{\mathbf{P}}_j, \mathbf{Q}_j)$ and the equality constraint in \cref{eqn_pt_based_unext} is invariant to global rigid transformations. We borrow the term \textit{gauge freedom} to denote such transformations that have no measurable effect from similar usage for pose estimation in robotics \cite{zhang2018comparison} and \y{sfm} \cite{mclauchlan1999gauge, triggs2000bundle}.

\section{Proof of \cref{lem_acc}}\label{proof_lem_3_app}

We complete this proof by assuming the case of $P=2$, the proof then generalises to any $P>2$. Note that the overall optimisation in \cref{eqn_pt_based_unext_2} can be split into $P=2$ separate and independent \y{sft} problems if the constraint $\mathbf{w}_{1,i} = \mathbf{w}_{2,i}, ~\forall i$ is dropped, we assume the first optimisation with $\mathbf{w}_1$ is {\tt optim-1} and the second optimisation with $\mathbf{w}_2$ is {\tt optim-2}. $\mathbf{w}_{1,i} = \mathbf{w}_{2,i}, \forall i$ is a crucial constraint ensuring shape consistency between viewpoints. Now suppose {\tt optim-1} has been supplied with $\tau$ correspondences, therefore in the absence of the equality constraints with $\mathbf{w}_2$, the primal of \cref{eqn_pt_based_unext_2} in {\tt optim-1} evaluates to $\nu$. However, if {\tt optim-2} is supplied with $\tau' < \tau$ correspondences, therefore in the absence of the equality constraints with $\mathbf{w}_1$, the primal of \cref{eqn_pt_based_unext_2} in {\tt optim-2} evaluates to $\nu' > \nu$ due to $\nu \propto \frac{1}{\eta}$, i.e., the cost is inversely proportional to accuracy (following the last clause of \cref{lem_acc}). With this background, if the constraint $\mathbf{w}_{1,i} = \mathbf{w}_{2,i}, \forall i$ is introduced back and the optimisation solved globally, the resulting primal cost would be $\nu''$ with $\nu' > \nu'' > \nu$, at least in the final \y{sdp} formulation of \cref{eqn_final_sdp_unknown}; since the primal is convex and sum of convex functions is convex, naturally generalising to the case of $P>2$. Therefore the global accuracy would be some $\eta' < \eta$ if $\nu \propto \frac{1}{\eta}$. 

\section{Verifying \y{emnc}}\label{taueta}
We offer an empirical evidence for the universality of \y{emnc} in correspondence-based \y{sft} by demonstrating its existence in two of the baseline approaches from literature in the methods {\tt CPB14I} and {\tt RS}, i.e., the two best state-of-the-art surface-based \y{sft} methods from our experiments in \cref{sec_addnl_analysis}, excluding our own proposed approaches. 

We utilise synthetic data to avoid contamination from spurious noise and to ensure high accuracy of \y{gt}. We use the synthetic surface generator from \cite{perriollat2013comp} to generate a random, ruled surface with bending angle of 20$^{\circ}$ and randomly roto-translate the resulting shape in space. We do \y{sft} with {\tt CPB14I} and {\tt RS} on the normalised image coordinates from the perspective projection of the resulting shape, the undeformed (flat) template is the flattening of this developable surface, obtained from \cite{perriollat2013comp}. The generated synthetic surface is attributed with 42 keypoints arranged in a $7 \times 6$ grid. From these keypoints, we subsample keypoints in 10 steps, starting from the sparsest 4 keypoints and going up to 40 keypoints. With these reconstructed 4 through 40 keypoints, we interpolate the entire set of 42 keypoints via a \y{tps} warp, e.g.: for 4 keypoints, we initialise the warp going from 4 keypoints in the undeformed template in $\mathbb{R}^2$ to the reconstructed 4 keypoints in $\mathbb{R}^3$ and then interpolate the remaining keypoints from this initialised warp. If \cref{assum_basic} holds, the final reconstruction accuracy, evaluated on all 42 keypoints, should increase monotonically with increasing sample size. We repeat each subsample step over 35 repeats, to discount the effect of accuracy variability due to the varying object pose. The results are shown in \cref{fig:taueta}.

\Cref{fig:te_1} shows the median of the reconstruction error in terms of \y{rms}. To make things clearer, \cref{fig:te_2} shows the median of the accuracy in terms of inverse \y{rms}. Despite significant outliers, especially from {\tt RS}, there is a clear and monotonic increase in accuracy with increasing correspondences, thus bolstering \cref{assum_basic}.

\begin{figure}[b]
\centering
\begin{subfigure}{0.4\textwidth}
  \centering
  \includegraphics[width=0.7\textwidth]{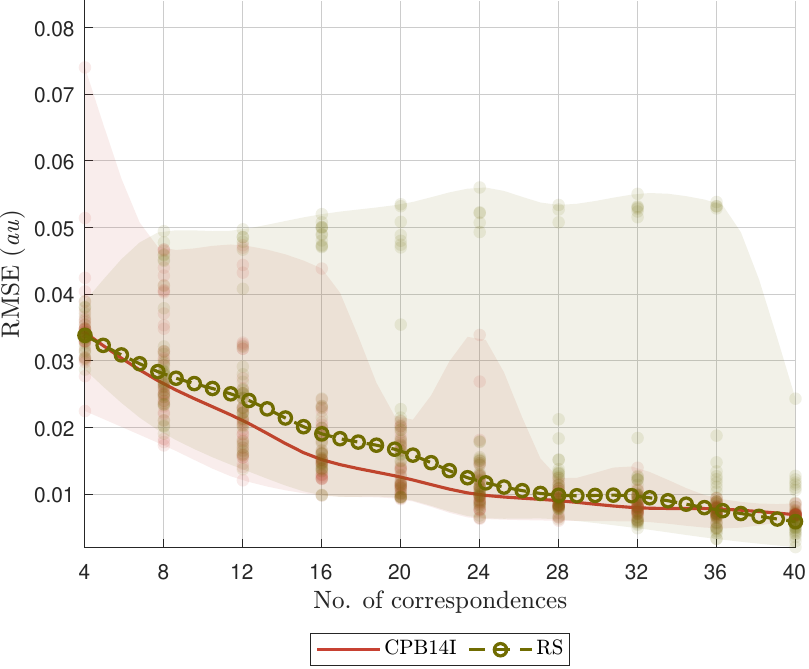}
  \caption{\y{rms} with increasing correspondences}
  \label{fig:te_1}
\end{subfigure}%
\begin{subfigure}{0.4\textwidth}
  \centering
  \includegraphics[width=0.65\textwidth]{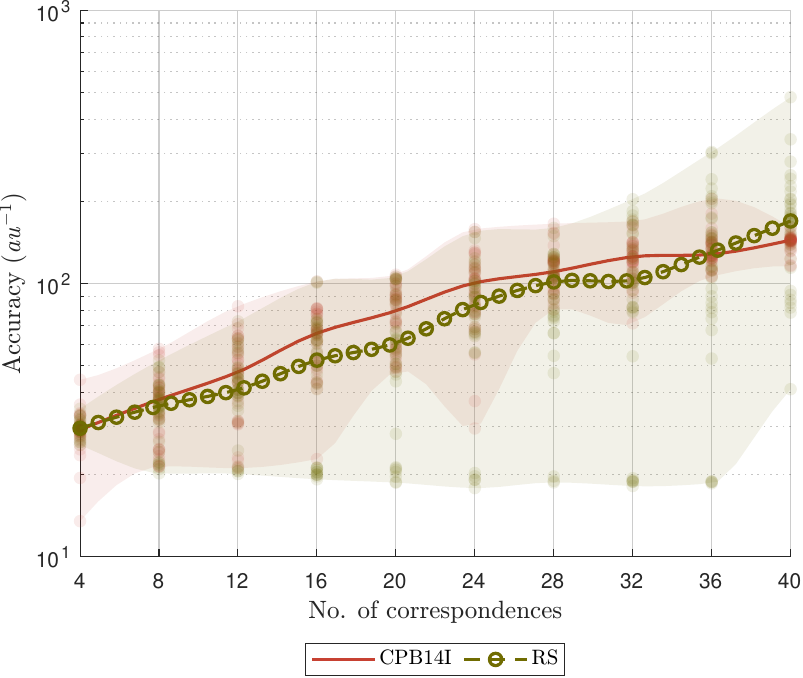}
  \caption{Accuracy with increasing correspondences}
  \label{fig:te_2}
\end{subfigure}%
\caption{Effect of increasing correspondences on global shape reconstruction error and accuracy}
\label{fig:taueta}
\end{figure}

\section{Sampling functions}\label{appen_sample}
We describe the various sampling functions used in our \y{sdp} formulations. First, we describe $\omega_a$, $\{\omega_{c,r}\}$, and $\{\omega_{d,r}\}$, since they are rather intuitive:
\begin{equation}
    \begin{gathered}
        \omega_a(\Delta_e) = \mathrm{tr}\Big( \mathrm{diag} \big( (\mathbf{0}_{1 \times 10}, \mathbf{1}_{1 \times M})^{\top} \big)\Delta_e\Big)
    \end{gathered}
\end{equation}

\begin{equation}
    \begin{gathered}
        \omega_{c,r}(\Delta_e) = \mathrm{tr}\big(\Delta_{e,[3r-1:3r+1],[3r-1:3r+1]} \big), \quad \forall r \in [1,3]
    \end{gathered}
\end{equation}

\begin{equation}
    \begin{gathered}
        \omega_{d,1}(\Delta_e) = \mathrm{tr}\big(\Delta_{e,[2:4],[5:7]} \big), \quad \omega_{d,2}(\Delta_e) = \mathrm{tr}\big(\Delta_{e,[2:4],[8:10]} \big), \quad \omega_{d,3}(\Delta_e) = \mathrm{tr}\big(\Delta_{e,[5:7],[8:10]} \big).
    \end{gathered}
\end{equation}

\noindent $\omega_b\Big(\Delta_e, \mathbf{p}_{x,j'}, \bar{\mathbf{P}}_j, \{\mathbf{P}_{i,j}\}, \mathbf{t}\Big)$ is derived by the following steps. First, the rotated mean point and the summed, rotated basis shapes are:

\begin{equation}
    \begin{gathered}
        \bar{\mathbf{P}}_{j}^{\perp} = \Big(\Delta_{e, [1:1], [2:10]} \big( \mathbbm{1}_{3} \otimes \bar{\mathbf{P}}_{j}\big)\Big)^{\top}, \quad \text{and} \quad \mathbf{P}_{j}^{\perp} = \sum_{i=1}^M \Big(\Delta_{e, [1:1], [2:10]} \big( \mathbbm{1}_{3} \otimes \bar{\mathbf{P}}_{i,j}\big)\Big)^{\top}
    \end{gathered}
\end{equation}

\noindent respectively. The final translated point is:
\begin{equation}
    \mathbf{P}_{j}^{\uparrow} = \big(\bar{\mathbf{P}}_{j}^{\perp} + \mathbf{P}_{j}^{\perp}\big) + \mathbf{t},
\end{equation}
\noindent and therefore, the final, linearised cost is:
\begin{equation}
    \begin{gathered}
        \omega_b\Big(\Delta_e, \mathbf{p}_{x,j'}, \bar{\mathbf{P}}_j, \{\mathbf{P}_{i,j}\}, \mathbf{t}\Big) = \big(\mathbf{P}_{j}^{\uparrow} - \mathbf{C}_x\big) \times \mathfrak{d}_{x,j'},
    \end{gathered}
\end{equation}

\noindent where $\mathfrak{d}_{x,j'} \in \mathbb{S}^2$ is the direction component of Pl\"{u}cker vectors corresponding to 2D image keypoint $\mathbf{p}_{x,j'}$ and $\mathbf{C}_x$ is the viewpoint centre.

{\color{revCol1}
\section{Challenges with physics-based deformation models}\label{app_challenge_phy}
 It is noteworthy that deformation models other than \y{ssm}, including physics-based paradigms such as isometry, conformality, equiareality, and volume-preservation, have not been integrated with camera pose estimation within a convex programming framework. The underlying challenge can be illustrated by analyzing the transformed shape $\mathbf{R} \mathbf{Q} + \mathbf{t}$. When $\mathbf{Q}$ is parameterized using physics-based models, such as the zeroth-order positional coordinates in $\mathbb{R}^3$ or depth values along sightlines, the emergence of multivariate monomials in $\mathbf{R} \mathbf{Q}$ grows exponentially with the number of matched keypoints, making the problem computationally challenging enough to render it impractical for realistic cases. For higher-order methods employing continuous surface representations, deformation typically relies on `control handles.' Realistic deformable surfaces often necessitate tens or even hundreds of such handles, leading to excessively large multivariate terms in $\mathbf{R} \mathbf{Q}$. The only common method for linearising such multivariate polynomials is the Lasserre's hierarchy of low-rank relaxations \cite{lasserre2009moments} which famously scales very poorly. Conversely, \y{ssm} represents one of the lowest-dimensional methodologies for deformation representation, thereby making it particularly advantageous for simultaneous optimization of pose and shape. The convex integration of physics-based models to achieve joint shape and pose estimation under perspective projection for deforming objects remains an yet unresolved computational challenge.}

% \end{appendices}

% \section{Pose estimation}\label{sec_wahba}
% For the $x'$-th camera below, the rotation and translation is obtained as follows:
% \begin{enumerate}
%     \item We derive the summed outer-product matrix $\mathbf{B}_{x'} = \sum_{j=1}^N \hat{\mathbf{Q}}_j^{\ast}(\mathbf{Q}^{\ast}_j)^{\top}$, where $\hat{\mathbf{Q}}^{\ast} = \mathbf{R}_1 \mathbf{Q}^{\ast} + \mathbf{t}_1$, and use the \y{svd} as $\mathbf{U}_{x'}\mathbf{S}_{x'}\mathbf{W}_{x'}^{\top} = \mathrm{svd}(\mathbf{B}_{x'})$ to derive the optimal rotation:
%     \begin{equation}
%         \mathbf{R}_{x'} = \mathbf{U}_{x'}\mathrm{diag}\Big(\big(1, 1, \mathrm{det}(\mathbf{U}_{x'})\mathrm{det}(\mathbf{W}_{x'})\big)^{\top}\Big)\mathbf{W}_{x'}^{\top}
%     \end{equation}
%     \item Given the rotation, the translation is simply $\mathbf{t}_{x'} = \hat{\mathbf{Q}}^{\ast} - \mathbf{R}_{x'} \mathbf{Q}^{\ast}$
% \end{enumerate}
% We repeat the two-steps given above for all $x' \in [2,P]$, exactly once per camera.

\bibliographystyle{abbrvnat}
\bibliography{references}
\end{document}